\newcommand{\isPreprint}{1}

\if\isPreprint1
\newcommand{\isNotPreprint}{0}
\else
\newcommand{\isNotPreprint}{1}
\fi

\newcommand{\AAAIdoesNotKnowHowToUseLaTeX}{\isNotPreprint}
\newcommand{\boolHardCodedAppendixRef}{\isNotPreprint}
\newcommand{\showAppendix}{\isPreprint}

\if\AAAIdoesNotKnowHowToUseLaTeX 1
\documentclass[letterpaper]{article} 
\else
\documentclass[letterpaper,backref=page]{article}
\fi

\if\isPreprint1
\usepackage{arXiv_aaai23} 
\else
\usepackage{aaai23}  
\fi
\usepackage{times}  
\usepackage{helvet}  
\usepackage{courier}  
\usepackage[hyphens]{url}  
\usepackage{graphicx} 
\usepackage{natbib}  
\usepackage{caption} 
\title{Bayesian Optimization-Based Combinatorial Assignment\footnotemark[1]}

\author{
	Jakob Weissteiner\textsuperscript{\rm 1,\rm 3}\footnotemark[2]{\normalfont,}
	Jakob Heiss\textsuperscript{\rm 2,\rm 3}\footnotemark[2]{\normalfont,}
	Julien Siems\textsuperscript{\rm 1}\footnotemark[2] {\normalfont and} 
	Sven Seuken\textsuperscript{\rm 1,\rm 3}
}
\affiliations{
\textsuperscript{\rm 1}University of Zurich\\
\textsuperscript{\rm 2}ETH Zurich\\
\textsuperscript{\rm 3}ETH AI Center\\
weissteiner@ifi.uzh.ch,
jakob.heiss@math.ethz.ch,
juliensiems@gmail.com,
seuken@ifi.uzh.ch
}

\usepackage{arXiv_our_commands} 

\providecommand{\AAAIdoesNotKnowHowToUseLaTeX}{0}

\begin{document}

\maketitle

\begin{abstract}
We study the combinatorial assignment domain, which includes combinatorial auctions and course allocation. The main challenge in this domain is that the bundle space grows exponentially in the number of items. To address this, several papers have recently proposed \emph{machine learning-based preference elicitation} algorithms that aim to elicit only the most important information from agents. However, the main shortcoming of this prior work is that it does not model a mechanism's uncertainty over values for not yet elicited bundles. In this paper, we address this shortcoming by presenting a \emph{Bayesian optimization-based combinatorial assignment (BOCA)} mechanism. Our key technical contribution is to integrate a method for capturing model uncertainty into an iterative combinatorial auction mechanism. Concretely, we design a new method for estimating an \emph{upper uncertainty bound} that can be used to define an acquisition function to determine the next query to the agents. This enables the mechanism to properly \emph{explore} (and not just \emph{exploit}) the bundle space during its preference elicitation phase. We run computational experiments in several spectrum auction domains to evaluate BOCA's performance. Our results show that BOCA achieves higher allocative efficiency than state-of-the-art approaches.
\end{abstract}

\renewcommand{\thefootnote}{\fnsymbol{footnote}}
\footnotetext[1]{\if\isPreprint1
This paper is the full version of \cited{weissteiner2023bayesian} published at AAAI'23 including the appendix.
\else
The full paper including appendix is available on arXiv via: \url{https://arxiv.org/abs/2208.14698}.%
\fi}
\footnotetext[2]{These authors contributed equally.}
\renewcommand{\thefootnote}{\arabic{footnote}}

\section{Introduction}\label{sec:Introduction}
Many economic problems require finding an efficient combinatorial assignment of multiple indivisible items to multiple agents. Popular examples include \emph{combinatorial auctions (CAs)}, \emph{combinatorial exchanges (CEs)}, and \emph{combinatorial course allocation}. In CAs, heterogeneous items are allocated among a set of bidders, e.g., for the sale of spectrum licenses \cite{cramton2013spectrumauctions}. In CEs, items are allocated among agents who can be sellers \emph{and} buyers at the same time, e.g., for the reallocation of catch shares \cite{bichler2019designing}. In course allocation, course seats are allocated among students at universities \cite{budish2011combinatorial}.

In all these domains, agents have preferences over \emph{bundles} of items. In particular, agents' preferences may exhibit \textit{complementarities} and \textit{substitutabilities} among items.
A mechanism that allows agents to report values for bundles rather than just for individual items can achieve significantly higher efficiency. However, this also implies that agents' preferences are exponentially-sized (i.e., for $m$ items there are $2^m$ different bundles), which implies that agents cannot report values for all bundles. Therefore, the key challenge in combinatorial assignment is the design of a \emph{preference elicitation (PE)} algorithm that is (i) \textit{practically feasible} w.r.t. elicitation costs and (ii) \emph{smart}, i.e., it should elicit the information that is ``most useful'' for achieving high efficiency.

\subsection{Iterative Combinatorial Auctions (ICAs)}\label{ICAs}
While the PE challenge is common to all combinatorial assignment problems, it has been studied most intensely in the CA domain \citep{sandholm2006preference}. In CAs with general valuations, the amount of communication  needed to guarantee full efficiency is exponential in the number of items \cite{nisan2006communication}. Thus, practical CAs cannot provide efficiency guarantees. In practice, \textit{iterative combinatorial auctions (ICAs)} are therefore employed, where the auctioneer interacts with bidders over rounds, eliciting a \emph{limited} (and thus \emph{practically feasible}) amount of information, aiming to find a highly efficient allocation. ICAs are widely used; e.g., for the sale of licenses to build offshore wind farms \cite{ausubel2011auction}. The provision of spectrum licenses via the \emph{combinatorial clock auction (CCA) \cite{ausubel2006clock}} has generated more than \$20 billion in total revenue \cite{ausubel2017practical}. Thus, increasing the efficiency of such real-world ICAs by only 1\% point translates into monetary gains of hundreds of millions of dollars.

\subsection{ML-Powered Preference Elicitation}\label{subsec:Machine Learning-powered Assignemt Mechanims}

In recent years, researchers have used machine learning (ML) to design smart PE algorithms. Most related to this paper is the work by \citet{brero2018combinatorial,brero2021workingpaper}, who developed the first practical ML-powered ICA that outperforms the CCA. The main idea of their mechanism is two-fold: first, they train a separate \textit{support vector regression} model to learn each bidder's full value function from a small set of bids; second, they solve an \textit{ML-based winner determination problem (WDP)} to determine the allocation with the highest predicted social welfare, and they use this allocation to generate the next set of queries to all bidders. This process repeats in an iterative fashion until a fixed number of queries has been asked. Thus, their ML-powered ICA can be interpreted as a form of combinatorial \textit{Bayesian optimization (BO)} (see \Appendixref{sec:appendix:BO}{Appendix~C}).

In several follow-up papers, this work has been extended by developing more sophisticated ML methods for this problem. \citet{weissteiner2020deep} integrated \emph{neural networks (NN)} in their ICA and further increased efficiency. \citet{weissteiner2022fourier} used\emph{ Fourier transforms (FTs)} to leverage different notions of sparsity of value functions. Finally, \citet{weissteiner2022monotone} achieved state-of-the-art (SOTA) performance using \emph{monotone-value NNs (MVNNs)}, which incorporate important prior domain knowledge.

The main shortcoming of this prior work is that all of these approaches are \emph{myopic} in the sense that the resulting mechanisms simply query the allocation with the highest predicted welfare. In particular, the mechanisms do not have any model of \emph{uncertainty} over bidders' values for not yet elicited bundles, although handling uncertainty in a principled manner is one of the key requirements of a smart PE algorithm \citep{guo2010gaussian}. Thus, the mechanisms cannot properly control the \textit{exploration-exploitation trade-off} inherent to BO. For ML-based ICAs, this means that the mechanisms may get stuck in local minima, repeatedly querying one part of the allocation space while not exploring other, potentially more efficient allocations.

\subsection{Our Contributions}\label{subsec:Contribution}
In this paper, we address this main shortcoming of prior work and show how to integrate a notion of \textit{model uncertainty} (i.e., epistemic uncertainty) over agents' preferences into iterative combinatorial assignment. Concretely, we design a \emph{Bayesian optimization-based combinatorial assignment (BOCA)}\footnote{The acronym BOCA has also been used for a different method, namely for \emph{Bayesian optimisation with continuous approximations} by \citet{kandasamy2017multi}.} mechanism that makes use of model uncertainty in its query generation module. The main technical challenge is to design a new method for estimating an \emph{upper uncertainty bound} that can be used to define an acquisition function to determine the next query. For this we combine \emph{MVNNs} \citep{weissteiner2022monotone} with \emph{neural optimization-based model uncertainty (NOMU)} \cite{heiss2022nomu}, a recently introduced method to estimate model uncertainty for NNs. In detail, we make the following contributions: 
\begin{enumerate}[leftmargin=*,topsep=0pt,partopsep=0pt, parsep=0pt]
    \item We present a modified NOMU algorithm (\Cref{subsec:NOMU for Combinatorial Assignment}), tailored to CAs, exploiting monotonicity of agents' preferences and the discrete (finite) nature of this setting.
    \item We show that generic parameter initialization for monotone NNs can dramatically fail and propose a new initialization method for MVNNs based on uniform mixture distributions (\Cref{subsec:New MVNN Random Initialization}).
    \item We present a more succinct mixed integer linear program for MVNNs to solve the ML-based WDP (\Cref{subsec:New MVNN-based MIP}).
    \item We experimentally show that BOCA outperforms prior approaches in terms of efficiency (\Cref{sec:experiments}).
\end{enumerate}

Although our contribution applies to any combinatorial assignment setting, we focus on CAs to simplify the notation and because there exist well-studied preference generators for CAs that we use for our experiments.

Our source code is publicly available on GitHub via: \url{https://github.com/marketdesignresearch/BOCA}.
\subsection{Related Work on Model Uncertainty}\label{subsec:related work}
Estimating model uncertainty for NNs is an active area of research in AI and ML, with a plethora of new methods proposed every year. Classic methods can be broadly categorized into \emph{(i) ensemble methods}: training multiple different NNs to estimate model uncertainty \cite{gal2016dropout,lakshminarayanan2017simple,wenzel2020hyperparameter} and \emph{(ii) Bayesian NNs (BNNs)}: assuming a prior distribution over parameters and then estimating model uncertainty by approximating the intractable posterior \cite{graves2011practical,blundell2015weight,hernandez2015probabilistic,ober2019benchmarking}. However, for ML-based iterative combinatorial assignment, a key requirement is to be able to efficiently solve the ML-based WDP based on these uncertainty estimates. As there is no known computationally tractable method to perform combinatorial optimization over ensembles or BNNs, we cannot use these approaches for ML-based ICAs. In contrast, NOMU \citep{heiss2022nomu} enables the computationally efficient optimization over its uncertainty predictions, which is why we use it as a building block for BOCA.

\section{Preliminaries}\label{sec:Preliminaries}
In this section, we present our formal model (\Cref{subsec:Formal Model for ICAs}), review the ML-based ICA by \cited{brero2021workingpaper} (\Cref{subsec:A Machine Learning powered ICA}), briefly review Bayesian optimization (BO) (\Cref{subsec:app:Bayesian Optimization Background}), and review \emph{monotone-value neural networks (MVNNs)} by \cited{weissteiner2022monotone} (\Cref{subsec:mvnns}) as well as \emph{neural optimization-based model uncertainty (NOMU)} by \cited{heiss2022nomu} (\Cref{subsec:NOMU}).

\subsection{Formal Model for ICAs}\label{subsec:Formal Model for ICAs}
We consider a CA with $n$ bidders and $m$ indivisible items. Let $N=\{1,\ldots,n\}$ and $M=\{1,\ldots,m\}$ denote the set of bidders and items. We denote by $x\in \X=\{0,1\}^m$ a bundle of items represented as an indicator vector, where $x_{j}=1$ iff item $j \in M$ is contained in $x$. Bidders' true preferences over bundles are represented by their (private) value functions $v_i: \X\to \R_+,\,\, i \in N$, i.e., $v_i(x)$ represents bidder $i$'s true value for bundle $x\in \X$.

By $a=(a_1,\ldots,a_n) \in \X^n$ we denote an allocation of bundles to bidders, where $a_i$ is the bundle bidder $i$ obtains. We denote the set of \emph{feasible} allocations by $\F=\left\{a \in \X^n:\sum_{i \in N}a_{ij} \le 1, \,\,\forall j \in M\right\}$. We let $p\in \R^n_+$ denote the bidders' payments. We assume that bidders have quasilinear utility functions $u_i$ of the form $u_i(a,p) = v_i(a_i) - p_i$. This implies that the (true) \emph{social welfare} $V(a)$ of an allocation $a$ is equal to the sum of all bidders' values $\sum_{i\in N} v_i(a_i)$.
We let $a^* \in \argmax_{a \in {\F}}V(a)$ denote a social-welfare maximizing, i.e., \textit{efficient}, allocation. The \emph{efficiency} of any allocation $a \in \F$ is determined as $V(a)/V(a^*)$.

An ICA \textit{mechanism} defines how the bidders interact with the auctioneer and how the allocation and payments are determined. We denote a bidder's (possibly untruthful) reported value function by $\hvi{}:\X\to\R_+$. In this paper, we consider ICAs that ask bidders to iteratively report their values $\hvi{x}$ for bundles $x$ selected by the mechanism. A  finite set of reported bundle-value pairs of bidder $i$ is denoted as ${R_i=\left\{\left(x^{(l)},\hvi{x^{(l)}}\right)\right\},\,x^{(l)}\in \X}$. Let $R=(R_1,\ldots,R_n)$ be the tuple of reported bundle-value pairs obtained from all bidders. We define the \textit{reported social welfare} of an allocation $a$ given $R$ as $\hV{a|R}\coloneqq \sum_{i \in N:\, \left(a_i,\hvi{a_i}\right)\in R_i}\hvi{a_i},$ where $\left(a_i,\hvi{a_i}\right)\in R_i$ ensures that only values for reported bundles contribute. The ICA's optimal allocation $a^*_{R}\in \F$ and payments $p(R)\in \R^n_+$ are computed based on the elicited reports $R$ \emph{only}. More formally, $a^*_{R}\in \F$ given reports $R$ is defined as 
\begin{align}\label{WDPFiniteReports}
a^*_{R} \in \argmax_{a \in {\F}}\hV{a|R}.
\end{align}

As the auctioneer can only query a limited number of bundles $|R_i| \leq \Qmax$ (e.g., $\Qmax=100$), an ICA needs a practically feasible and smart PE algorithm.

\subsection{A Machine Learning-Powered ICA}\label{subsec:A Machine Learning powered ICA}
We now provide a high-level review of the \textit{machine learning-powered combinatorial auction (MLCA)} by \cited{brero2021workingpaper} (please see \Appendixref{sec:appendix_A Machine Learning powered ICA}{Appendix~A} for further details). MLCA proceeds in rounds until a maximum number of value queries per bidder $\Qmax$ is reached. In each round, for every bidder $i$, an ML model $\mathcal{A}_i$ is trained on the bidder's reports $R_i$ to learn an approximation of bidders' value functions. Next, MLCA generates new value queries by computing the allocation with the highest predicted social welfare. Concretely, it computes $\qnew=(\qnew_i)_{i=1}^{n}$ with $\qnew_i\in \X \setminus R_i$ by solving an ML-based WDP:
\begin{equation}\label{eq:ML-based-WDP}
\qnew \in \argmax\limits_{a \in {\F}}\sum\limits_{i \in N} \mathcal{A}_i(a_i)
\end{equation}
The idea is the following: if the $\mathcal{A}_i$'s are good surrogate models of the bidders' value functions, then the efficiency of $\qnew$ is likely to be high as well. Thus, in each round, bidders are providing value reports on bundles that are guaranteed to fit into a feasible allocation and that together are predicted to have high social welfare. Additionally, bidders are also allowed to submit ``push-bids,'' enabling them to submit information to the auctioneer that they deem useful, even if they are not explicitly queried about it.
At the end of each round, MLCA receives reports $\Rnew$ from all bidders for the newly generated queries $\qnew$, and updates the overall elicited reports $R$. When $\Qmax$ is reached, MLCA computes an allocation $a^*_R$ that maximizes the \emph{reported} social welfare (\Cref{WDPFiniteReports}) and determines VCG payments $p(R)$ based on all reports (see \Appendixref[Appendix ]{def:vcg_payments}{Definition~B.1}).

\begin{remark}[IR, No-Deficit, and Incentives of MLCA]
\cited{brero2021workingpaper} showed that MLCA satisfies \emph{individual rationality (IR)} and \emph{no-deficit}, with \emph{any} ML algorithm. They also studied MLCA's incentive properties; this is important, since manipulations may lower efficiency. Like all deployed ICAs (including the CCA), MLCA is not strategyproof. However, they argued that it has good incentives in practice; and given two additional assumptions, bidding truthfully is an ex-post Nash equilibrium. We present a detailed summary of their incentive analysis in \Appendixref{sec:appendix_Incentives of MLCA}{Appendix~B}.
\end{remark}

\subsection{Bayesian Optimization Background}\label{subsec:app:Bayesian Optimization Background}
In this section, we briefly review Bayesian optimization (BO). BO refers to a class of \emph{machine learning-based} \emph{gradient-free} optimization methods, which, for a given black-box objective function $f:X\to \R$, aim to solve
\begin{equation}\label{eq:bo_problem}
    \max_{x\in X}f(x)
\end{equation}
in an \emph{iterative} manner. Specifically, given a budget of $T$ queries (i.e., function evaluations of $f$), a BO algorithm generates queries $\{x^{(1)},\ldots,x^{(T)}\}$ with the aim that

\begin{equation}\label{eq:goal_of_BO}
    \max_{x\in \{x^{(1)},\ldots,x^{(T)}\}}f(x) \approx \max_{x\in X} f(x).
\end{equation}

In each BO step $t$, the algorithm selects a new input point $x^{(t)}\in X$ and observes a (potentially noisy) output
\begin{equation}
    y^{(t)}=f(x^{(t)})+\varepsilon^{(t)},
\end{equation}
where $\varepsilon^{(t)}$ is typically assumed to be i.i.d.\ Gaussian, i.e., $\varepsilon^{(t)}\sim \mathcal{N}(0,\sigma^2)$.\footnote{In this paper, we assume that $\sigma^2=0$.} The BO algorithm's decision rule for selecting the query $x^{(t)}$ is based on
\begin{enumerate}
    \item A \emph{probabilistic model} representing an (approximate) posterior distribution over $f$ (e.g., Gaussian processes, NOMU, ensembles, BNNs, etc.).
    \item An \emph{acquisition function} $\mathcal{A}:X \to \R$ that uses this probabilistic model to determine the next query $x^{(t)}\in\argmax_{x\in X} \mathcal{A}(x)$ by properly trading off \emph{exploration} and \emph{exploitation}. See \Appendixref{subsec:app:Choice of Acquisition function}{Appendix~C.3} for popular examples of acquisition functions including:
    \begin{itemize}[leftmargin=*]
        \item \emph{Upper uncertainty bound} (uUB) (aka \emph{upper confidence bound (UCB)})~\cite{Srinivas_2012}
        \item \emph{Expected improvement}~\cite[Section 4.1]{frazier2018tutorial}
        \item \emph{Thompson sampling}~\cite{Chapelle2011AnEE}
    \end{itemize}
\end{enumerate}

\begin{remark}
MLCA (\Cref{subsec:A Machine Learning powered ICA}) can be seen as a combinatorial BO algorithm with acquisition function $\mathcal{A}(a)\coloneqq \sum_{i\in N}\mathcal{A}_i(a_i)$ (see \Appendixref{sec:appendix:BO}{Appendix~C} for a discussion).
\end{remark}

\subsection{MVNNs: Monotone-Value Neural Networks}\label{subsec:mvnns}
MVNNs \cite{weissteiner2022monotone} are a new class of NNs specifically designed to represent \emph{monotone} \emph{combinatorial} valuations. First, we reprint the definition of MVNNs and then discuss their desirable properties.
\begin{definition}[MVNN, \cited{weissteiner2022monotone}]\label{def:MVNN}
		An MVNN $\MVNNi{}:\X \to \mathbb{R}_+$  for bidder $i\in N$ is defined as
		\begin{equation}\label{eq:MVNN}
		\MVNNi{x}\coloneqq  W^{i,K_i}\varphi_{0,t^{i, K_i-1}}\left(\ldots\varphi_{0,t^{i, 1}}(W^{i,1}x+b^{i,1})\ldots\right)
		\end{equation}
		\begin{itemize}[leftmargin=*,topsep=0pt,partopsep=0pt, parsep=0pt]
		\item $K_i+1\in\mathbb{N}$ is the number of layers ($K_i-1$ hidden layers),
		\item $\{\varphi_{0,t^{i, k}}{}\}_{k=1}^{K_i-1}$ are the MVNN-specific activation functions with cutoff $t^{i, k}>0$, called \emph{bounded ReLU (bReLU)}:
		\begin{align}\label{itm:MVNNactivation}
		\varphi_{0,t^{i, k}}(\cdot)\coloneqq\min(t^{i, k}, \max(0,\cdot))
		\end{align}
		\item $W^i\coloneqq (W^{i,k})_{k=1}^{K_i}$ with $W^{i,k}\ge0$ and $b^i\coloneqq (b^{i,k})_{k=1}^{K_i-1}$ with $b^{i,k}\le0$ are the \emph{non-negative} weights and \emph{non-positive} biases of dimensions $d^{i,k}\times d^{i,k-1}$ and $d^{i,k}$, whose parameters are stored in $\theta=(W^i,b^i)$.
		\end{itemize}
\end{definition}
MVNNs are particularly well suited for the design of combinatorial assignment mechanism for two reasons. First, MVNNs are \textit{universal} in the set of monotone and normalized value functions \citep[Theorem~1]{weissteiner2022monotone}, i.e., \emph{any} $\hvi{}:\X\to\Rp$ that satisfies the following two properties can be represented \emph{exactly} as an MVNN~$\MVNNi{}$:
\begin{enumerate}[align=left, leftmargin=*,topsep=2pt]\label{monotonicity_and_normalization}
    \item\textbf{Monotonicity (M)}~(\emph{``additional items increase value''}):\\ For $A,B \in 2^M$: if $A\subseteq B$ it holds that $\hvi{A}\le \hvi{B}$
    \item\textbf{Normalization (N)}~(\emph{''no value for empty bundle''}):\\ $\hat{v}_i(\emptyset)=\hat{v}_i((0,\ldots,0))\coloneqq 0$,
\end{enumerate}
Second, \cited{weissteiner2022monotone} showed that an MVNN-based WDP, i.e., $\argmax\limits_{a\in \F}\sum_{i \in N}\MVNNi{a_i}$, can be succinctly encoded as a MILP, which is key for the design of MVNN-based iterative combinatorial assignment mechanisms. Finally, \cited{weissteiner2022monotone} experimentally showed that using MVNNs as $\mathcal{A}_i$ in MLCA leads to SOTA performance.

\subsection{NOMU}\label{subsec:NOMU}
Recently, \citet{heiss2022nomu} introduced a novel method to estimate model uncertainty for NNs: \emph{neural optimization-based model uncertainty (NOMU)}. In contrast to other methods (e.g., ensembles), NOMU represents an \emph{upper uncertainty bound (uUB)} as a \emph{single} and \emph{MILP-formalizable} NN. Thus, NOMU is particularly well suited for iterative combinatorial assignment, where uUB-based \emph{winner determination problems (WDPs)} need to be solved hundreds of times to generate new informative queries. This, together with NOMU's strong performance in noiseless BO, is the reason why we build on it and define a modified NOMU algorithm tailored to iterative combinatorial assignment (\Cref{subsec:NOMU for Combinatorial Assignment}).

\section{Bayesian Optimization-Based ICA}\label{sec:Integrating Uncertainty into an ICA}
In this section, we describe the design of our Bayesian optimization-based combinatorial assignment (BOCA) mechanism. While the design is general, we here present it for the CA setting, leading to a BO-based ICA. Recall that MLCA generates new value queries by solving the ML-based WDP $\qnew \in \argmax\limits_{a \in {\F}}\sum\limits_{i \in N} \mathcal{A}_i(a_i)$ (see \Cref{subsec:A Machine Learning powered ICA}). For the design of BOCA, we integrate a proper notion of uncertainty into MLCA by using a bidder-specific \emph{upper uncertainty bound (uUB)}, taking the role of the ML model $\mathcal{A}_i$, to define our acquisition function $\mathcal{A}(a)\coloneqq \sum_{i\in N}\mathcal{A}_i(a_i)$. To define our uUB and make it amenable to MLCA, we proceed in three steps: First, we combine MVNNs with a modified NOMU algorithm that is tailored to the characteristics of combinatorial assignment (\Cref{subsec:NOMU for Combinatorial Assignment}). Second, we highlight the importance of proper parameter initialization for MVNNs and propose a more robust method (\Cref{subsec:New MVNN Random Initialization}). Third, we present a more succinct MILP for MVNNs (\Cref{subsec:New MVNN-based MIP}).
In the remainder of the paper, we make the following assumption:
\begin{assumption}\label{assumption:monotonicity_and_normalization}
For all agents $i\in N$, the true and reported value functions $v_i$ and $\hvi{}$ fulfill the \textbf{Monotonicity (M)} and \textbf{Normalization (N)} property (see \Cref{monotonicity_and_normalization}).
\end{assumption}

\subsection{Model Uncertainty for Monotone NNs}\label{subsec:NOMU for Combinatorial Assignment}
We propose a modified NOMU architecture and loss that is specifically tailored to combinatorial assignment. Concretely, our algorithm is based on the following two key characteristics of combinatorial assignment: (i) since agents' value functions are monotonically increasing, the uUBs need to be monotonically increasing too, and (ii) due to the (finite) discrete input space, one can derive a closed-form expression of the 100\%-uUB as an MVNN. Before we present our modified NOMU architecture and loss, we introduce the MVNN-based 100\%-uUB.

Let $\HC$ denote a \emph{hypothesis class} of functions $f:X\to \R$  for some input space $X$ and let $\HC_{\Dtr}\coloneqq \{f\in \HC:f(x^{(l)})=y^{(l)}, l=1,\ldots,\ntr\}$ denote the set of all functions from $\HC$ that fit exactly through training points $\Dtr=\left\{\left(x^{(l)},f(x^{(l)})\right)\right\}_{l=1}^{\ntr}$.
\begin{definition}[100\%-uUB]\label{def:100 upper UB MVNN}
For a hypothesis class $\HC$ and a training set $\Dtr$, we define the 100\%-uUB as $\fuUB(x)\coloneqq \sup_{f\in\HC_{\Dtr}}f(x)$ for every $x\in X$.
\end{definition}
In the following, let
\begin{align}\label{eq:Vmon}
    \Vmon\coloneqq \{\hv{}:\X \to \Rp|\, \text{satisfy \textbf{(N)} and \textbf{(M)}}\}
\end{align}
denote the set of all value functions that satisfy the \emph{normalization} and \emph{monotonicity} property. Next, we define the 100\%-uUB. 
In \Cref{proposition:100uUB as MVNN}, we show that for $\HC=\Vmon$ the 100\%-uUB can be explicitly represented as an MVNN.
\begin{theorem}[MVNN-based 100\%-uUB]\label{proposition:100uUB as MVNN}
Let $\left((1,\ldots,1), \hvi{1,\ldots,1}\right)\in \Dtr$. Then for $\HC=\Vmon$ it holds that $\fuUB(x)=\max_{f\in\Vmon_{\Dtr}}f(x)$ for all $x\in \mathcal{X}$ and $\fuUB\in\Vmon_{\Dtr}$ can be represented as a two hidden layer MVNN with $\ntr$ neurons per layer, which we denote as $\oneMi$ going forward.\footnote{Note that $\oneMi(\cdot)$ depends on a training set $\Dtr$, but we omit this dependency in our notation to improve readability.}
\end{theorem}
\begin{proof}
The proof for \Cref{proposition:100uUB as MVNN} is provided in \Appendixref{proof:proposition:100uUB as MVNN}{Appendix~D.1}. It follows a similar idea as the universality proof in \cite[Theorem 1]{weissteiner2022monotone}. In particular, \Appendixref{eq:last_term}{Equation~(27)} in \Appendixref{proof:proposition:100uUB as MVNN}{Appendix~D.1} provides the closed-form expression of $\fuUB$ as MVNN $\oneMi$.
\end{proof}
Using the MVNN-based 100\%-uUB $\oneMi$, we can now define our modified NOMU architecture and loss.

\paragraph{The Architecture.}
Towards defining the architecture, we first observe that if the true function is monotonically increasing, the corresponding uUB needs to be monotonically increasing as well (\Appendixref{prop:BayesUUBmonotone,prop:frequUUBmonotone}{Proposition~1 and~2} in \Appendixref{subsec:monotoneUUB}{Appendix~D.2}). Given that bidders' value functions are monotone (\Cref{assumption:monotonicity_and_normalization}), this implies that our uUB must also be monotonically increasing. Thus, instead of the original NOMU architecture that outputs the (raw) uncertainty (i.e., an estimate of the posterior standard deviation) which is \emph{not} monotone, we can modify NOMU's architecture and directly output the monotone uUB. Given this, we propose the following architecture $\NOMUi$ to estimate the uUB for bidder $i\in N$. $\NOMUi$ consists of two sub-MVNNs with two outputs: the mean prediction $\meanMi:\X\to\R$ and the estimated uUB $\uUBMi:\X\to\R$. In \Cref{fig:nn_tik_extension}, we provide a schematic representation of $\NOMUi$
(see \Appendixref{subsec:monotoneUUB}{Appendix~D.2} for details).
\begin{figure}[t!]
        \begin{center}
        \centerline{
        \resizebox{1\columnwidth}{!}{
                \begin{tikzpicture}
                [cnode/.style={draw=black,fill=#1,minimum width=3mm,circle}]
                
                \node[cnode=gray,label=180:{\Huge$x \in \X$}] (x1) at (0.5,0) {};
                
                \node[cnode=gray] (x2+3) at (3,3) {};
                \node at (3,2) {$\vdots$};
                \node[cnode=gray] (x2+1) at (3,1) {};
                
                \node[cnode=gray] (x2-3) at (3,-3) {};
                \node at (3,-2) {$\vdots$};
                \node[cnode=gray] (x2-1) at (3,-1) {};
                
                \draw (x1) -- (x2+3);
                \draw (x1) -- (x2+1);
                
                \draw (x1) -- (x2-1);
                \draw (x1) -- (x2-3);

                \node[cnode=gray] (x3+3) at (6,3) {};
                \node at (6,2) {$\vdots$};
                \node[cnode=gray] (x3+1) at (6,1) {};
                
                \node[cnode=gray] (x3-3) at (6,-3) {};
                \node at (6,-2) {$\vdots$};
                \node[cnode=gray] (x3-1) at (6,-1) {};

                \foreach \y in {1,3}
                {   \foreach \z in {1,3}
                        {   \draw (x2+\z) -- (x3+\y);
                        }
                }
                \foreach \y in {1,3}
                {   \foreach \z in {1,3}
                        {   \draw (x2-\z) -- (x3-\y);
                        }
                }
                
                \node at (7.5,+3) {$\ldots$};
                \node at (7.5,+2) {$\ldots$};
                \node at (7.5,+1) {$\ldots$};
                
                \node at (7.5,-1) {$\ldots$};
                \node at (7.5,-2) {$\ldots$};
                \node at (7.5,-3) {$\ldots$};
                
                \node[cnode=gray] (x4+3) at (9,3) {};
                \node at (9,2) {$\vdots$};
                \node[cnode=gray] (x4+1) at (9,1) {};
                
                \node[cnode=gray] (x4-3) at (9,-3) {};
                \node at (9,-2) {$\vdots$};
                \node[cnode=gray] (x4-1) at (9,-1) {};
                
                \node[cnode=gray,label=360:{\Huge$\uUBMi(x)$}] (x5+2) at (11.5,2) {};
                \node[cnode=gray,label=360:{\Huge$\meanMi(x)$}] (x5-2) at (11.5,-2) {};
                
                \draw (x4+3)--(x5+2);
                \draw (x4+1)--(x5+2);
                
                \draw (x4-1)--(x5-2);
                \draw (x4-3)--(x5-2);

                \draw [dotted, line width=0.4mm] (2.5,0.5) rectangle (9.5,+3.5);
                \draw [dotted, line width=0.4mm] (2.5,-3.5) rectangle (9.5,-0.5);
                
                \node at (0.25,-2.5) {\Huge Mean-network};
                \node at (0.25,+2.5) {\Huge uUB-network};

                \end{tikzpicture}
        }
        }
        \caption{$\NOMUi$: a modification of NOMU's original architecture for the combinatorial assignment domain.}
        \label{fig:nn_tik_extension}
        \end{center}
\vskip -0.2in
\end{figure}
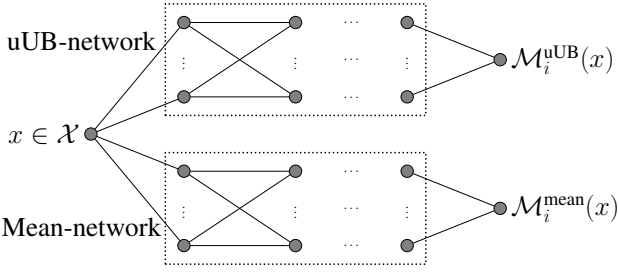

\paragraph{The Loss.}
Next, we formulate a new NOMU loss function $L^\hp$ tailored to combinatorial assignment.
Since we have a closed-form expression of the 100\%-uUB as MVNN $\oneMi$ (\Cref{proposition:100uUB as MVNN}), we are able to enforce that $\meanMi\le\uUBMi\le \oneMi$ via the design of our new loss function.
Let $\meanMi$ be a trained mean-MVNN with a standard loss (e.g., MAE and L2-regularization).
Using $\meanMi$ and $\oneMi$, we then only train the parameters $\theta$ of $\uUBMi$ with loss $L^\hp$ and L2-regularization parameter $\lambda>0$, i.e., minimizing $L^\hp(\uUBMi)+\lambda\twonorm[{\theta}]^2$ via gradient descent. In particular, the parameters of $\oneMi$ and $\meanMi$ are not influenced by the training of $\uUBMi$ (see \Appendixref{sec.appednix:DetailedNOMULoss}{Appendix~D.3} for details on the loss and training procedure).

\begin{definition}[NOMU Loss Tailored to Combinatorial Assignment]\label{def:NOMU_Loss} Let $\hp=(\musqr,\muexp,\cexp, \piOneuUB,\piMean)\in \R_+^5$ be a tuple of hyperparameters and let $s(\uUBMi,x)\coloneqq\min\{\uUBMi(x),\oneMi(x)\}  - \meanMi(x)$ for all $x\in \X$. For a training set $\Dtr$, $L^\hp$ is defined as
\begin{subequations}\label{eq:NOMU_Loss_Extension}%
\begin{align}
&L^\hp(\uUBMi)\coloneqq\musqr\sum_{l=1}^{\ntr}\sL\left(\uUBMi(x^{(l)}),y^{(l)}\right)\label{subeq:dataLoss}\\
  &+\muexp\int_{[0,1]^m}g\left(-\cexp s(\uUBMi,x)\right)\,dx\label{subeq:pushUpLoss}\\&+\muexp\cexp\piOneuUB\int_{[0,1]^m}\sL\left((\uUBMi(x)-\oneMi(x))^{+}\right)\,dx\label{subeq:below100Loss}\\&+\muexp\cexp\piMean\int_{[0,1]^m}\sL\left((\meanMi(x)-\uUBMi(x))^{+}\right)\,dx\,,\label{subeq:aobveMeanLoss}
\end{align}
\end{subequations}
where $\sL$ is the smooth L1-loss with threshold $\beta$ (see \Appendixref[Appendix ]{def:appendix_smoothL1Loss}{Definition~D.1}), $(\cdot)^{+}$ the positive part, and $g\coloneqq1+\elu$\footnote{In our notation, $g(\cdot)$ is the analog of the function $e^{(\cdot)}$ used in the original NOMU loss in \citep{heiss2022nomu}.} is convex monotonically increasing with $\elu$ being the \emph{exponential linear unit} (see \Appendixref[Appendix ]{def:appendix_ELU}{Definition~D.2}).
\end{definition}
The interpretations of the four terms are as follows:
\begin{enumerate}[align= left]
	\item[\eqref{subeq:dataLoss}] enforces that $\uUBMi$ fits through the training data.
	\item[\eqref{subeq:pushUpLoss}] pushes $\uUBMi$ up as long as it is below the 100\%-uUB $\oneMi$.
	This force gets weaker the further $\uUBMi$ is above the mean $\meanMi$ (especially if $\cexp$ is large).
	$\muexp$ controls the overall strength of \eqref{subeq:pushUpLoss} and $\cexp$ controls how fast this force increases when $\uUBMi \to\meanMi$.
	Thus, increasing $\muexp$ increases the uUB and increasing $\cexp$ increases the uUBs in regions where it is close to $\meanMi$.
	Weakening \eqref{subeq:pushUpLoss} (i.e.,  $\muexp\cexp\to0$) leads to $\uUBMi \approx\meanMi$. Strengthening \eqref{subeq:pushUpLoss} by increasing $\muexp\cexp$ in relation to regularization\footnote{Regularization can be early stopping or a small number of neurons (implicit) or L2-regularization on the parameters (explicit).
	} leads to $\uUBMi\approx\oneMi$.
	\item[\eqref{subeq:below100Loss}] enforces that $\uUBMi\le\oneMi$. The strength of this term is determined by $\piOneuUB\cdot (\muexp\cexp)$, where $\piOneuUB$ is the \eqref{subeq:below100Loss}-specific hyperparameter and $\muexp\cexp$ adjusts the strength of \eqref{subeq:below100Loss} to \eqref{subeq:pushUpLoss}.
	\item[\eqref{subeq:aobveMeanLoss}] enforces $\uUBMi\ge\meanMi$. The interpretation of $\piMean$ and $\muexp\cexp$ is analogous to \eqref{subeq:below100Loss}.
\end{enumerate}

As in \cite{heiss2022nomu}, in the implementation of $L^\hp$, we approximate \Crefrange{subeq:pushUpLoss}{subeq:aobveMeanLoss} via Monte Carlo integration using additional, \emph{artificial input points} ${\Dart\coloneqq \left\{x^{(l)}\right\}_{l=1}^{\nart}\stackrel{i.i.d.}{\sim}\textrm{Unif}([0,1]^m)}$.

\paragraph{Visualization of the uUB.}
In Figure 2, we present a visualization of the output of $\NOMUi$ (i.e., $\meanMi$ and $\uUBMi$) and $\oneMi$ for the national bidder in the LSVM domain of the spectrum auction test suite (SATS) \cite{weiss2017sats}. In noiseless regression, uncertainty should vanish at observed training points, but (model) uncertainty should remain about value predictions for bundles that are very different from the bundles observed in training. \Cref{fig:1d_path_plot} shows that our uUB $\uUBMi$  nicely fulfills this.
Moreover, we have shown in \Appendixref{subsec:monotoneUUB}{Appendix~D.2} that $\uUBMi$ is monotonically increasing, since we assume that value functions fulfill the monotonicity property. This implies that once we observe a value for the full bundle, we obtain a globally bounded 100\%-uUB, i.e., see $\oneMi$ in \Cref{fig:1d_path_plot}.
Furthermore, we see that $\oneMi$ jumps to a high value when only a single item is added to an already queried bundle, but then often stays constant (e.g., $|x|=12,\ldots,18$ in \Cref{fig:1d_path_plot}).
Thus, using such a 100\%-uUB in our acquisition function, BOCA would only add a single item to an already queried bundle to have more items left for the other bidders instead of properly exploring the bundle space.
Our uUB $\uUBMi$ circumvents this via implicit and explicit regularization and yields a useful uUB.

\begin{figure}[t!]
    \begin{center}
    \includegraphics[width=1\columnwidth,trim=0 0 0 0, clip]{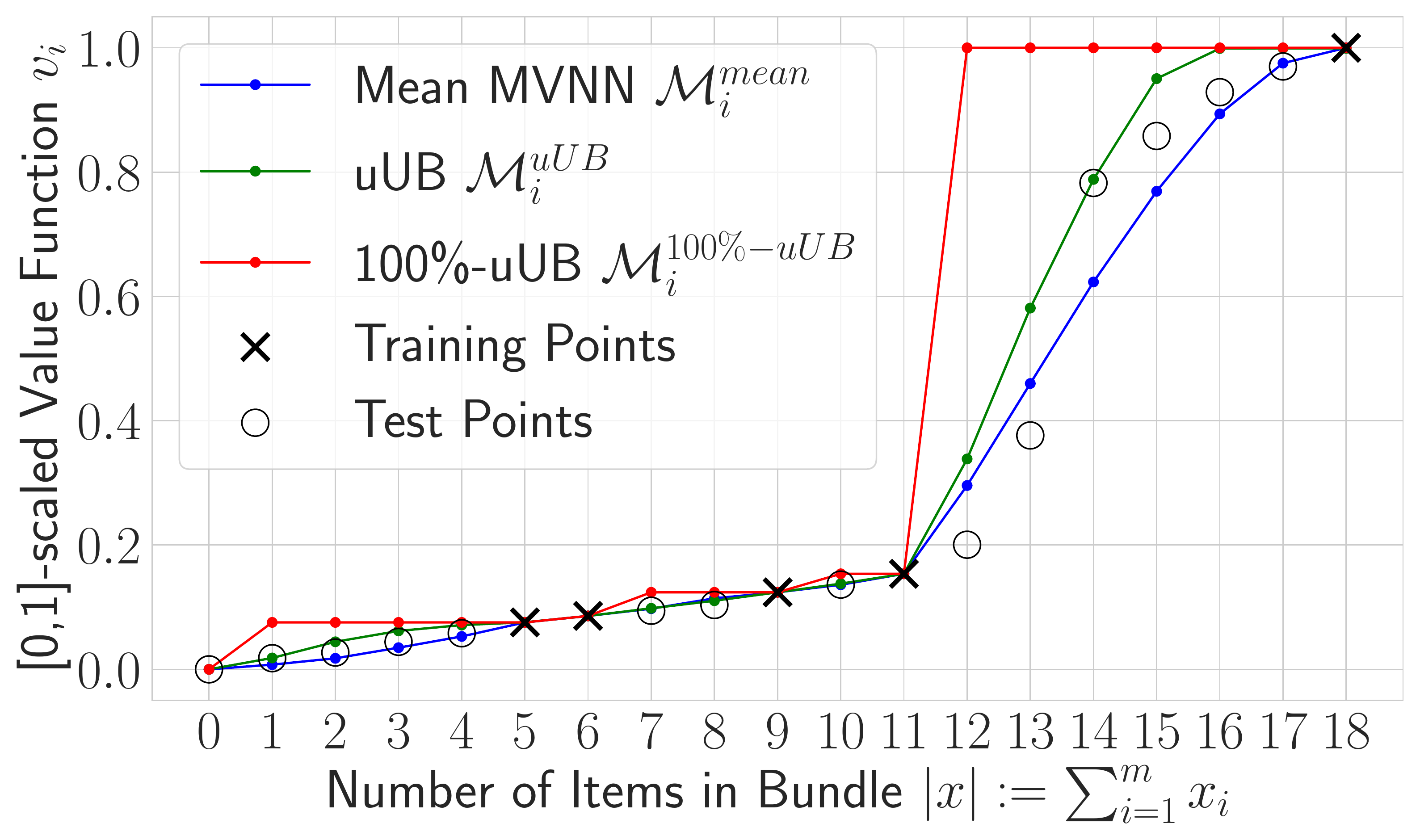}
    \caption{$\meanMi$, $\uUBMi$ and $\oneMi$ along an increasing 1D subset-path (i.e., for all bundles $x^{(j)}, x^{(k)}$ on the x-axis it holds that for $j\le k: x^{(j)}\subset x^{(k)}$).}
    \label{fig:1d_path_plot}
    \end{center}
\end{figure}
\subsection{Parameter Initialization for MVNNs}\label{subsec:New MVNN Random Initialization}
We now discuss how to properly initialize parameters for MVNNs. Importantly, the \emph{MVNN-based uUBs} $\uUBMi$ are MVNNs. As we will show next, to achieve the best performance of BOCA (in fact of any MVNN training), an adjusted, non-generic parameter initialization is important.

\paragraph{Generic Initialization.} For standard NNs, it is most common to use a parameter initialization with zero mean $\mu_k\coloneqq\E[W^{i,k}_{j,l}]=0$ and non-zero variance $\sigma_k^2\coloneqq\Var[W^{i,k}_{j,l}]\neq0$. Then the mean of each pre-activated neuron of the first hidden layer is zero and the variance $\Var[\left(W^{i,1}x\right)_j]=d^{i,0}\sigma_1^2\overline{x^2}$, if $\left(W^{i,1}_{j,l}\right)_{l=1}^{d^{i,0}}$ are i.i.d., where $\overline{x^2}=\frac{1}{d^{i,0}}\sum_{l=1}^{d^{i,0}} x_l^2$.\footnote{We assume that the biases $b^{i,k}=0$ are all initialized to zero throughout \Cref{subsec:New MVNN Random Initialization} to keep the notation simpler, while we formulate everything for the general case including random biases in \Appendixref{sec:apendix_New MVNN Random Initialization}{Appendix~E} and in our code.} Analogously, one can compute the \emph{conditional} mean 
and the \emph{conditional} variance 
of a pre-activated neuron in any layer $k$ by replacing $x$ by the output $z^{i,k-1}$ of the previous layer, i.e., $\Eco{\left(W^{i,k}z^{i,k-1}\right)_j}{z^{i,k-1}}=0$ and $\Varco{\left(W^{i,k}z^{i,k-1}\right)_j}{z^{i,k-1}}=d^{i,k-1}\sigma_k^2\overline{\left(z^{i,k-1}\right)^2}$
. For $\sigma_k\propto\frac{1}{\sqrt{d^{i,k-1}}}$, the conditional mean and variance do not depend on the layer dimensions $d^{i,k}$, which is why generic initialization methods scale the initial distribution by $s_k\propto\frac{1}{\sqrt{d^{i,k-1}}}$. 

\begin{figure}[t!]
    \begin{center}
    \includegraphics[width=1\columnwidth,trim=0 0 0 0, clip]{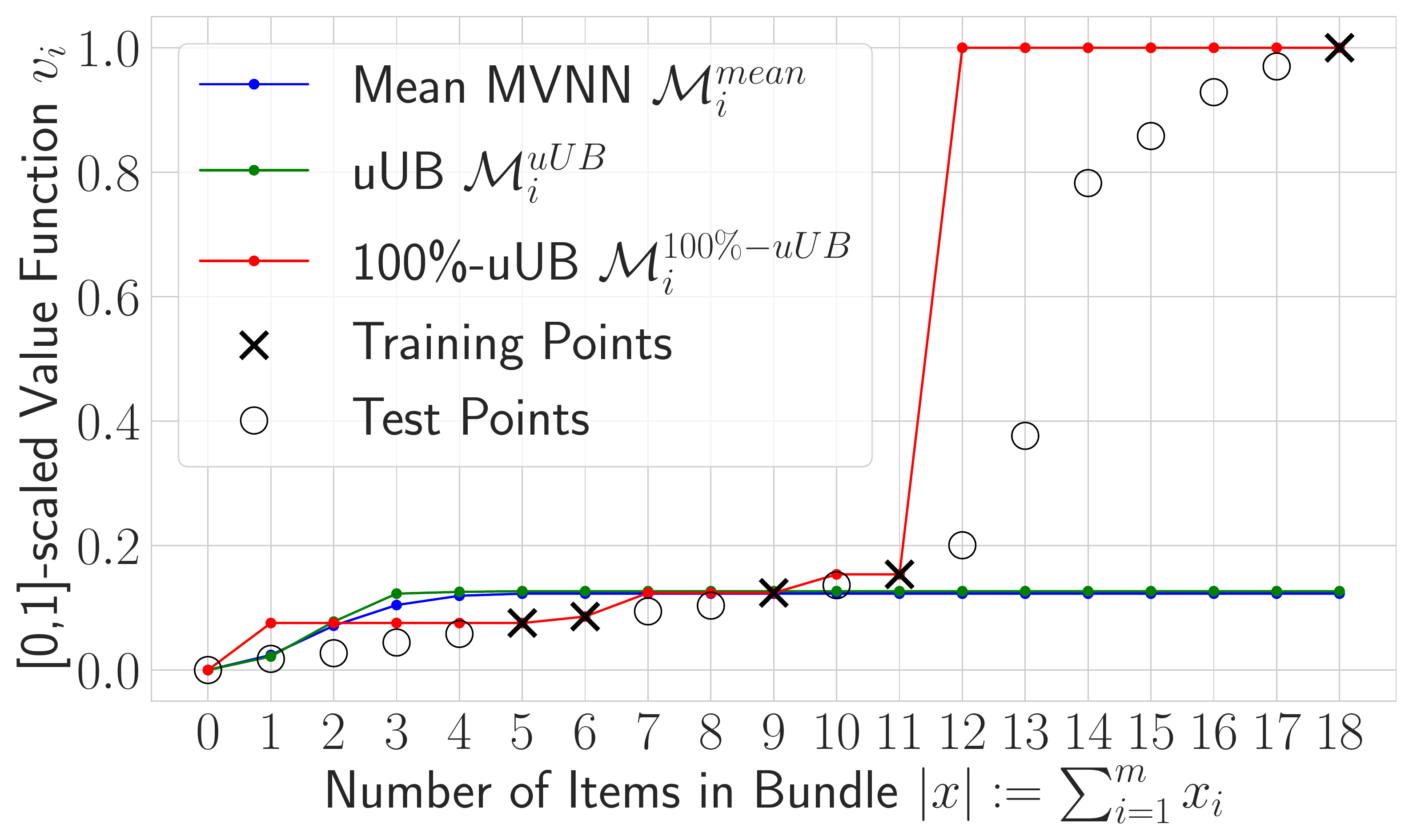}
    \caption{In contrast to our proposed initialization (see \Cref{fig:1d_path_plot}), training fails with generic initialization already for relatively small [64,64]-architectures that were used here.}
    \label{fig:1d_path_plot_initialization}
    \end{center}
\end{figure}
\paragraph{Problem.} Unfortunately, this generic initialization approach can dramatically fail for MVNNs: For any non-zero initialization, the non-negativity constraint of the weights implies that the mean $\mu_k>0$. This implies that the mean of a pre-activated neuron in the first hidden layer is $\E[\left(W^{i,1}x\right)_j]=d^{i,0}\mu_1\bar{x}$. For a generic scaling $s_k$ one would obtain $\mu_k\propto\frac{1}{\sqrt{d^{i,k-1}}}$ and thus the mean $\E[\left(W^{i,1}x\right)_j]\propto d^{i,0}\frac{1}{\sqrt{d^{i,0}}}\bar{x}={\sqrt{d^{i,0}}}\bar{x}$ of the pre-activated neurons diverges to infinity with a rate of ${\sqrt{d^{i,0}}}$ as $d^{i,0}\to \infty$. Analogously, the pre-activated neurons of every layer diverge to infinity as $d^{i,k-1}\to \infty$.  This is particularly problematic for bReLUs (as used in MVNNs) as their gradient is zero on $[0,t^{i,k}]^c$. \Cref{fig:1d_path_plot_initialization} shows that both MVNNs $\meanMi$ and $\uUBMi$ get ``stuck.'' This happens because already at initialization, every neuron in the first hidden layer has a pre-activation that is larger than $t^{i,1}$ for every training point.

This could be solved by scaling down the initial weights even more, e.g., $W^{i,k}_{j,l}\sim\text{Unif}[0,\frac{2}{d^{i,k-1}}]$ resulting in $\mu_k=\frac{1}{d^{i,k-1}}$. However, since for $W^{i,k}_{j,l}\sim\text{Unif}[0,\frac{2}{d^{i,k-1}}]$ it holds that $\sigma^2_k\propto \frac{1}{(d^{i,k-1})^2}$, this induces a new problem of vanishing conditional variance $\Varco{\left(W^{i,k}z^{i,k-1}\right)_j}{z^{i,k-1}}$ with a rate of $\BigO(\frac{1}{d^{i,k-1}})$ for wide (i.e., $d^{i,k-1}$ large) MVNNs. Overall, it is impossible to simultaneously solve both problems by just scaling the distribution by a factor $s_k$, because the conditional mean $\Eco{\left(W^{i,k}z^{i,k-1}\right)_j}{z^{i,k-1}}$ scales with $s_k \cdot d^{i,k-1}$ and the conditional variance $\Varco{\left(W^{i,k}z^{i,k-1}\right)_j}{z^{i,k-1}}$ scales with $s_k^2 \cdot d^{i,k-1}$. Thus, for wide MVNNs, one of those two problems (i.e., either diverging expectation or vanishing variance) would persist.

\paragraph{Solution.} We introduce a new initialization method that solves \emph{both} problems at the same time. For this, we propose a mixture distribution of two different uniform distributions (see \Appendixref[Appendix ]{def:mixtureDistribution}{Definition~E.1}). For each layer $k$, we independently sample all weights $W^{i,k}_{jl}$ i.i.d.\ with probability $(1-p_k)$ from $\Unif[0,A_k]$, and with probability $p_k$ from $\Unif[0,B_k]$.
If we choose $p_k$ and $A_k$ small enough, we can get arbitrarily small $\mu_k$ while not reducing $\sigma_k$ too much. In \Appendixref{sec:apendix_New MVNN Random Initialization}{Appendix~E}, we provide formulas for how to choose $A_k$, $B_k$ and $p_k$ depending on $d^{i,k-1}$. In \Appendixref{thm:ScalingInitConstantEandV}{Theorem~3} in \Appendixref{sec:apendix_New MVNN Random Initialization}{Appendix~E}, we prove that, if the parameters are chosen in this way, then the conditional mean and conditional variance neither explode nor vanish with increasing $d^{i,k-1}$ but rather stay constant for large $d^{i,k-1}$. Note that, in \Cref{fig:1d_path_plot}, for $\meanMi$ and $\uUBMi$, we used our proposed initialization method for suitable $A_k$, $B_k$ and $p_k$, such that the problem induced by a generic initialization from \Cref{fig:1d_path_plot_initialization} is resolved.

\subsection{Mixed Integer Linear Program (MILP)}\label{subsec:New MVNN-based MIP}
A key step in ML-powered iterative combinatorial assignment mechanisms is finding the (predicted) social welfare-maximizing allocation, i.e., solving the \emph{ML-based WDP}. Thus, a key requirement posed on any acquisition function $\mathcal{A}$ in such a mechanism is to be able to efficiently solve $\max\limits_{a\in \F} \mathcal{A}(a)$. Recall that, to define our acquisition function $\mathcal{A}$, we use $\mathcal{A}(a) = \sum_{i \in N}\Ai{a_i}$ where the $\mathcal{A}_i$'s are bidder-specific upper uncertainty bounds. Thus, the ML-based WDP becomes
\begin{align}\label{eq:ML-WDP}
\max_{a\in \F} \sum_{i \in N}\Ai{a_i}.
\end{align}

\cited{weissteiner2022monotone} proposed a MILP for MVNNs with $\Ai{}\coloneqq\MVNNi{}$ to efficiently solve \cref{eq:ML-WDP}. Their MILP was based on a reformulation of the $\min(\cdot,\cdot)$ and $\max(\cdot,\cdot)$ in the bReLU activation $\min(\max(\cdot,0),t)$. Thus, it required twice the number of binary variables \emph{and} linear constraints as for a plain ReLU-NN. Since we use an MVNN-based uUB $\Ai{}\coloneqq\uUBMi{}$ to define our acquisition function, we could directly use their MILP formulation. However, instead, we propose a new MILP, which is significantly more succinct.
For this, let $o^{i, k}\coloneqq W^{i, k}z^{i, k-1} + b^{i, k}$ be the \emph{pre}-activated output and $z^{i, k}\coloneqq\varphi_{0,t^{i,k}}(o^{i, k})$  be the output of the $k$\textsuperscript{th} layer with $l^{i, k}\le o^{i, k} \le u^{i, k}$, where the tight lower (upper) bound $l^{i, k}$ ($u^{i, k})$ is derived by forward-propagating the empty (full) bundle \cite[Fact 1]{weissteiner2022monotone}. In \Cref{thm:milp}, we state our new MILP (see \Appendixref{subsec:appendix_proof_MILP}{Appendix~F.1} for the proof).\footnote{All vector inequalities should be understood component-wise.}
\begin{theorem}[MVNN MILP Tailored to Combinatorial Assignment]\label{thm:milp}
Let $\Ai{}=\uUBMi$ be our MVNN-based uUBs. The ML-based WDP~\eqref{eq:ML-WDP} can be formulated as the following MILP:
    \begin{align}
        &\max\limits_{a\in \F, z^{i,k},\alpha^{i,k},\beta^{i,k}}\left\{\sum_{i \in N} W^{i, K_{i}} z^{i, K_{i}-1}\right\}\label{eq:milp_objective}\\
        &\hspace{-1cm}\text{s.t. for } i\in N \text{ and } k \in \{1,\ldots,K_i-1\}\notag\\
        &z^{i,0}=a_i\label{eq:thm(i)}\\
        &z^{i, k}\le \alpha^{i,k}\cdot t^{i, k}\label{eq:thm(ii)}\\
        &z^{i, k}\le o^{i, k} - l^{i,k}\cdot(1-\alpha^{i,k})\label{eq:thm(iii)}\\
        &z^{i, k}\ge \beta^{i,k}\cdot t^{i, k}\label{eq:thm(iv)}\\
        &z^{i, k}\ge o^{i, k} + (t^{i, k}-u) \beta^{i,k}\label{eq:thm(v)}\\
        &\alpha^{i,k}\in \{0,1\}^{d^{i,k}},\,\beta^{i,k}\in \{0,1\}^{d^{i,k}}\label{eq:lemma(vi)}
    \end{align}
\end{theorem}
Note that for each neuron of $\Ai{}=\uUBMi$, our new MILP has only $4$ linear constraints, i.e., respective components of \crefrange{eq:thm(ii)}{eq:thm(v)}, compared to $8$ in \cite{weissteiner2022monotone}. Moreover, in contrast to the MILP in \cite{weissteiner2022monotone}, our MILP does not make use of any ``big-M'' constraints, which are known to be numerically unstable.

\begin{table*}[ht]
    \renewcommand\arraystretch{1.2}
    \setlength\tabcolsep{2pt}
	\robustify\bfseries
	\centering
	\begin{sc}
	\resizebox{1\textwidth}{!}{
	\small
    \begin{tabular}{ccccccccc}
    \toprule
         &  & \multicolumn{5}{c}{\textbf{Efficiency Loss in \%\,\,\textdownarrow}} & \multicolumn{2}{c}{\textbf{T-Test for Efficiency:}}\\
        \cmidrule(l{2pt}r{2pt}){3-7}
        \cmidrule(l{2pt}r{2pt}){8-9}
    \textbf{Domain} &$\boldsymbol{\Qmax}$ &  \multicolumn{1}{c}{BOCA}  &\multicolumn{1}{c}{MVNN-MLCA}  &   \multicolumn{1}{c}{NN-MLCA}  & \multicolumn{1}{c}{FT-MLCA}& \multicolumn{1}{c}{RS}  & $\mathcal{H}_0: \mu_{\text{MVNN-MLCA}}\le\mu_{\text{BOCA}}$ & $\mathcal{H}_0: \mu_{\text{NN-MLCA}}\le\mu_{\text{BOCA}}$\\
        \cmidrule(l{2pt}r{2pt}){1-2}
        \cmidrule(l{2pt}r{2pt}){3-7}
        \cmidrule(l{2pt}r{2pt}){8-9}
    
    LSVM & 100& \ccell 0.39$\pm$0.30  &\ccell00.70$\pm$0.40 & 02.91$\pm$1.44 & 01.54 $\pm$0.65&31.73$\pm$2.15 &  $p_{\text{val}}=9\mathrm{e}{-2}$&$p_{\text{val}}=3\mathrm{e}{-4\hphantom{3}}$\\
    
    SRVM& 100&  \ccell 0.06$\pm$0.02    &00.23$\pm$0.06  & 01.13$\pm$0.22 &00.72$\pm$0.16 &28.56$\pm$1.74 & $p_{\text{val}}=5\mathrm{e}{-6}$&$p_{\text{val}}=2\mathrm{e}{-13}$\\
    
    MRVM& 100&  \ccell 7.77$\pm$0.34   &\ccell08.16$\pm$0.41 & 09.05$\pm$0.53 & 10.37$\pm$0.57 &48.79$\pm$1.13 & $p_{\text{val}}=8\mathrm{e}{-2}$&$p_{\text{val}}=2\mathrm{e}{-5\hphantom{3}}$\\
    
    \bottomrule
    \end{tabular}
}
    \end{sc}
    \vskip -0.1 in
    \caption{BOCA vs MVNN-MLCA, NN-MLCA, Fourier transform (FT)-MLCA and random search (RS). Shown are averages and a 95\% CI on a test set of $50$ instances. Winners based on a t-test with significance level of 1\% are marked in grey.}
\label{tab:efficiency_loss_mlca}
\end{table*}
\begin{figure*}
    \begin{center}
    \resizebox{1\textwidth}{!}{
    \includegraphics{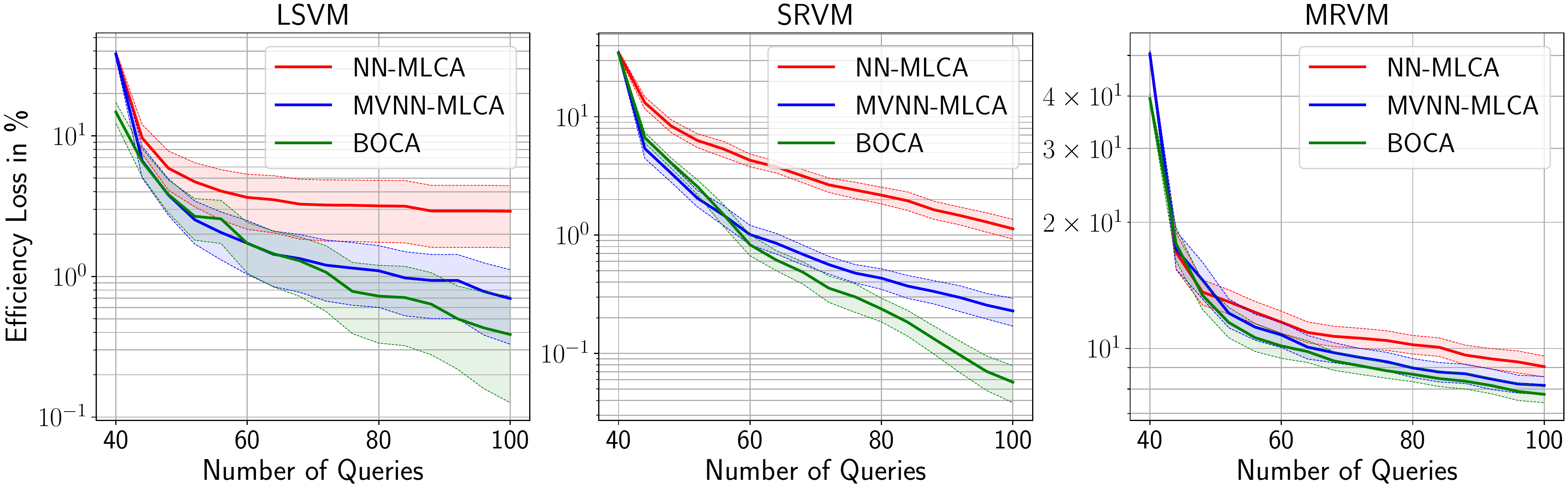}
    }
    \caption{Efficiency loss paths (i.e., regret plots) of BOCA compared to the results from \cited{weissteiner2022monotone} of MVNN-MLCA and NN-MLCA without any notion of uncertainty. Shown are averages with 95\% CIs over 50 CA instances.}
    \label{fig:efficiency_path_plot_summary}
    \end{center}
\end{figure*}
\section{Experiments}\label{sec:experiments}
In this section, we experimentally evaluate the performance of BOCA in CAs. To this end, we equip the MLCA mechanism (see \Cref{subsec:A Machine Learning powered ICA}) with our new acquisition function $\mathcal{A}(a)=\sum_{i\in N} \uUBMi(a_i)$. We compare the efficiency of BOCA against the previously proposed MVNN-based and NN-based MLCA from \cite{weissteiner2022monotone} which do not explicitly model the mechanism's uncertainty over values for not yet elicited bundles.\footnote{In these methods, uncertainty over not yet elicited bundles is only modeled via the retraining of the (MV)NNs in each round, i.e., the random parameter initialization of the (MV)NNs. This can be seen as simple form of Thompson sampling (see last paragraph in \Appendixref{sec:appendix:BO}{Appendix~C}).} We use our new parameter initialization method (\Cref{subsec:New MVNN Random Initialization}) for $\uUBMi$, and we use our new MILP (\Cref{thm:milp}) for solving the WDPs.

\paragraph{Experiment Setup.}
To generate synthetic CA instances, we use the following three domains from the spectrum auction test suite (SATS) \cite{weiss2017sats}: LSVM, SRVM, and MRVM (see \Appendixref{subsec:appendix_SATS_domains}{Appendix~G.1} for details).\footnote{We do not use GSVM, as \cited{weissteiner2022monotone} already achieved 0\% efficiency loss in GSVM via MVNN-based MLCA.} SATS gives us access to the true optimal allocation $a^*$, which we use to measure the \emph{efficiency loss}, i.e., $1-V(a^*_R)/V(a^*)$ when eliciting reports $R$ via MLCA. We report efficiency loss (and not revenue), as spectrum auctions are government-run, with a mandate to maximize welfare \citep{cramton2013spectrumauctions}. See \Appendixref{subsec:revenue}{Appendix~G.6} for a discussion of the corresponding results on revenue. To enable a fair comparison against prior work, for each domain, we use $\Qinit=40$ initial random queries (including the full bundle for the calculation of $\oneMi$) and set the query budget to $\Qmax=100$ (see \Appendixref{subsec:Reduced Number of Initial Queries}{Appendix~G.8} for results for $\Qinit=20$). We terminate any mechanism in an intermediate iteration if it already found an allocation with 0\% efficiency loss.

\paragraph{Hyperparameter Optimization (HPO).}
We use \emph{random search (RS)} \citep{bergstra2012random} to optimize the hyperparameters of the mean MVNN $\meanMi$ and of our MVNN-based uUB $\uUBMi$. The HPO includes the NN-architecture parameters, training parameters, NOMU parameters, and initialization parameters (see \Cref{subsec:New MVNN Random Initialization}).
RS was carried out independently for each bidder type and SATS domain with a budget of $500$ configurations, where each configuration was evaluated on $100$ SATS instances. For each instance, the MVNNs $\meanMi$ and $\uUBMi$ were trained on uniformly at random chosen bundle-value pairs $\Dtr$ and evaluated on a disjoint test set of different bundle-value pairs $\Dtest$. To select the winner configuration, we consider as evaluation metric the quantile-loss on the test set \emph{and} the MAE on the training set, i.e., for each configuration and instance we calculate
\begin{align}\label{eq:evaluation_metric_HPO}
&\frac{1}{|\Dtest|}\hspace{-0.1cm}\sum\limits_{(x,y)\in \Dtest}\hspace{-0.4cm}\max\{(y\hspace{-0.08cm}-\hspace{-0.08cm}\uUBMi(x))q,(\uUBMi(x)\hspace{-0.08cm}-\hspace{-0.08cm}y)(1\hspace{-0.08cm}-\hspace{-0.08cm}q)\}\hspace{-0.08cm}\notag\\
& + \text{MAE}(\Dtr),
\end{align}
which we then average over all $100$ instances. We used four quantile parameters $q\in \{0.6,0.75,0.9,0.95\}$ in \cref{eq:evaluation_metric_HPO} to achieve different levels of exploration (i.e., the resulting uUBs become larger the more we increase $q$ in \cref{eq:evaluation_metric_HPO}). This evaluation metric \emph{simultaneously} measures the quality of the uUB on the test data (via the quantile-loss) as well as the quality of the uUB predictions on the training data (via the MAE). For each quantile $q$ and SATS domain, we then proceed with the winner configuration of $\uUBMi$ and evaluate the efficiency of BOCA on a separate set of 50 instances. Details on hyperparameter ranges and the training procedure are provided in \Appendixref{subsec:appendix_hpo,subsec:Details MVNN-Training}{Appendices~G.2 and~G.3}.

\paragraph{Results.}
In \Cref{tab:efficiency_loss_mlca}, we show the average efficiency loss of each approach after $\Qmax=100$ queries (see \Appendixref{subsec:Detailed MLCA Results}{Appendix~G.5} for details). We see that BOCA significantly outperforms MVNN-MLCA \citep{weissteiner2022monotone} in SRVM, and it performs on-par in LSVM and MRVM, with a better average performance. Since MVNNs previously achieved SOTA performance, BOCA also outperforms the other benchmarks (i.e., NN \citep{weissteiner2020deep} and FT-MLCA \citep{weissteiner2022fourier}). RS's poor performance highlights the intrinsic difficulty of this task. The amount of exploration needed is domain dependent (e.g., multi-modality of the objective), which explains why the significance of BOCA's improvement varies across domains. However, our results also show that using an uUB (as in BOCA) instead of just a mean prediction (as in MVNN-MLCA) never hurts.

\Cref{fig:efficiency_path_plot_summary} shows the efficiency loss path for all domains. We see that the superior (average) performance of $\uUBMi$ does not only hold at the end of the auction (at $\Qmax=100$), but also for a large range of queries: in LSVM, BOCA is better for [70,100]; in SRVM, BOCA is significantly better for [70,100]; in MRVM, BOCA is better for [50,100]. See \Appendixref{subsec:revenue}{Appendix~G.6} for results on revenue where BOCA significantly outperforms MVNN-MLCA also for MRVM. In \Appendixref{subsec:Ablation Study - When is Exploration Needed}{Appendix~G.7}, we study to what degree BOCA's performance increase is due to (a) our uncertainty model (\Cref{subsec:NOMU for Combinatorial Assignment}) versus (b) our new parameter initialization method (\Cref{subsec:New MVNN Random Initialization}).   Finally, in \Appendixref{subsec:Reduced Number of Initial Queries}{Appendix~G.8}, we provide further experiments for a reduced number of $\Qinit=20$ initial queries, which lead to similar results as shown in \Cref{tab:efficiency_loss_mlca}.

\section{Conclusion}
In this paper, we have proposed a Bayesian optimization-based combinatorial assignment (BOCA) mechanism. On a conceptual level, our main contribution was the integration of model uncertainty over agents' preferences into ML-based preference elicitation. On a technical level, we have designed a new method for estimating an upper uncertainty bound that exploits the monotonicity of agents' preferences in the combinatorial assignment domain and the finite nature of this setting. Our experiments have shown that BOCA performs as good or better than the SOTA in terms of efficiency. An interesting direction for future work is the evaluation of BOCA in other combinatorial assignment domains, such as combinatorial exchanges or course allocation (e.g., see \cite{Soumalias2023Machine}). Finally, it would also be interesting to apply BOCA's conceptual idea in the combinatorial BO settings outside of combinatorial assignment.

\section*{Acknowledgments}
We thank the anonymous reviewers for helpful comments. This paper is part of a project that has received funding from the European Research Council (ERC)
under the European Union’s Horizon 2020 research and innovation program (Grant agreement No. 805542).

\bibliography{BOCA_references}

\if\showAppendix1
\clearpage
\appendix
\section*{Appendix}

\counterwithin{definition}{section}
\counterwithin{corollary}{section}
\counterwithin{problem}{section}
\counterwithin{example}{section}
\counterwithin{remark}{section}
\counterwithin{fact}{section}

\section{A Machine Learning-Powered ICA}\label{sec:appendix_A Machine Learning powered ICA}
In this section, we present in detail the \textit{machine learning-powered combinatorial auction (MLCA)} by \cited{brero2021workingpaper}.

At the core of MLCA is a \textit{query module} (Algorithm~\ref{alg:QueryModule}), which, for each bidder $i\in I\subseteq N$, determines a new value query $q_i$. First, in the \textit{estimation step} (Line 1), an ML algorithm $\mathcal{A}_i$ is used to learn bidder $i$'s valuation from reports $R_i$. Next, in the \textit{optimization step} (Line 2), an \textit{ML-based WDP} is solved to find a candidate $q$ of value queries. In principle, any ML algorithm $\mathcal{A}_i$ that allows for solving the corresponding ML-based WDP in a fast way could be used. Finally, if $q_i$ has already been queried before (Line 4), another, more restricted ML-based WDP (Line 6) is solved and $q_i$ is updated correspondingly. This ensures that all final queries $q$ are new.
\setlength{\textfloatsep}{5pt}
\begin{algorithm}[h!]
        \DontPrintSemicolon
        \SetKwInOut{inputs}{Inputs}
        \inputs{~Index set of bidders $I$ and reported values $R$}
    \lForEach(\Comment*[f]{\color{blue}Estimation step}){$i \in I$}{
    \hspace{-0.07cm}Fit $\mathcal{A}_i$ on $R_i$: $\mathcal{A}_i[R_i]$
    }
    Solve $q \in \argmax\limits_{a \in {\F}}\sum\limits_{i \in I} \mathcal{A}_i[R_i](a_i)$ \hspace{-0.03cm}\Comment*[r]{\color{blue}Optimization step}
    \ForEach{$i \in I$}{
        \If(\Comment*[f]{\color{blue} Bundle already queried}){$(q_i,\hvi{q_i})\in R_i$}{ 
        Define $\pF=  \{a\in \F : a_i \neq x, \forall (x,\hvi{x})\in R_i\}$\;
        Re-solve $\pq \in \argmax_{a \in \pF}\sum_{l \in I} \mathcal{A}_l[R_l](a_l)$\;
        Update $q_i = \pqi\;$
        }
    }
    \Return{Profile of new queries $q=(q_1,\ldots,q_n)$}
    \caption{\textsc{NextQueries}$(I,R)$\, {\scriptsize (Brero et al. 2021)}}
    \label{alg:QueryModule}
\end{algorithm}

In Algorithm~\ref{MLCA}, we present \textsc{Mlca}. In the following, let $R_{-i}=(R_1,\ldots,R_{i-1},R_{i+1},\ldots, R_n)$. \textsc{Mlca} proceeds in rounds until a maximum number of queries per bidder $\Qmax$ is reached. In each round, it calls Algorithm \ref{alg:QueryModule}  $(\Qround-1)n+1$ times: for each bidder $i\in N$, $\Qround-1$ times excluding a different bidder $j\neq i$ (Lines 5--10,  sampled \textit{marginal economies}) and once including all bidders (Line 11, \textit{main economy}). In total each bidder is queried $\Qround$ bundles per round in \textsc{MLCA}. At the end of each round, the mechanism receives reports $\Rnew$ from all bidders for the newly generated queries $\qnew$ and updates the overall elicited reports $R$ (Lines 12--14). In Lines 16--17, \textsc{Mlca} computes an allocation $a^*_R$ that maximizes the \emph{reported} social welfare (see \Cref{WDPFiniteReports}) and determines VCG payments $p(R)$ based on the reported values $R$ (see \Appendixref[Appendix ]{def:vcg_payments}{Definition~B.1}).
\setlength{\textfloatsep}{5pt}
\begin{algorithm}[t!]
        \DontPrintSemicolon
        \SetKwInOut{parameters}{Params}
        \parameters{$\Qinit,\Qmax,\Qround$ {initial, max and \#queries/round}}
    \ForEach{$i \in N$}{Receive reports $R_i$ for $\Qinit$ randomly drawn bundles}
    \For(\Comment*[f]{\hspace{-0.05cm}\color{blue}Round iterator}){$k=1,...,\floor{(\Qmax-\Qinit)/\Qround}$}{
        \ForEach(\Comment*[f]{\color{blue}Marginal economy queries}){$i \in N$}{
            {Draw uniformly without replacement $(\Qround\hspace{-0.1cm}-\hspace{-0.1cm}1)$ bidders from $N\setminus\{i\}$ and store them in $\tilde{N}$}\;
            \ForEach{$j \in \tilde{N}$}{
            $\qnew=\qnew\cup$ \textit{NextQueries$(N\setminus\{j\},R_{-j})$}
            }
        }
        $\qnew=$ \textit{NextQueries$(N,R)$} \Comment*[r]{\color{blue}Main economy queries}
        \ForEach{$i \in N$}{
         Receive reports $\Rnewi$ for $\qnew_i$, set $R_i=R_i\cup\Rnewi$
        }
    }
    Given elicited reports $R$ compute $a^*_{R}$ as in \Cref{WDPFiniteReports}\;
    Given elicited reports $R$ compute VCG-payments $p(R)$\;
    \Return{Final allocation $a^*_{R}$ and payments $p(R)$}
    \caption{\small \textsc{Mlca}($\Qinit,\Qmax,\Qround$)\, {\scriptsize (Brero et al. 2021)}}
    \label{MLCA}
\end{algorithm}

In this paper, we consider the following two minor adaptations of the generic \textsc{Mlca} mechanism described above:
\begin{enumerate}
    \item \textbf{Balanced and global marginal economies:} In Lines 5--6 of \Cref{MLCA}, \textsc{Mlca} draws \emph{for each bidder $i\in N$} uniformly at random a set of marginal economies $\tilde{N}$ to generate queries in this marginal economy. However, this implies that at the end of the auction it only holds on average that the same number of queries is asked from each bidder in each marginal economy. Moreover, since $\tilde{N}$ is re-drawn for each bidder this creates an computational overhead, since typically the WDPs in Line 8 in \textit{NextQueries} needs to be solved more often. For example consider the case $N=\{1,2,3,4,5,6\}, \Qround=3$ and that for bidder $1$ the $\Qround-1=2$ sampled marginal economies were given as $\tilde{N}=\{3,4\}$, whilst for bidder $2$, $\tilde{N}=\{5,6\}$, and for bidder $3$, $\tilde{N}=\{1,2\}$. In this case, the WDPs in \textit{NextQueries} would need to be solved $6$ times, which is the maximum possible. In our implementation, we change the following two things: First, we reduce the computational overhead by once globally selecting $\Qround$ marginal economies in each iteration, i.e., we select a set $\tilde{N}_{global}$ consisting of $\Qround$ marginal economies before Line 5, and then select $\tilde{N}$ for each bidder $i\in N$ in the loop in Line 5 as admissible subset of size $\Qround-1$ of $\tilde{N}_{global}$. In the above example, if $\tilde{N}_{global}=\{3,4,1\}$, then this ensures that only $\Qround=3$ WDPs in the marginal economies are solved in one iteration. Second, we do not determine $\tilde{N}_{global}$ uniformly at random, but ensure that at the end of the auction each marginal economy was selected equally often up to a difference in counts of at most one.
    \item \textbf{Single training per iteration:} To further reduce computational overhead, we train each bidder's ML algorithm $\mathcal{A}_i$ once at the beginning of each iteration, and then only select in \textit{NextQueries} the trained $\mathcal{A}_i$ corresponding to the active set of bidders $I$. This reduces the amount of total training procedures per iteration from (worst case) $n^2$ to $n$.
\end{enumerate}

\section{Incentives of MLCA}\label{sec:appendix_Incentives of MLCA}
In this section, we review the key arguments by \cited{brero2021workingpaper} why MLCA has good incentives in practice. First, we define VCG-payments\footnote{VCG is an abbreviation for \enquote{Vickrey--Clarke--Groves}. For the VCG-mechanism (which can be seen as a generalization of the second prize auction) it is always optimal for the bidders to report the truth.} given bidder's reports.

\begin{definition}{\textsc{(VCG Payments from Reports)}}\label{def:vcg_payments}
Let $R=(R_1,\ldots,R_n)$ denote an elicited set of reported bundle-value pairs from each bidder obtained from \textsc{Mlca} (\Cref{MLCA}) and let $R_{-i}\coloneqq(R_1,\ldots,R_{i-1},R_{i+1},\ldots,R_n)$. We then calculate the VCG payments $p(R)=(p(R)_1\ldots,p(R)_n) \in \R_+^n$ as follows:
\begin{align}\label{VCGPayments}
&p(R)_i \coloneqq \hspace{-0.2cm}\sum_{j \in N \setminus \{i\}} \hvj{}{}\left(\left(a^*_{R_{-i}}\right)_j\right) - \hspace{-0.2cm}\sum_{j \in N \setminus \{i\}}\hvj{}{}\left(\left(a^*_{R}\right)_j\right).
\end{align}
where $a^*_{R_{-i}}$ maximizes the reported social welfare (SW) when excluding bidder $i$, i.e.,
\begin{align}
&a^*_{R_{-i}}\in \argmax_{a \in \F} \hV{a|R_{-i}} = \argmax_{a \in \F}\hspace{-0.4cm}\sum_{\substack{j \in N\setminus\{i\}:\\ \left(a_j,\hvj{}(a_j)\right)\in R_j}}\hspace{-0.4cm}\hvj{}(a_j),
\end{align}
and $a^*_R$ is a reported-social-welfare-maximizing allocation (including all bidders), i.e,
\begin{align}
&a^*_{R}\in \argmax_{a \in \F} \hV{a|R} = \argmax_{a \in \F} \hspace{-1.3cm}\sum_{\hspace{1cm}i \in N:\, \left(a_i,\hvi{}(a_i)\right)\in R_i}\hspace{-1.3cm}\hvi{}(a_i).
\end{align}
\end{definition}

Therefore, when using VCG, bidder $i$'s utility is:
{\small\begin{align*}
u_i \hspace{-0.05cm}=&v_i((a_R^*)_i)-p(R)_i\\
=&\hspace{-0.05cm}  \underbrace{v_i((a_R^*)_i) + \hspace{-0.3cm}\sum_{j \in N \setminus \{i\}}\hspace{-0.25cm}\hvj{}((a^*_{R})_j)}_{\textrm{\scriptsize (a) Reported SW of main economy}} - \hspace{-0.1cm} \hspace{-0.1cm}\underbrace{\sum_{j \in N \setminus \{i\}}\hspace{-0.25cm} \hvj{}((a^*_{R_{-i}})_j).}_\textrm{{\scriptsize (b) Reported SW of marginal economy}}
\end{align*}}
Any beneficial misreport must increase the difference (a) $-$ (b). 

MLCA has two features that mitigate manipulations. First, MLCA explicitly queries each bidder's marginal economy (\Cref{MLCA}, Line 5), which implies that (b) is practically independent of bidder $i$'s bid (Section 7.3 in \cite{brero2021workingpaper} provides experimental support for this). Second, MLCA enables bidders to ``push'' information to the auction which they deem useful. This mitigates certain manipulations that target (a), as it allows bidders to increase (a) with truthful information. \cited{brero2021workingpaper} argue that any remaining manipulation would be implausible as it would require almost complete information.\footnote{In this paper, we propose a new method that uses a notion of epistemic uncertainty to actively explore regions of the bundle space with high uncertainty. Intuitively, this makes manipulation even harder, since additional exploration makes it more difficult for a bidder to prevent other bidders from getting queries in certain regions.}

If we are willing to make further assumptions, we also obtain two theoretical incentive guarantees: 
\begin{itemize}
\item Assumption 1 requires that, for all bidders $i\in N$, if all other bidders report truthfully, then the reported social welfare of bidder $i$'s marginal economy (i.e., term (b)) is \emph{independent} of her value reports. 
\item Assumption 2 requires that, if all bidders $i\in N$ bid truthfully, then MLCA \emph{finds an efficient allocation}. 
\end{itemize}
\paragraph{Result 1: Social Welfare Alignment} If Assumption 1 holds, and if all other bidders are truthful, then MLCA is social welfare aligned, i.e., increasing the reported social welfare of $a_R^*$ in the main economy (i.e., term (a)), which in this case equals the \emph{true} social welfare of $a_R^*$, is the only way for a bidder to increase her true utility \citep[Proposition 3]{brero2021workingpaper}.

\paragraph{Result 2: Ex-Post Nash Equilibrium} Moreover, if Assumption 1 \emph{and} Assumption 2 hold, then bidding truthfully is an ex-post Nash equilibrium in MLCA \citep[Proposition 4]{brero2021workingpaper}.

\section{BO Perspective of ICAs}\label{sec:appendix:BO}
In this section, we discuss the Bayesian optimization (BO) perspective of iterative combinatorial assignment (see \Cref{subsec:app:Bayesian Optimization Background}). Specifically, we analyze the MLCA mechanism (see \Cref{subsec:A Machine Learning powered ICA} for an overview or \Cref{sec:appendix_A Machine Learning powered ICA} for a detailed description) in the light of BO.

Iterative combinatorial assignment can be seen as a combinatorial BO task with an expensive-to-evaluate function:
\begin{itemize}
    \item The objective (e.g., social welfare) in general lacks known structure and when evaluating it (e.g., value queries) one only observes the objective at a single input point and no derivatives such that gradient-based optimization cannot be used.
    \item Typically one can only query a very limited amount of information to find an approximately optimal allocation, For example, in a real-world spectrum auction, the auctioneer can only ask each bidder to answer on the order of 100 value queries for different bundles of items, even though the space of possible bundles is exponential in the number of items $m$, i.e., there are $2^m$ possible bundles and $(n+1)^m$ possible allocations.
\end{itemize}

\subsection{BO Perspective of MLCA}
In this paper, we extend prior work on MLCA with a notion of uncertainty that makes MLCA more similar to classic BO. Specifically, we now use an MVNN-based upper uncertainty bound (uUB) to define our acquisition function. This allows MLCA to trade-off exploration and exploitation making it more likely to find optimal allocations.

However, in addition to the challenges that arise in BO, combinatorial assignment adds their own set of challenges. For example, Gaussian process-based BO often does not extend beyond 10-20 input dimensions, which is problematic as in combinatorial assignment the input space can be much larger, e.g., for $m=98$ items and $n=10$ bidder the input space would be $980$ dimensional (MRVM). In addition, integrality constraints to obtain only whole items (i.e., combinatorial assignment deals with assigning $m$ \emph{indivisible} items to agents) and feasibility constraints that ensure each item is only allocated once to a single agent are often only incorporated via rounding or randomization.

Our work addresses both problems by combining (MV)NN-based MLCA by~\citep{weissteiner2020deep, weissteiner2022monotone} with the recently introduced NOMU~\citep{heiss2022nomu}, an optimization-based method to obtain uncertainty estimates for the predictions of NNs. Importantly, NOMU enables to represent an uUB as a \emph{single} NN in contrast to other more expensive methods such as ensembles \citep{lakshminarayanan2017simple}. Thus, NOMU is particularly suited for iterative combinatorial assignment, where uUB-based WDPs are solved hundreds of times to generate informative queries.

\subsection{How does BOCA Differ from Classical BO?}
There are two main differences that set BOCA from classical BO apart.
\begin{enumerate}
\item \textbf{More information per query:} In classical BO, one would obtain only one number $f(x)$ per query, which would correspond to $f(x)=\sum_{i \in N}\hvi{a_i}$ in the case of ICA. However, we obtain all the individual values $\hvi{a_i}$ instead of only obtaining their sum. This additional information is very valuable for MLCAs such as BOCA. We benefit from this additional information by representing our acquisition function $\mathcal{A}(a)=\sum_{i\in N}\mathcal{A}_i(a_i)$ as sum of functions $\mathcal{A}_i$, which are trained based on their individual values. The benefits of this additional information are even more valuable because of our very strong prior knowledge on the individual value functions $\hat{v}_i$ (e.g., monotonicity) compared to rather vague prior on their sum. Also our notion of uncertainty incorporates this additional information.
\newline
In classical BO $\argmax_{x\in \{x^{(1)},\ldots,x^{(T)}\}}f(x)$ is outputted as an approximation for $\argmax_{x\in X}f(x)$. However, the fact that we revive the individual values~$\hvi{a_i}$, allows us to output the solution $a^*_{R}$ of \eqref{WDPFiniteReports} instead of just outputting the best allocation that we have queried. In other words \eqref{WDPFiniteReports} allows us to recombine queried bundles potentially differently than in any queried \emph{allocation}.
\item \textbf{Multiple objectives:} In classical BO there is only one objective function $f$. In the case of MLCA, this corresponds to the primary objective $f(x)=\sum_{i \in N}\hvi{a_i}$. However, besides social welfare, we also have the secondary objective of achieving a high revenue in the case of MLCA. The revenue depends on the payments $p(R)$ from \Cref{def:vcg_payments}. In order to compute the payments $p(R)$, we need to approximate $n$ further maximization problems (i.e., the $n$ marginal economies $\sum_{j \in N\setminus\{i\}}\hvj{a_j}$, see \Cref{sec:appendix_A Machine Learning powered ICA,sec:appendix_Incentives of MLCA}). Since all this $n+1$ objectives are different sub-sums of the same $n$ unknown functions $\hat{v}_i$ each bundle-value-pair $\left(a_i,\hvi{a_i}\right)$ can provide useful information for a better approximation of $n$ objectives.
Intuitively speaking, we think that in a multi-objective setting, exploration is particularly useful (compared to exploitation), since the uncertainty of the individual value functions are independent of the different objectives. If a bundle of high uncertainty is queried one gains plenty of new information which can be very helpful for the better approximation of $n$ objectives. Whereas, very specialized queries that seem to be optimal with respect to one specific objective might not be as helpful for the other objectives in general.
\end{enumerate}
\subsection{Choice of Acquisition Function}\label{subsec:app:Choice of Acquisition function}
Recall that in BOCA we use as acquisition function $\mathcal{A}(a) = \sum_{i \in N}\Ai{a_i}$ where the $\mathcal{A}_i$'s are bidder-specific MVNN-based upper uncertainty bounds (uUBs). Thus, the WDP $\max_{a\in \F} \mathcal{A}(a)$, which we solve to generate new informative queries, decomposes to
\begin{align}\label{eq:appML-WDP}
\max_{a\in \F} \mathcal{A}(a)=\max_{a\in \F} \sum_{i \in N}\Ai{a_i}.
\end{align}
The key design choice to make solving this WDP practically feasible, is that we were able to encode each uUB $\Ai{}$ as a succinct MILP such that the acquisition function $\mathcal{A}(a) = \sum_{i \in N}\Ai{a_i}$ as a whole is also MILP-encodable and thus the WDP can be reformulated itself as a MILP. 
Importantly, our acquisition function $\mathcal{A}$ was carefully chosen to make solving this WDP practically feasible for moderate-sized MVNNs. In the following, we briefly discuss the usage (i.e., advantages and disadvantages) of other popular acquisition functions in the combinatorial assignment setting.

Popular examples of acquisition functions in classical BO from the literature include:
    \begin{itemize}[leftmargin=*]
        \item \emph{Upper uncertainty bound} (uUB) (aka \emph{upper confidence bound or upper credible bound (UCB)})~(e.g., see \citep{Srinivas_2012}),
        \item \emph{Expected improvement}~(e.g., see \citep[Section 4.1]{frazier2018tutorial}),
        \item \emph{Thompson sampling}~(e.g., see \citep{Chapelle2011AnEE}),
        \item \emph{Probability of improvement} (e.g., see \citep[Section 2.2.4]{de2021greed}),
        \item \emph{Expected value of information} (e.g., see \citep{guo2010gaussian}).
    \end{itemize}
    \subsubsection{Upper Uncertainty Bound (uUB)}
    We decided to use an UCB-type acquisition function since this enables a succinct MILP formulation of the acquisition function optimization (i.e., the WDP). Moreover, this also offers good performance in practice and good intuitive and  theoretical motivation \citep{Srinivas_2012}.
    
    \subsubsection{Expected improvement (EI)}
    While it would be theoretically optimal to use EI for the very last query, it is neither theoretically nor empirically clear if EI is better or worse than the uUB for all other queries. We are not aware of any practically feasible solver in the combinatorial assignment domain of the corresponding WDP based on EI and on a monotonic regression technique. Moreover, we do not see any straightforward practically feasible implementation and we do not expect any significant improvement from EI over uUB. 

\subsubsection{Thompson Sampling} Recently, \cited{pmlr-v162-papalexopoulos22a} also proposed a MILP-based BO method using a ReLU NN as surrogate model and Thompson sampling as acquisition function. Concretely, \cited{pmlr-v162-papalexopoulos22a} approximate Thompson sampling via retraining of the NN from scratch with a new random initialization. Subsequently, they determine the next query by solving an NN-based MILP. However, unlike our proposal,~\cited{pmlr-v162-papalexopoulos22a} do only implicitly integrate a very limited notion of uncertainty, i.e., via a random parameter initialization of the NN, making their approach conceptually equivalent to the (MV)NN-based MLCA by~\citet{weissteiner2020deep,weissteiner2022monotone}.
In particular, this approximation of Thompson sampling only achieves sufficient exploration if the diversity induced by different random initialization seeds is large enough. However, \citet[Remark B.5]{heiss2022nomu} showed that the diversity of an NN ensemble (in noiseless settings) is rather small and in \citep[Remark B.5]{heiss2022nomu} the ensemble's uncertainty needed to be scaled up by a factor of $\sim\!\!10$ to make its uncertainty competitive. 
However, in Thompson sampling the desired amount of uncertainty/exploration cannot be easily calibrated/scaled. Moreover, our experimental results in spectrum auctions suggests that indeed the method by \citet{weissteiner2022monotone}, which uses conceptually the same notion of uncertainty as in \citep{pmlr-v162-papalexopoulos22a}, is outperformed by our proposal.

\subsubsection{Probability of Improvement (PI)}
Intuitively we believe that PI is inferior to EI and uUB. We are not aware of any theoretical or empirical results that suggest significant advantages of PI. We are not aware of any practically feasible solver of the corresponding WDP based on PI and on a monotonic regression technique. 

\subsubsection{Expected Value of Information (EVOI)}
All previously mentioned acquisition functions are still to some extend myopic/greedy, as they do not optimize explicitly for multiple queries ahead (while still being significantly less myopic than just using the mean prediction as acquisition function). E.g., as mentioned above, EI is only optimal for the last query. The other acquisition functions are not mathematically optimal in any provable sense.
In theory, EVOI is optimal for a horizon of two BO-steps. In principle, one could adapt it recursively to be optimal also for longer horizons, i.e. multiple BO-steps. While this approach would be the most appealing from a purely theoretical point of view, it would also be by far the hardest from a computational point of view (because of the nested EI-like optimization problems). We are not aware of any practically feasible solver of the corresponding WDP based on EVOI and on a monotonic regression technique. We believe this would be a very exciting direction for long-term future research. For example, a very weak approximation of an EVOI-WDP could outperform a good approximation of our uUB-WDP, especially if compute power significantly increases in the future.

\section{NOMU for Monotone NNs}\label{sec:appendix_NOMU for Combinatorial Assignment}
In this section, we give more details regarding our uncertainty estimates for monotonically increasing functions based on NOMU (see \Cref{subsec:NOMU for Combinatorial Assignment}).
\subsection{Proof of \Cref{proposition:100uUB as MVNN}}\label{proof:proposition:100uUB as MVNN}
The 100\%- uUB~$\fuUB(x)\coloneqq \sup_{f\in\HC_{\Dtr}}f(x)$ is defined via a $2^m$-dimensional optimization problem with $\ntr$ constraints. In the following proof, we analytically derive the explicit closed-form joint solution of $2^m$ such optimization problems (one for each $x$), which we can represent as an MVNN that does not require \emph{any} optimization algorithm.
\begin{proof}[Proof of \Cref{proposition:100uUB as MVNN}]\label{appproof:proposition:100uUB as MVNN}
In the \ref{itm:proof:inVDtrain}\textsuperscript{st} part of the proof (which is based on the proof of \citet[Theorem 1]{weissteiner2022monotone}), we explicitly construct $\oneMi$ and show that $\oneMi\in \Vmon_{\Dtr}$.
\Cref{eq:last_term} gives our explicit (fast to evaluate) closed-form formula for $\oneMi$, which can be directly written down without any training algorithm.

The \ref{itm:proof:is100uUB}\textsuperscript{nd} part of the proof (which is new) shows that $\oneMi$ is indeed the 100\%-uUB by showing that it is maximal (and thus the supremum is actually a maximum).
\begin{enumerate}
    
    
    \item\label{itm:proof:inVDtrain} 
    
    Let $\hvi{} \in \Vmon$ and let $\Dtr=\{(x^{(l)},\hvi{x^{(l)}})\}_{l=1}^{\ntr}$ be a set of $\ntr$ observed training points corresponding to $\hvi{}$.
    First, given $\Dtr$, we construct an MVNN $\oneMi{}$ with weights $\theta=(W^i_{\Dtr},b^i_{\Dtr})$ such that $\oneMi(x)=\fuUB(x)$ for all $x\in \X$.
    
    Recall, that by definition $\hvi{(0,\ldots,0)}=0$ and that we assume that $((1\ldots,1),\hvi{(1\ldots,1)})\in \Dtr$.
    Now let $(w_l)_{l=0}^{\ntr}$ denote the observed/known values corresponding to $\hvi{}$ sorted in increasing order, i.e, let
    $x^{(0)}=(0,\ldots,0)$ with 
    \begin{align}\label{eq:w_1}
    w_0\coloneqq \hvi{x^{(0)}}=0,
    \end{align}
    let $x^{(\ntr)}=(1,\ldots,1)$ with \begin{align}\label{eq:w_2m}
    w_{\ntr}\coloneqq \hvi{x^{(\ntr)}},
    \end{align}
    and $x^{(j)},x^{(k)} \in \X\setminus\{x^{(0)},x^{(\ntr)}\}$\ for $0< j < k \le \ntr-1$ with
    \begin{align}\label{eq:w_remaining}
    w_j\coloneqq \hvi{x^{(j)}}\, \leq\, w_k\coloneqq \hvi{x^{(k)}}.
    \end{align}
    In the following, we slightly abuse the notation and write for $x^{(j)},x^{(k)}\in \X$, $x^{(j)}\subseteq x^{(k)}$ iff for the corresponding sets $A^j, A^k\in 2^M$ it holds that $A^j\subseteq A^k$. Furthermore, we denote by $ \left<\cdot,\cdot\right>$ the Euclidean scalar product on $\R^m$. 
    Before we show that our construction fulfills $\oneMi\in\Vmon_{\Dtr}$, we define it as
    {\small
    \begin{align}
    &\oneMi(x)\coloneqq 
    \hspace{-0.05cm}\sum_{k=0}^{\ntr-1}\hspace{-0.05cm} \left(w_{k+1}-w_{k}\right)\1{\forall j\in\{0,\dots,k\}\,:\,x\not\subseteq x^{(j)}}\label{eq:firstTermoneMi}\\
    &=\sum_{k=0}^{\ntr-1}\hspace{-0.05cm} \left(w_{k+1}-w_{k}\right)		  \phiu_{0,1}{}\hspace{-0.05cm}\left(\sum_{j=0}^{k}	\phiu_{0,1}{}\left( \left<1-x^{(j)},x\right>\right)-k\right)\label{eq:last_term},	      
    \end{align}
    }%
    where the second equality follows since
    {\small
    \begin{align}
        x\not\subseteq x^{(j)} & \iff \left<1-x^{(j)},x\right> \ge1\\
        &\iff \phiu_{0,1}{}\left( \left<1-x^{(j)},x\right>\right)=1,
    \end{align}
    }%
    which implies that
    \begin{align}
        &\forall j\in\{0,\dots,k\}: x\not\subseteq x^{(j)}\\ \iff& \sum_{j=0}^{k}\phiu_{0,1}{}\left( \left<1-x^{(j)},x\right>\right)=k+1,
    \end{align}
    and
    {\small
    \begin{align}
        \1{\forall j\in\{0,\dots,k\}\,:\,x\not\subseteq x^{(j)}}=\phiu_{0,1}{}\left(\sum_{j=0}^{k}\phiu_{0,1}{}\left( \left<1-x^{(j)},x\right>\right)-k\right).
    \end{align}
    }%
    We now show that $\oneMi$ is actually an MVNN.
    \Cref{eq:last_term} can be equivalently written in matrix notation as
    
    \resizebox{!}{0.18\columnwidth}{\parbox{\columnwidth}{
    \begin{align*}
    \hspace{-0.1cm}\underbrace{\begin{bmatrix} w_1 - w_0  \\ w_2 - w_1  \\ \vdots \\ w_{\ntr} - w_{\ntr - 1} \end{bmatrix}^\top}_{(W^{i,3}_{\Dtr})^\top\in \R_{\ge0}^{\ntr}}\hspace{-0.3cm}\phiu_{0,1}{}\hspace{-0.1cm}\left(\hspace{-0.1cm}W^{i,2}_{\Dtr} \phiu_{0,1}{}\hspace{-0.1cm}\left(\underbrace{\begin{bmatrix} 1- x^{(0)} \\1- x^{(1)} \\ \vdots \\ 1 - x^{(\ntr-1)}\end{bmatrix}}_{W^{i,1}_{\Dtr}\in \R_{\ge0}^{(\ntr)\times m }} x\hspace{-0.05cm}\right) + \underbrace{\begin{bmatrix}  0\\-1\\ \vdots \\ -(\ntr - 1)\end{bmatrix}}_{b^{i,2}_{\Dtr}\in \R_{\le0}^{\ntr}}\hspace{-0.1cm}\right)
    \end{align*}
    }}%
    
    with $W^{i,2}_{\Dtr}\in \R_{\ge0}^{(\ntr)\times (\ntr)}$ a lower triangular matrix of ones, i.e.,
    {\scriptsize
    \begin{align*}
    W^{i,2}_{\Dtr}\coloneqq \begin{bmatrix}
    1       &  0 &  \ldots &0      \\
    \vdots  &\ddots  & \ddots  &\vdots     \\
    \vdots   &  &  \ddots &0     \\
    1       &   \ldots     & \ldots & 1
    \end{bmatrix}.
    \end{align*}
    }%
    From that, we can see that $\oneMi$ is indeed an MVNN with four layers in total (i.e., two hidden layers) and respective dimensions $[m,\ntr,\ntr,1]$. 
    From \Cref{eq:firstTermoneMi} we can see that for all $l=0,\ldots,\ntr:\, \oneMi(x^{(l)})=\hvi{x^{(l)}}$ and therefore $\oneMi\in \Vmon_{\Dtr}$. 
    
    \item\label{itm:proof:is100uUB}  Now we have to check what happens for $x\in \X:(x,\hvi{x})\not\in\Dtr$. Therefore, let $x\in \X:(x,\hvi{x})\not\in\Dtr$. Then, by definition we get that $\oneMi(x)=w_k$, where $k\in \{0,\ldots,\ntr\}$ is the smallest integer such that $x\subseteq x^{(k)}$. Assume there exists an $h \in \Vmon_{\Dtr}$ with $h(x)>\oneMi(x)=w_k=h(x^{(k)})$. However, since $x\subseteq x^{(k)}$ this is a contradiction to $h$ fulfilling the monotonicity property. Thus, we get that $\oneMi(x)=\max_{f\in\Vmon_{\Dtr}}f(x)$. Since $x$ was chosen arbitrarily, we finally get that
    $\oneMi(x)=\max_{f\in\Vmon_{\Dtr}}f(x)$ for all $x\in \X$ which concludes the proof.
\end{enumerate}
\end{proof}

\subsection{ICA-Based New NOMU Architecture}
In this section, we provide more details on our carefully chosen ICA-based new NOMU architecture. The original NOMU architecture from \cite{heiss2022nomu} outputs a mean prediction $\hat{f}$, which is in our case the mean-MVNN $\meanMi$ and a model uncertainty prediction $\hat{\sigma}_f$. Then uUBs in the original NOMU algorithm at an input point $x$ are defined as $\hat{f}(x)+c\cdot\hat{\sigma}_f(x)$ for $c>0$.

\subsubsection{Monotone uUBs and Non-Monotone \texorpdfstring{$\hat{\sigma}_f$}{sigma}}\label{subsec:monotoneUUB}

\paragraph{Monotone Value Functions Imply Monotone uUBs.} Knowing that the unknown ground truth function is monotonically increasing\footnote{Within this paper \enquote{monotonically increasing} should always be interpreted as \enquote{monotonically non-decreasing}.} means that the support of the prior in function space only contains monotonically increasing functions in Bayesian language. In frequentist language this means that the hypothesis class~$\HC$ only contains monotonically increasing functions. In both cases the resulting uUBs are monotonically increasing too, which we next prove in \Cref{prop:BayesUUBmonotone,prop:frequUUBmonotone}.
\begin{proposition}\label{prop:BayesUUBmonotone} 
Let $\PP[f \text{ is monotonically increasing}]=1$ according to the prior and let $\text{uUB}_\alpha(x)\coloneqq \inf\{y\in\R : \PP[f(x)\leq y | \Dtr]\geq \alpha\} \ \forall x\in X$ be the $\alpha$-credible upper bound (i.e., the $\alpha$-quantile of the posterior distribution of $f(x)$).\footnote{In the literature, $\PP[f(x)\leq y | \Dtr,x]$ is often used instead of \PP[f(x)\leq y | \Dtr] which is equal for every given $x\in X$. In both notations, $f$ is seen as a random variable in function space.}
Then it holds that $\text{uUB}_\alpha(x)$ is monotonically increasing (i.e., $x\leq\tilde{x}\implies \text{uUB}_\alpha(x)\leq \text{uUB}_\alpha(\tilde{x})$).
\end{proposition}
\begin{proof}
Let $x\leq\tilde{x}$. For a shorter notation we write \enquote{$f$ is (M)} instead of \enquote{$f$ is monotonically increasing} and we define $\tilde{\Vmon}\coloneqq \{f\in \R^X: f \text{ is (M)}\}$. Then, from the definition of monotonicity, it follows for every $y\in\R$ that\footnote{For the second equivalence we use that for every measurable set $A$, we obtain $\PP[A]-0=\PP[A]-\PP[\Vmon^c]\leq\PP[A\setminus\Vmon^c]\leq\PP[A]$, thus $P[A\setminus\Vmon^c]=\PP[A]$.} 
\begin{align*}
&\{f : f(x)>y, f \text{ is (M)}\}\subseteq\{f : f(\tilde{x})>y, f \text{ is (M)}\}
\\
\iff& \{f : f(x)>y\}\setminus\tilde{\Vmon}^c\subseteq\{f : f(\tilde{x})>y\}\setminus\tilde{\Vmon}^c
\\
\iff&\PP[\{f : f(x)>y\}|\Dtr]\leq\PP[\{f : f(\tilde{x})>y|\Dtr\}]\\
\iff&\PP[\{f : f(x)\le y\}|\Dtr]\ge\PP[\{f : f(\tilde{x})\le y|\Dtr\}]
\end{align*}
Since $\text{uUB}_\alpha(x)$ is the infimum, we know that for every $y<\text{uUB}_\alpha(x)$:
\[\PP[\{f : f(\tilde{x})\leq y|\Dtr\}]\leq \PP[\{f : f(x)\leq y\}|\Dtr]< \alpha.\]
This means that every $y<\text{uUB}_\alpha(x)$ is too small to be equal to $\text{uUB}_\alpha(\tilde{x})$, thus $\text{uUB}_\alpha(x)\leq\text{uUB}_\alpha(\tilde{x})$.
\end{proof}
\begin{proposition}\label{prop:frequUUBmonotone}
Let $\HC$ be a hypothesis class that only contains monotonically increasing functions and let $\text{uUB}_\HC(x)\coloneqq \sup\{f(x): f\in\HC_{\Dtr}\}$ be an uUB, then $\text{uUB}_\HC$ is monotonically increasing  (i.e., $x\leq\tilde{x}\implies \text{uUB}_\HC(x)\leq \text{uUB}_\HC(\tilde{x})$).
\end{proposition}
\begin{proof}
Let $x\leq\tilde{x}$, then $\sup\{f(x) : f\in\HC_{\Dtr}\}\leq\sup\{f(\tilde{x}) : f\in\HC_{\Dtr}\}$, since $\forall f\in\HC_{\Dtr}: f(x)\leq f(\tilde{x})$.
\end{proof}

\paragraph{Monotone Value Functions Do Not Imply Monotone Uncertainty.} However, the model uncertainty $\hat{\sigma}_f$ (defining the width of the UBs) is \emph{not} monotonic at all. For already observed training input points there is zero model uncertainty while smaller unobserved input points can have much bigger model uncertainty.

If one would simply use the prior knowledge about the monotonicity to improve the mean prediction (by assuring its monotonicity) but then estimate the uUB by simply adding a standard (e.g., original NOMU or Gaussian processes) estimator for the (scaled) model uncertainty $c\cdot\hat{\sigma}_f$ to the mean prediction to obtain an uUB, one would obtain non-monotonic uUBs violating \Cref{prop:BayesUUBmonotone,prop:frequUUBmonotone}.

If one would simply use ensemble methods \citep{lakshminarayanan2017simple,gal2016dropout} with MVNNs as ensemble members the uUBs obtained from the formula given in \citep{lakshminarayanan2017simple,gal2016dropout} (restated in \citep{heiss2022nomu} for the noiseless case) would also lead to non-monotonic uUBs violating \Cref{prop:BayesUUBmonotone,prop:frequUUBmonotone}. In theory, this problem could be circumvented by directly using a certain quantile of the ensembles as uUB instead of calculating uUBs based on empirical mean and variance of the ensemble. However, even the $100\%$-quantile of the ensemble (the point-wise maximum of the ensemble predictions) would often not sufficiently capture enough uncertainty, since \citet[Remark B.5]{heiss2022nomu} empirically showed that the uncertainty of ensemble methods often needs to be scaled up by a very big factor to capture enough uncertainty. Moreover, \citet{pmlr-v80-kuleshov18a} empirically showed that the uncertainty needs to be calibrated to achieve good results. To summarize, ensemble methods of monotonic functions have the problem that their calibration is limited or calibrating them results in non-monotonic uUBs. Note that the calibration method from \citet{pmlr-v80-kuleshov18a} cannot solve this problem.
Furthermore, in the ICA setting of this paper, optimizing the acquisition function (i.e., the ML-based WDP) based on the uUBs obtained from deep ensembles \citep{lakshminarayanan2017simple} would be computationally too expensive.

However, for our proposed uUB $\uUBMi$ the uncertainty can be calibrated by varying $\muexp$ without sacrificing monotonicity of the uUB. In the limit $\muexp\to0$ one would just obtain the mean prediction and in the limit $\muexp\to\infty$ in relation to explicit and implicit regularization one would obtain the 100\%-uUB as solution to our optimization problem.

\subsubsection{Linear Skip Connections}

We have a hyper-parameter (that is part of our HPO) that decides whether we add a trainable linear connection directly from the input to the output for $\uUBMi$ and $\meanMi$ as formally defined in \Cref{def:lskipMVNN}.

\subsubsection{No Connections Between the Two Sub-Architectures}

The architecture suggested in \citet{heiss2022nomu} contains connections from the mean-sub-architecture to the uncertainty-sub-architecture (the dashed lines in \cite[Figure~2]{heiss2022nomu}). In our architecture (see \Cref{fig:nn_tik_extension}) we do not use such connections for two reasons:
\begin{enumerate}
    \item Solving the WDP via the MILP given in \Cref{thm:milp} is computationally much faster because we can completely ignore $\meanMi$ for the MILP formulation.
    \item In the case of \citet{heiss2022nomu}, without these connections, the uncertainty would be completely independent from the mean prediction (or from the labels $y^\text{train}$), prohibiting \citep[Desiderata D4]{heiss2022nomu}. However, our architecture directly outputs the uUB instead of the uncertainty $\hat{\sigma}_f$, thus $\uUBMi$ is automatically not independent from the labels $y^{(l)}=\hvi{x^{(l)}}$.
\end{enumerate}

\subsection{ICA-Based New NOMU Loss}\label{sec.appednix:DetailedNOMULoss}
\begin{definition}[\href{https://pytorch.org/docs/stable/generated/torch.nn.SmoothL1Loss.html}{Smooth L1 Loss}]\label{def:appendix_smoothL1Loss}
The smooth L1 loss $\sL:\mathbb{R}\times\mathbb{R}\to \mathbb{R}$ with threshold parameter $\beta \ge0$ is defined as follows:
\begin{align}
    \sL(x,y)=\begin{cases}
    \frac{0.5}{\beta}\cdot(x-y)^2, &|x-y|\le \beta\\
    |x-y|-0.5\cdot\beta, &\text{otherwise}.
    \end{cases}
\end{align}
\end{definition}

\begin{definition}[\href{https://pytorch.org/docs/stable/generated/torch.nn.ELU.html}{Exponential Linear Unit (ELU)}]\label{def:appendix_ELU}
The ELU function $\elu:\mathbb{R}\to \mathbb{R}$ (with default parameter $\alpha=1$) is defined as follows:
\begin{align}
    \elu(x)=\begin{cases}
    x, &x\ge 0\\
    1\cdot(\exp(x)-1), &x<0.
    \end{cases}
\end{align}
\end{definition}

\begin{definition}[Detailed ICA-based New NOMU Loss]\label{def:app:NOMU_Loss} Let $\hp=(\musqr,\muexp,\cexp, \piOneuUB,\piMean)\in \R_+^5$ be a tuple of hyperparameters. For a training set $\Dtr$, $L^\hp$ is defined as

\begin{subequations}\label{eq:app:NOMU_Loss_Extension}%
\resizebox{\columnwidth}{!}{\parbox{1.2\columnwidth}{%
\begin{align}
  \hspace{-1cm}&L^\hp(\uUBMi)\coloneqq 
  \musqr\sum_{l=1}^{\ntr}\sL\left(\uUBMi(x^{(l)}),y^{(l)}\right)\label{subeq:app:dataLoss}\\
  \hspace{-.375cm}&+\muexp\int_{[0,1]^m}\hspace{-0.5cm}g\left(0.01-\cexp\left(\min\{\uUBMi(x),\oneMi(x)\} - \meanMi(x)\right)\right)\,dx\label{subeq:app:pushUpLoss}\\
  \hspace{-.375cm}&+\muexp\cexp\piOneuUB\int_{[0,1]^m}\hspace{-0.5cm}\sL\left((\uUBMi(x)-\oneMi(x))^{+}\right)\,dx\label{subeq:app:below100Loss}\\
  \hspace{-.375cm}&+\muexp\cexp\piMean\int_{[0,1]^m}\hspace{-0.5cm}\sL\left((\meanMi(x)-\uUBMi(x))^{+}\right)\,dx\label{subeq:app:aobveMeanLoss}\\
  \hspace{-.375cm}&+\musqr\sum_{l=1}^{\ntr}0.001\left(\uUBMi(x^{(l)})-y^{(l)}\right)^{+}\hspace{-0.2cm}+0.5\sL\left((\uUBMi(x^{(l)})-y^{(l)})^{+},0\right) \Big)\label{subeq:app:dataLossAssymetric}
\end{align}
}
}
\end{subequations}

\noindent where $\sL$ is the smooth L1-loss with threshold $\beta$ (see \Cref{def:appendix_smoothL1Loss}), $(\cdot)^{+}$ the positive part, and $g\coloneqq 1+\elu$ is convex monotonically increasing with $\elu$ being the \emph{exponential linear unit (ELU)} (see \Cref{def:appendix_ELU}).
\end{definition}
Detailed interpretations of all five terms (including \eqref{subeq:app:dataLossAssymetric} which was added to slightly improve the numerical stability) are as follows:
\begin{enumerate}[align= left]
	\item[\eqref{subeq:app:dataLoss}] enforces that $\uUBMi$ fits through the training data.
	\item[\eqref{subeq:app:pushUpLoss}] pushes $\uUBMi$ up as long as it is below the 100\%-uUB $\oneMi$.
	This force gets weaker the further $\uUBMi$ is above the mean $\meanMi$ (especially if $\cexp$ is large).
	$\muexp$ controls the overall strength of \eqref{subeq:pushUpLoss} and $\cexp$ controls how fast this force increases when $\uUBMi \to\meanMi$.
	Thus, increasing $\muexp$ increases the uUB and increasing $\cexp$ increases the uUB in regions where it is close to the mean.
	Weakening \eqref{subeq:pushUpLoss} (i.e.,  $\muexp\cexp\to0$) leads $\uUBMi \approx\meanMi$. Strengthening \eqref{subeq:pushUpLoss} by increasing $\muexp\cexp$ in relation to regularization\footnote{Regularization can be early stopping or a small number of neurons (implicit) or L2-regularization on the parameters (explicit). The same principle would also hold true if one uses other forms of regularization such as L1-regularization or dropout.
	} leads $\uUBMi\approx\oneMi$.
 	In practice, we obtain reasonable approximations to $\alpha\%$-uUB with $\alpha\in[50,100]$ depending on the value of $\muexp\cexp$ in relation to regularization.
	\item[\eqref{subeq:app:below100Loss}] enforces that $\uUBMi\le\oneMi$. In theory, \eqref{subeq:below100Loss} would be redundant in the limit $\musqr\to\infty$, because $\uUBMi\in\Vmon_{\Dtr}$. The strength of this term is determined by $\piOneuUB\cdot (\muexp\cexp)$, where $\piOneuUB$ is the \eqref{subeq:below100Loss}-specific hyperparameter and $\muexp\cexp$ adjusts the strength of \eqref{subeq:below100Loss} to \eqref{subeq:pushUpLoss}.
	\item[\eqref{subeq:app:aobveMeanLoss}] enforces $\uUBMi\ge\meanMi$. In theory, one should take the limit $\piMean\to\infty$. However, in practice, a moderate value of $\piMean$ is numerically more stable and typically enforces that $\uUBMi\ge\meanMi$. The interpretation of $\piMean$ and $\muexp\cexp$ is analogous to \eqref{subeq:below100Loss}.
	\item[\eqref{subeq:app:dataLossAssymetric}] is an asymmetric version of \eqref{subeq:app:dataLoss} for numerical stability. Since \eqref{subeq:app:pushUpLoss} pushes $\uUBMi$ for all $x\in \X$ upwards, $\uUBMi$ would have a tendency to give slightly too high predictions for training data points, but \eqref{subeq:app:dataLossAssymetric} compensates this effect for slightly improved numerical stability. In theory \eqref{subeq:app:dataLossAssymetric} would be redundant in the limit $\musqr\to\infty$.
\end{enumerate}

As in \cite{heiss2022nomu}, in the implementation of $L^\hp$, we approximate \Crefrange{subeq:app:pushUpLoss}{subeq:app:aobveMeanLoss} (in the main paper \Crefrange{subeq:pushUpLoss}{subeq:aobveMeanLoss}) via MC-integration using additional, \emph{artificial input points} ${\Dart\coloneqq \left\{x^{(l)}\right\}_{l=1}^{\nart}\stackrel{i.i.d.}{\sim}\textrm{Unif}([0,1]^m)}$, where we sample new artificial input points for each batch of our mini-batch gradient descent based training algorithm.

Furthermore, note that in practice we train $\meanMi$ and $\uUBMi$ simultaneously but where $\meanMi$ is \emph{detached} (using \texttt{torch.tensor.detach()}) in the loss $L^\hp$ such that $L^\hp$ does not influence the mean MVNN $\meanMi$.

\section{Parameter Initialization for MVNNs}\label{sec:apendix_New MVNN Random Initialization}
In this section, we provide more details on our new parameter initialization method for MVNNs (see \Cref{subsec:New MVNN Random Initialization}). Specifically, we give recommendations on how to set the hyperparameters of our proposed mixture distribution depending on the architecture size of the considered MVNN.
\subsection{Theoretical Results}
\begin{definition}[Uniform Mixture Distribution]\label{def:mixtureDistribution}
        We define the probability density $g_{A_k,B_k,p_k}$ of an initial weight $W_{j,l}^{i,k}$ corresponding to an MVNN as\footnote{The bidder index $i$ is not relevant in this section and we assume that each bidder $i$ has the same architecture. If bidders would have different layer widths $d^{i,k-1}$, then all the values $A_{k},B_k, P_k,\mu_k,\sigma_k$ would depend on the bidder index $i$ too, i.e., $A_{i,k},B_{i,k}, P_{i,k},\mu_{i,k},\sigma_{i,k}$.}
        \begin{align}
        g_{A_k,B_k,p_k}(w)=\frac{1-p_k}{A_k}\indicatorOne_{[0,A_k]}(w)+\frac{p_k}{B_k}\indicatorOne_{[0,B_k]}(w),
        \end{align} which corresponds to a mixture distribution of $\Unif[0,A_k]$ and $\Unif[0,B_k]$.
\end{definition}
We assume that for each layer $k$, we independently sample all weights $W^{i,k}_{jl}$ i.i.d.\ according to the distribution in \Cref{def:mixtureDistribution}.

We then obtain that
\begin{align}\label{eq:mean_mixture}
\mu_k=\E[W_{j,l}^{i,k}]=(1-p_k) \frac{A_k}{2} + p_k \frac{B_k}{2}.
\end{align}
Moreover, by using that the variance of an $\Unif[0,c]$-distributed random variable is \href{https://en.wikipedia.org/w/index.php?title=Continuous_uniform_distribution&oldid=1098500085}{$\frac{c^2}{12}$} we get that
\begin{equation}\label{eq:var_mixture}
\sigma_k^2=\Var[W_{j,l}^{i,k}] \href{https://en.wikipedia.org/w/index.php?title=Mixture_distribution&oldid=1091582237\#Moments}{=} (1 - p_k)\frac{A_k^2}{3} + p_k\frac{B_k^2}{3} - \mu_k^2.
\end{equation}

In the following, we provide a rule how to set for each layer $k$, $A_k$, $B_k$ and $p_k$ depending on the dimension $d^{i,k-1}$ of the previous $(k-1)$\textsuperscript{st} layer.

This rule ensures that the conditional expectation and the conditional variance are equal to two constants $\Einit$, and $\Vinit$, which are independent of the dimension of the previous layer $d^{i,k-1}$. More formally, let $z^{i,k-1}\in \R^{d^{i,k-1}}$ be the output of the neurons in the $(k-1)$\textsuperscript{st} layer, then we set $A_k$, $B_k$ and $p_k$ such that\footnote{Note that $\Einit=\text{\eqref{subeq:constE}}$ ($\texttt{init\_E}$ in our code) includes already the bias, while $\Vinit=\text{\eqref{subeq:constV}}$ ($\texttt{init\_V}$ in our code) does not include the bias. But this is a simple additive term that could easily be adjusted. One could easily include the bias in both or in none of them to make it more consistent.}
\begin{subequations}\label{eq:constEAndV}
\begin{align}
    \Eco{\left(W^{i,k}z^{i,k-1}\right)_j +b^{i,k}_j }{ z^{i,k-1}=(1,\dots,1)^\top}=\Einit
    \label{subeq:constE}\\
    \Varco{\left(W^{i,k}z^{i,k-1}\right)_j }{ z^{i,k-1}=(1,\dots,1)^\top}=\Vinit
    .\label{subeq:constV}
\end{align}
\end{subequations}
\begin{lemma}\label{rem:differentFormulationsOfConstEV} Equivalent formulations of
\eqref{subeq:constE} are \begin{align*}
    \text{\eqref{subeq:constE}}&\Leftrightarrow\E[{\left(W^{i,k}(1,\dots,1)^\top\right)_j +b^{i,k}_j }] = \Einit
    \\
    &\Leftrightarrow\Eco{\left(W^{i,k}z^{i,k-1}\right)_j +b^{i,k}_j }{ \overline{z^{i,k-1}}=1}=\Einit,
\end{align*}
where $\overline{z^{i,k-1}}=\frac{1}{d^{i,k-1}}\sum_{l=1}^{d^{i,k-1}}z^{i,k-1}_l$. And equivalent formulations of \eqref{subeq:constV} are 
\begin{align*}
    \text{\eqref{subeq:constV}}&\Leftrightarrow\Var[{\left(W^{i,k}(1,\dots,1)^\top\right)_j }]=\Vinit
    \\
    &\Leftrightarrow\Varco{\left(W^{i,k}z^{i,k-1}\right)_j }{ z^{i,k-1}}=\Vinit\overline{\left(z^{i,k-1}\right)^2},
\end{align*}
where $\overline{\left(z^{i,k-1}\right)^2}=\frac{1}{d^{i,k-1}}\sum_{l=1}^{d^{i,k-1}}\left(z^{i,k-1}_l\right)^2$.
\end{lemma}
\begin{proof}
The first and the third equivalence are trivial. The second equivalence can be seen as:
 {\small
\begin{align*}
&\Eco{\left(W^{i,k}z^{i,k-1}\right)_j +b^{i,k}_j }{ \overline{z^{i,k-1}}=1}=\\
&=\sum_{l=1}^{d^{i,k-1}}\Eco{W^{i,k}_{jl} z^{i,k-1}_l}{\overline{z^{i,k-1}}=1}+\E[b^{i,k}_j]
\end{align*}
Using that $W^{i,k}_{jl}$ and $z^{i,k}$ are independent, we can further simplify this term as
\begin{align*}
&=\sum_{l=1}^{d^{i,k-1}}\Eco{z^{i,k-1}_l}{\overline{z^{i,k-1}}=1}\E[W^{i,k}_{jl} ]+\E[b^{i,k}_j]\\
&=\E[W^{i,k}_{j1} ]\sum_{l=1}^{d^{i,k-1}}\Eco{z^{i,k-1}_l}{\overline{z^{i,k-1}}=1}+\E[b^{i,k}_j]\\
&=\E[W^{i,k}_{j1} ]\Eco{\sum_{l=1}^{d^{i,k-1}}z^{i,k-1}_l}{\overline{z^{i,k-1}}=1}+\E[b^{i,k}_j]\\
&=\E[W^{i,k}_{j1} ]d^{i,k-1}+\E[b^{i,k}_j]\\
&=\E[{\left(W^{i,k}(1,\dots,1)^\top\right)_j +b^{i,k}_j }].
\end{align*}} 
Finally, the last equivalence can be seen as
\begin{align*}
&\Varco{\left(W^{i,k}z^{i,k-1}\right)_j }{z^{i,k-1}}=\\
&=\sum_{l=1}^{d^{i,k-1}}\Varco{W^{i,k}_{jl} z^{i,k-1}_l}{z^{i,k-1}}\\
&=\sum_{l=1}^{d^{i,k-1}}\Var\left[W^{i,k}_{jl}\right]\left(z^{i,k-1}_l\right)^2\\
&=d^{i,k-1}\Var\left[W^{i,k}_{j1}\right]\frac{1}{d^{i,k-1}}\sum_{l=1}^{d^{i,k-1}}\left(z^{i,k-1}_l\right)^2\\
&=\Var[{\left(W^{i,k}(1,\dots,1)^\top\right)_j }]\overline{\left(z^{i,k-1}\right)^2}.
\end{align*} 
\end{proof}
The solution to problem~\eqref{eq:constEAndV} is not unique, but we propose three hyperparameters $\Binit, \Biasinit, \littleConst\in\Rp$ to characterize one specific solution that has multiple nice properties (see \Cref{thm:ScalingInitConstantEandV}).
\begin{definition}[Scaling Rule]\label{def:initScalingRule}
        For any choice of hyperparameters $\Einit, \Vinit, \Binit, \Biasinit, \littleConst\in\Rp$ 
        we propose the following values for $B_k$, $A_k$ and $p_k$  (for a shorter notation we write $d$ instead of $d^{i,k-1}\in\N_+$):
       \begin{align}
        &B_k=\begin{cases}
        \max \left(\frac{3 \Muk^2+3 d \Vinit}{2 \Muk d}+\frac{\littleConst}{d},\Binit\right) &\text{, }d>\frac{\Muk^2}{3 \Vinit}\\
        \frac{2}{d}\Muk&,\,\text{else}
        \end{cases}\label{eq:def:Bk}\\
        &p_k=\begin{cases}
        1-\frac{B_k^2 d^2-4 B_k \Muk d+4 \Muk^2}{B_k^2 d^2-4 B_k \Muk d+3 \Muk^2+3 d \Vinit}
        &\hspace{-0.15cm}\text{, }d>\frac{\Muk^2}{3 \Vinit}\\
        1&\hspace{-0.15cm},\,\text{else}
        \end{cases}\label{eq:def:pk}\\
        &A_k=\begin{cases}
        \frac{
        2 \Muk-B_k d p_k
        }{d (1-p_k)
        }
        &\text{, }d>\frac{\Muk^2}{3 \Vinit}\\
        0&,\,\text{else,}
        \end{cases}\label{eq:def:Ak}\\
        &W_{j,l}^{i,k}\stackrel{i.i.d.}{\sim} g_{A_k,B_k,p_k}\\
        &b^{i,k}_j\stackrel{i.i.d.}{\sim}\Unif[-\Biasinit,0]\label{eq:biasDistribution}
        \end{align}
        where $\Muk=\Einit + \frac{\Biasinit}{2}$ and where $g_{A_k,B_k,p_k}$ is the mixture distribution described in \Cref{def:mixtureDistribution} and where all weights and biases are sampled independently.
\end{definition}
Finally, when using the scaling rule from \Cref{def:initScalingRule}, we can prove the following theorem.
\begin{theorem}\label{thm:ScalingInitConstantEandV}For any choice of hyperparameters $\Einit, \Vinit, \Binit, \Biasinit, \littleConst\in\Rp$, for every $k\in\{1,\dots,K_i\}$, for every dimension $d^{i,k-1}\in\N_+$ it holds that if we sample according to the distribution described in \Cref{def:initScalingRule}, we obtain that:\footnote{We assume that $\Einit,\Vinit$ and $d$ are strictly positive. However, $\Binit,\Biasinit$ and $\littleConst$ could in theory also be set to 0 and \Cref{thm:ScalingInitConstantEandV} is formulated such that it would still hold true for $\Binit,\Biasinit,\littleConst\in\Rp\cup\{0\}$.}
\begin{enumerate}
    \item\label{itm:Ecorrect} \eqref{subeq:constE} holds,
    \item\label{itm:VnotTooSmall} $\Varco{\left(W^{i,k}z^{i,k-1}\right)_j }{ z^{i,k-1}=(1,\dots,1)^\top}\ge\Vinit
    $ holds,
    \item\label{itm:Vcorrect} \eqref{subeq:constV} holds if $d^{i,k-1}>\frac{\Muk^2}{3 \Vinit}$,
    \item\label{itm:NonZeorWeights} $\PP[W^{i,k}_{j,l}=0]=0$ if $\littleConst>0$,
    \item\label{itm:ZeorWeights} $\PP[W^{i,k}_{j,l}=0]=1-p_k$ if $\littleConst=0$ and $\Binit< B_k$,
    \item\label{itm:validMVNNInitialization} $\PP[W^{i,k}_{j,l}<0]=0$,
    \item\label{itm:boundedWeights} $\PP[W^{i,k}_{j,l}>B_k]=0$,
    \item\label{itm:BlowerBound} $B_k\geq \frac{3\Vinit}{2\Muk}$,
    \item\label{itm:Blimit} $\lim_{d^{i,k-1}\to\infty}B_k= \max\left\{\frac{3\Vinit}{2\Muk},\Binit\right\}$.
\end{enumerate}

\end{theorem}
\begin{proof}
First, note that by \Cref{rem:differentFormulationsOfConstEV}
\begin{align*}
\text{\eqref{subeq:constE}}&\Leftrightarrow
\E[{\left(W^{i,k}(1,\dots,1)^\top\right)_j +b^{i,k}_j }]=\Einit
\\
&\Leftrightarrow
\E[{\left(W^{i,k}(1,\dots,1)^\top\right)_j }]=\Einit-\E[b^{i,k}_j]
\\
&\Leftrightarrow
d^{i,k-1}\E[W^{i,k}_{j,l}]=\Einit+\frac{\Biasinit}{2}
\\
&\Leftrightarrow
\E[W^{i,k}_{j,l}]=\frac{\Muk}{d^{i,k-1}}.
\end{align*}
We prove this in section \enquote{Proof of item 1} of the \MathematicaScript{}\footnote{You can open the \MathematicaScript{} including the results in your web-browser (\url{https://www.wolframcloud.com/obj/jakob.heiss/Published/MVNN_initialization_V1_4.nb}) without the need to install anything on your computer. (The interactive plot might only work if you open the script with an installed version of Mathematica, but it is not necessary for the proof.) We use a computer algebra system to make it much more convenient to check the correctness of our proof. Some of the terms that appear in this proof are already very long as one can see for example in section \enquote{Proof of item 2 and 3} in our \MathematicaScript{}. Simplifying these terms by hand would take multiple pages of very tedious calculations, which would be highly prone to typos and other mistakes. Within our \MathematicaScript{} we only use exact symbolic methods and no numerical approximations (except for the visualizations).}, where we use the notation $\texttt{d}\coloneqq d^{i,k-1}$, $\texttt{M}\coloneqq \Muk$, $\texttt{V}\coloneqq \Vinit$, $\texttt{epsilon}\coloneqq \littleConst$, $\texttt{Binit}\coloneqq \Binit$,
$\texttt{Bchoice}\coloneqq B_k$ given in \eqref{eq:def:Bk},
$\texttt{pchoice}\coloneqq p_k$ given in \eqref{eq:def:pk},
$\texttt{Achoice}\coloneqq A_k$ given in \eqref{eq:def:Ak} and 
$\texttt{EofW}\coloneqq \E[W^{i,k}_{j,l}]$ computed in \eqref{eq:mean_mixture}.
This proves \cref{itm:Ecorrect}.

We manipulate using \Cref{rem:differentFormulationsOfConstEV}
\begin{align*}
\text{\eqref{subeq:constV}}&\Leftrightarrow
\Var[{\left(W^{i,k}(1,\dots,1)^\top\right)_j }]=\Vinit
\\
&\Leftrightarrow
d^{i,k-1}\Var[W^{i,k}_{j,l}]=\Vinit
\\
&\Leftrightarrow
\Var[W^{i,k}_{j,l}]=\frac{\Vinit}{d^{i,k-1}}.
\end{align*}
We prove this in section \enquote{Proof of item 2 and 3} of the \MathematicaScript{} with the additional notation $\texttt{VofW}\coloneqq \Var[W^{i,k}_{j,l}]$ computed in \eqref{eq:var_mixture}.
This section proves \cref{itm:Vcorrect,itm:VnotTooSmall}.

To show \cref{itm:NonZeorWeights} it is sufficient to show $A_k>0$ (because of $A_k\leq B_k$, as shown in section \enquote{Proof of item 7: (\texttt{A}$\leq$\texttt{B})}) or ($p_k=1$ and $B_k>0$).
If $d^{i,k-1}>\frac{\Muk^2}{3 \Vinit}$, we show $A_k>0$ in section \enquote{Proof of item 4} of the \MathematicaScript{}.
In the other case $p_k=1$ and $B_k>0$ as one can directly see from \eqref{eq:def:Bk} and \eqref{eq:def:pk}.

To show \cref{itm:ZeorWeights}, it is sufficient to show that $A_k=0$. Section \enquote{Proof of item 5} of the \MathematicaScript{} shows in the case $d>\frac{\Muk^2}{3 \Vinit}$ that $A_k=0$ $\iff$ ($\littleConst=0$ and $\Binit\leq\frac{3 \Muk^2+3 d \Vinit}{2 \Muk d}$). By definition this statement is equivalent to $A_k=0 \iff B_k=\frac{3 \Muk^2+3 d \Vinit}{2 \Muk d}$. Finally, since ($\littleConst=0$ and $\Binit<B_k) \implies B_k=\frac{3 \Muk^2+3 d \Vinit}{2 \Muk d}$, \cref{itm:ZeorWeights} holds true. 
In the case $d^{i,k-1}\le\frac{\Muk^2}{3 \Vinit}$, $A_k=0$ (i.e., \cref{itm:ZeorWeights}) follows directly from \eqref{eq:def:Ak}.
Note that $B_k>0$ always holds true, since we always assume $\Muk>0$ and $d>0$.

\Cref{itm:validMVNNInitialization} follows directly from \Cref{def:mixtureDistribution} and the fact that $A_k\geq0$ and $B_k\geq0$ (see section \enquote{Proof of item 6} in the \MathematicaScript{}).

\Cref{itm:boundedWeights} follows directly from \Cref{def:mixtureDistribution} and $A_k\leq B_k$ (see section \enquote{Proof of item 7: (\texttt{A}$\leq$\texttt{B})} in the \MathematicaScript{}).

\Cref{itm:BlowerBound} is shown for the two cases in section \enquote{Proof of item 8} of the \MathematicaScript{}.

\Cref{itm:Blimit} is shown in section \enquote{Proof of item 9} in our \MathematicaScript{}.
\end{proof}

\paragraph{Discussion of \Cref{thm:ScalingInitConstantEandV}.}\Cref{itm:Ecorrect,itm:Vcorrect} in \Cref{thm:ScalingInitConstantEandV} tell us that our initialization scheme actually solves problem~\eqref{eq:constEAndV} if $d^{i,k-1}>\frac{\Muk^2}{3\Vinit}$. Note that problem~\eqref{eq:constEAndV} does not have any solution for for $d^{i,k-1}\leq\frac{\Muk^2}{3\Vinit}$ (as can be seen in the first 5 lines of our \MathematicaScript{} with the notation explained in the proof of \Cref{thm:ScalingInitConstantEandV}). However, \cref{itm:Ecorrect,itm:VnotTooSmall} in \Cref{thm:ScalingInitConstantEandV} tell us that we still have a reasonable initialization scheme for $d^{i,k-1}\leq\frac{\Muk^2}{3\Vinit}$ that solves the relaxed problem of \cref{itm:Ecorrect,itm:VnotTooSmall}.
In practice solutions to this relaxed problem are still fine, since it also prevents us from exploding expectation or vanishing variance. Too much variance in the initialization is less of a problem.\footnote{Alternatively one could consider to relax problem~\eqref{eq:constEAndV} in the other direction by allowing smaller expectation instead of bigger variance, which would also be fine in practice. A solution to this alternative relaxed problem could be simply achieved by changing the definition of $B_k$ in the \enquote{else}-case of \eqref{eq:def:Bk}.}

\Cref{itm:NonZeorWeights,itm:ZeorWeights} motivate to choose $\littleConst>0$ to prevent weights being initialized to zero, which can lead to bad local minima.

\Cref{itm:validMVNNInitialization} guarantees that our network is actually a valid MVNN at initialization fulfilling the non-negativity constraints of the weights at initialization.\footnote{Note that if we would initialize our network with a weight-distribution that solves problem~\eqref{eq:constEAndV}, but does not fulfill \Cref{itm:validMVNNInitialization} (e.g., a generic initialization with $\mu_k=0$ and $\sigma_k\propto\frac{1}{\sqrt{d^{i,k-1}}}$), we would get a valid MVNN after the first gradient step, which projects all the negative weights to zero. However, this almost initial network would have exploding conditional expectations, i.e., almost all neurons would have pre-activations $o^{i,k}_j>1$ independent of the input training data point and thus one would end up with an almost constant function and zero gradients as shown in \Cref{fig:appendix_1d_path_plot_initialization}.}

\Cref{itm:boundedWeights,itm:BlowerBound,itm:Blimit} give us guarantees on the upper bound of the weights.

\subsection{Recommended Hyperparameter Choices}\label{sec:appendix:RecomendedHyperParametersInit}
In this section, we provide intuition about each hyperparameter $\Einit, \Vinit, \Binit, \Biasinit, \littleConst\in\Rp$ and recommendations on how to set them in practice.

\begin{figure*}
    \centering
    \resizebox{\textwidth}{!}{
    \includegraphics[trim=0 0 0 0, clip]{figures/main/Seed11_1dpathPlot_12_08_2022_10-17-30_glorotDefault_NEW.pdf}
    \includegraphics[trim=95 0 0 0, clip]{figures/main/Seed11_1dpathPlot_12_08_2022_10-27-23_custom_NEW.pdf}
    \includegraphics[trim=95 0 0 0, clip]{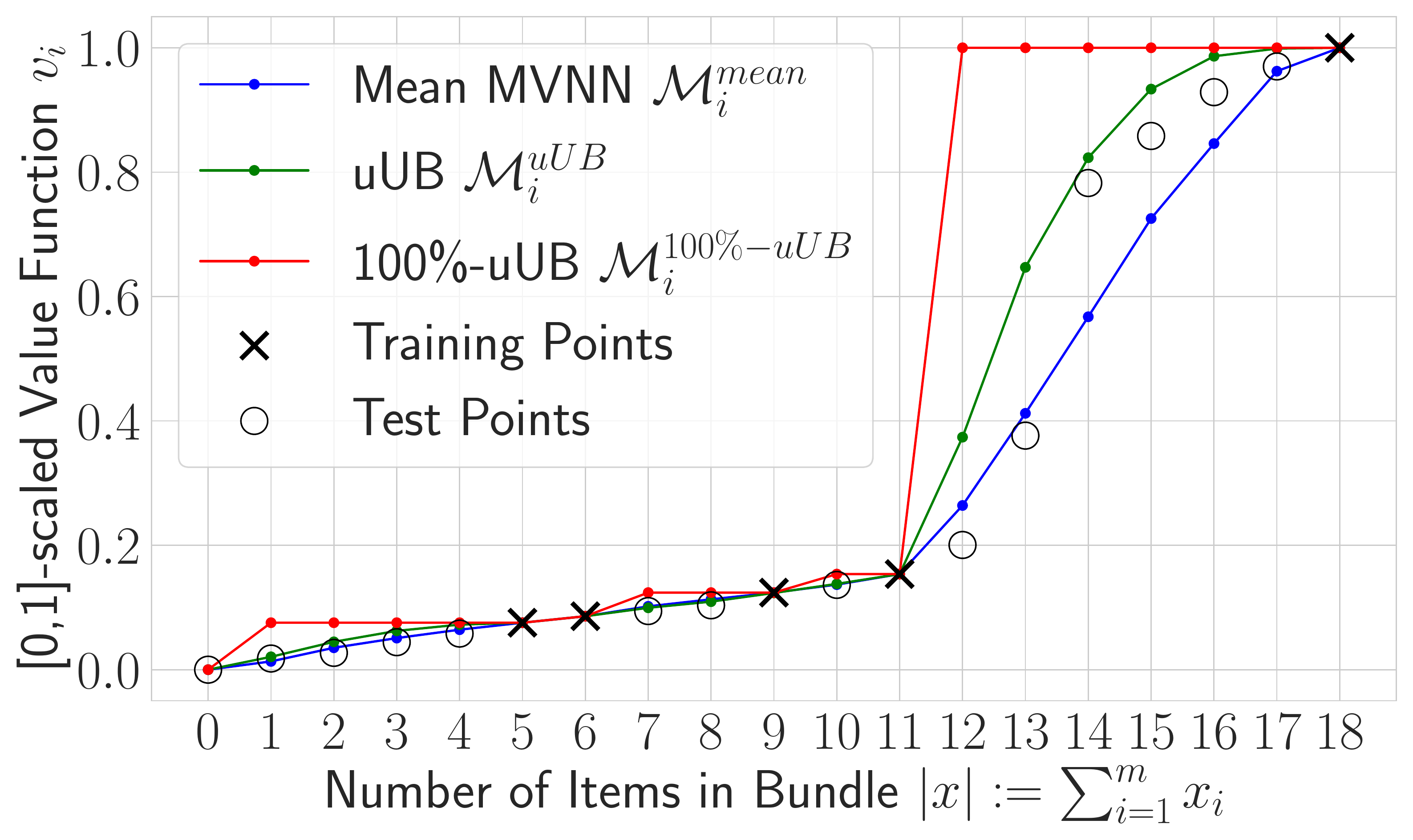}
    }
    \caption{$\meanMi$, $\uUBMi$ and $\oneMi$ along an increasing 1D subset-path in LSVM for the national bidder. \textbf{Left:} [64,64]-architecture with generic initialization fails. \textbf{Middle:} [64,64]-architecture with our initialization works. \textbf{Right:} even larger [256,256]-architecture with our initialization still works.}
    \label{fig:appendix_1d_path_plot_initialization}
\end{figure*}
\begin{enumerate}
    \item \textbf{Parameter $\Einit$:} (\texttt{init\_E} in our code) gives the conditional expectation~\eqref{subeq:constE} of a pre-activated neuron (including bias) conditioned on $\overline{z^{i,k-1}}=1$ (see \cref{itm:Ecorrect} in \Cref{thm:ScalingInitConstantEandV} and \Cref{rem:differentFormulationsOfConstEV}). This corresponds to an upper bound of the expected value of the MVNN $\E[\MVNNi{}((1,...,1))]$ (i.e., the predicted value of the full bundle) at initialization of the network, if all cutoffs $t^{i,k}$ of the bReLU activation function are equal to $1$. $\Einit$ is approximately equal to $\E[\MVNNi{}((1,...,1))]$ if $\Einit\geq1$.
    
    If you normalize the data such that the full bundle has always value 1 (i.e., $\hvi{}((1,...,1))=1$), setting $\Einit=1$ is our recommended choice. If you choose to initialize the bReLU cutoffs $t^{i,k}$ i.i.d.\ uniformly at random, i.e.,  $t^{i,k}\sim \text{Unif}(0,1)$, then $\Einit \in [1,2]$ is recommended, because in this case $\frac{\Einit}{2}$ is an upper bound for the expectation of the pre-activated values\footnote{Recall, that $o^{i,k}_j$ is the pre-activated value of the $j$-th neuron in the $k$-th layer (of the $i$-th bidder) including biases.} $o^{i,k}_j$ of any neuron in the MVNN at initialization for $k>1$. This can be seen as follows:
    \begin{align}
    \E[o^{i,k}_j]&\stackrel{\text{\Cref{rem:differentFormulationsOfConstEV}}}{=}(\Einit-\E[b^{i,k}_j])\E[z_1^{i,k-1}]+\E[b^{i,k}_j]\\
    &\leq(\Einit-\E[b^{i,k}_j])\E[t^{i,k-1}]+\E[b^{i,k}_j]\\
    &=\frac{\Einit+\frac{\Biasinit}{2}}{2}-\frac{\Biasinit}{2}\\
    &=\frac{\Einit}{2}-\frac{\Biasinit}{4}<\frac{\Einit}{2},
    \end{align}
    where the second inequality follows from the definition of the bRelu activation function.
     Specifically, this also shows for $j=1$ and $k=K_i$ that $\frac{\Einit}{2}$ is an upper bound for $\E[\MVNNi{}((1,...,1))]=o^{i,K_i}_1((1,...,1))$ at initialization.\footnote{On the other hand one could argue that one wants the conditional expectation of pre-activated neurons to be smaller because of the smaller cut-offs. However, note especially for small values of $\Einit$, the actual expectation $\E[\MVNNi{}((1,...,1))]$ at initialization decreases with increasing depth of the network, since the upper bound can loose its tightness as depth increases especially for $\Einit<1$.} If the values of your MVNN are in a different order of magnitude, you should scale your data in a pre-processing step to $[0,1]$.
    \item \textbf{Parameter $\Vinit$:} (\texttt{init\_V} in our code) gives the conditional variance~\eqref{subeq:constV} of a pre-activated neuron (without bias) conditional on $z^{i,k-1}=(1,\dots,1)^\top$, if $d^{i,k-1}>\frac{\Muk^2}{3 \Vinit}$ (see \cref{itm:Vcorrect} in \Cref{thm:ScalingInitConstantEandV}). In any case, $\Vinit$ is a lower bound for this conditional variance (see \cref{itm:VnotTooSmall} in \Cref{thm:ScalingInitConstantEandV}). 
    Typically, we select $\Vinit \in [\nicefrac{1}{50},1]$.
    Choosing $\Vinit$ too small yields an almost deterministic network initialization.
    Since we prefer initial weights that are smaller than one, $\Vinit$ should not be chosen too large (i.e., preferably $\frac{3\Vinit}{2\Muk}\leq1$ because of \cref{itm:boundedWeights,itm:BlowerBound,itm:Blimit} in \Cref{thm:ScalingInitConstantEandV}).
    \item \textbf{Parameter $\Biasinit$:} (\texttt{init\_bias} in our code\footnote{\label{footnote:PaperToCodeBiasAndB}Be careful, when translating the notation of our paper into the notation of our code. $B_k$ is \texttt{b} in our code and the biases $b$ are denoted by \texttt{bias} in our code. $\Biasinit$ does \emph{not} translate to \texttt{init\_b}.}) All the biases of the MVNN $b^{i,k}_j$ are sampled uniformly at random from $[-\Biasinit,0]$ as given in \cref{eq:biasDistribution}. Setting $\Biasinit=0.05$ is our recommendation, although any other small values would be sufficient. We discourage zero to avoid numerical issues during training.
    \item \textbf{Parameter $\Binit$:} (\texttt{init\_b} in our code\textsuperscript{\ref{footnote:PaperToCodeBiasAndB}}) 
    $\Binit$ gives a \emph{lower} bound for $B_k$ in the case of large number of neurons (see \cref{itm:Blimit} in \Cref{thm:ScalingInitConstantEandV} and \cref{eq:def:Bk}).
    The "big" weights are sampled from $\text{Unif}(0,B_k)$, i.e., $B_k$ is an upper bound for the weights (see \cref{itm:boundedWeights} in \Cref{thm:ScalingInitConstantEandV}).
    Usually, we prefer $B_k$ that are not unnecessary big, but too small $B_k$ leads to almost vanishing $A_k$ (in the case of $d>\frac{\Muk^2}{3 \Vinit}$).\footnote{Section \enquote{Proof of item 5} in the \MathematicaScript{} shows that $A_k=0$ iff $\epsilon=0$ and $\Binit\le\frac{3 \Muk^2+3 d \Vinit}{2 \Muk d}$. This statement is equivalent to $A_k=0$ iff $B_k=\frac{3 \Muk^2+3 d \Vinit}{2 \Muk d}$, which is the minimal possible choice for $B_k\in[\frac{3 \Muk^2+3 d \Vinit}{2 \Muk d},\infty)$. By continuity $B_k \approx \frac{3 \Muk^2+3 d \Vinit}{2 \Muk d} \implies A_k\approx 0$. Note that in the other case $d\le\frac{\Muk^2}{3 \Vinit}$, $p_k=1$ and thus $A_k$ is irrelevant for the mixture distribution from \Cref{def:mixtureDistribution}.
    }
    Thus, setting $\Binit=0.05$ is our recommendation (but one could also use any other small value including zero).
    \item \textbf{Parameter $\littleConst$:} (\texttt{init\_little\_const} in our code) prevents weights from being initialized to zero, i.e., $\littleConst>0$ guarantees that no initial weight will be \emph{exactly} zero, see \cref{itm:NonZeorWeights} in \Cref{thm:ScalingInitConstantEandV}. Conversely setting $\littleConst=0$ leads weights to be zero with probability $(1-p_k)$ (see \cref{itm:ZeorWeights} in \Cref{thm:ScalingInitConstantEandV}).
    If one chooses $\littleConst$ too large $B_k$ can become too large (see \cref{eq:def:Bk}). Thus, setting $\littleConst=0.1$ is our recommendation (similarly to $\Biasinit$ any other small value is also admissible).
\end{enumerate}

In the last section of our \MathematicaScript{} we provide an interactive plot that shows how $B_k,A_k,p_k,\Varco{o^{i,k}_j-b^{i,k}_j}{z^{i,k-1}=(1,\dots,1)^\top}$ and  $\Eco{o^{i,k}_j}{z^{i,k-1}=(1,\dots,1)^\top}$ depend on $d^{i,k-1}$ for different choices of our hyperparameters.

\subsection{Visualization for Wider MVNNs}\label{subsec:appendix_Visualization for Wider MVNNs}
In this section, we provide an additional visualization for a wider MVNN that uses our proposed new initialization method.

In \Cref{fig:appendix_1d_path_plot_initialization}, we present the results. \Cref{fig:appendix_1d_path_plot_initialization} confirms that our initialization method also properly works for an even larger MVNN-architecture with two hidden layers with 256 neurons per hidden layer. While the problems of the generic initialization methods described in \Cref{subsec:New MVNN Random Initialization} increase as the number of neurons increase, our initialization method can deal with an arbitrarily large number of neurons.

\section{MILP}\label{sec:appendix_MVNN_based_MILP}
In this section, we provide more details on \Cref{subsec:New MVNN-based MIP}.

\subsection{Proof of \Cref{thm:milp}}\label{subsec:appendix_proof_MILP}
In this section, we provide the proof of \Cref{thm:milp}. 

First, we show in \Cref{lem:mvnn_layer} how to encode an arbitrary single hidden MVNN layer into multiple linear constraints. 
For this fix a bidder $i\in N$ and an arbitrary layer $k\in \{1,\ldots,K_i-1\}$. Recall, that $z^{i, k-1}\in \R^{d^{i,k-1}}$ denotes the output of the previous layer (with $z^{i, 0}$ being equal to the input $x\in \R^{d^{i,0}}=\R^{m}$) and that $o^{i, k}\coloneqq W^{i, k}z^{i, k-1} + b^{i, k}$ denotes the \emph{pre}-activated output of the $k$\textsuperscript{th} layer with $l^{i, k}\le o^{i, k} \le u^{i, k}$, where the tight lower/upper bound $l^{i, k}$/ $u^{i, k}$ can be computed by forward-propagating the empty/full bundle. Then the following Lemma holds:\footnote{All vector inequalities should be understood component-wise.}

\begin{lemma}\label{lem:mvnn_layer}
    Consider the following set of linear constraints:
    \begin{align}
        & z^{i, k}\le \alpha^{i,k}\cdot t^{i, k}\label{eq:lemma(i)}\\
        & z^{i, k}\le o^{i, k} - l^{i,k}\cdot(1-\alpha^{i,k})\label{eq:lemma(ii)}\\
        & z^{i, k}\ge \beta^{i,k}\cdot t^{i, k}\label{eq:lemma(iii)}\\
        & z^{i, k}\ge o^{i, k} + (t^{i, k}-u^{i, k})\cdot \beta^{i,k} \label{eq:lemma(iv)}\\
        & \alpha^{i,k}\in \{0,1\}^{d^{i,k}},\, \beta^{i,k}\in \{0,1\}^{d^{i,k}}\label{eq:lemma(v)}.
    \end{align}
    Then it holds for the output of the $k$\textsuperscript{th} layer $\varphi_{0,t^{i, k}}(o^{i, k})=z^{i,k}$.
\end{lemma}
\begin{figure}[t!]
    \centering
    \resizebox{\columnwidth}{!}{
    \includegraphics[trim=160 170 400 150, clip]{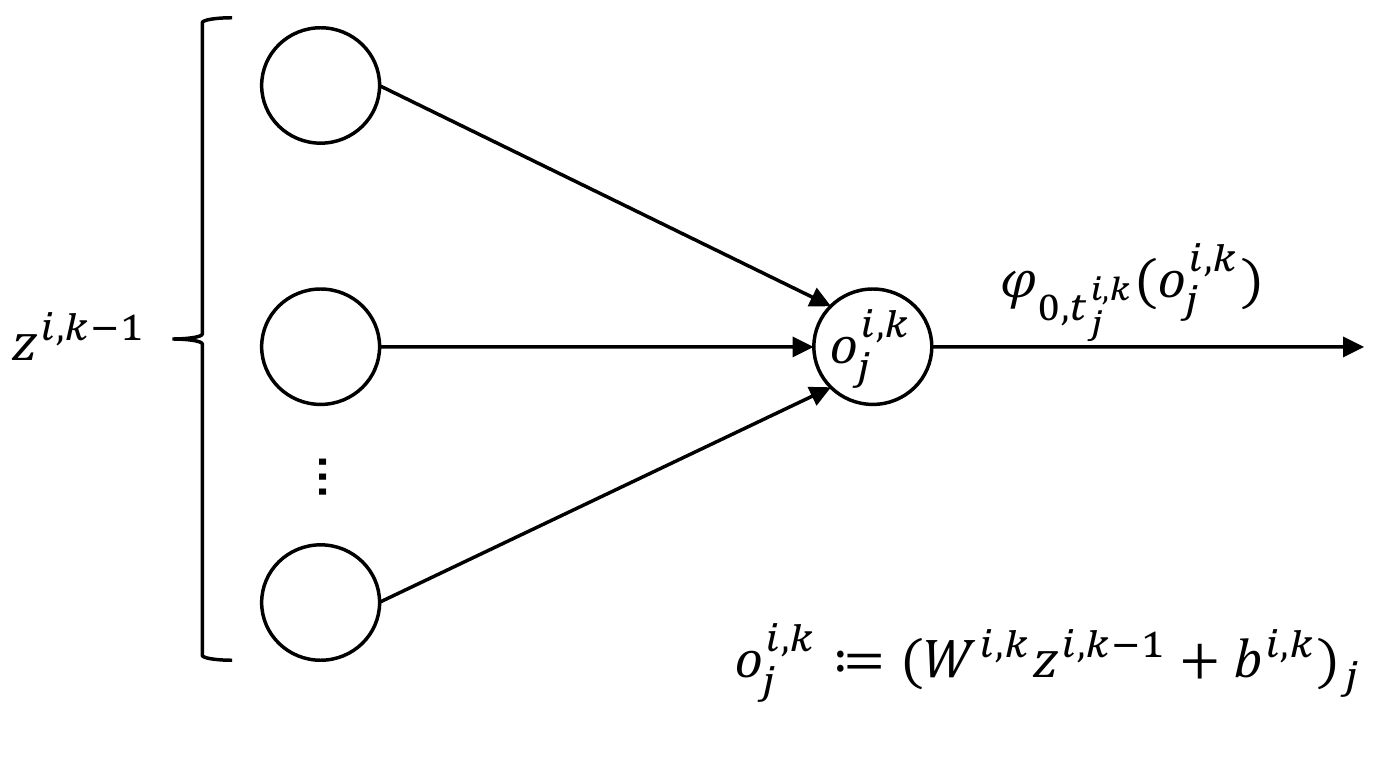}
    }
    \vskip -0.35cm
    \caption{Schematic representation of the $j$\textsuperscript{th} neuron in the $k$\textsuperscript{th} layer. Shown are the output $z^{i,k-1}$ of the previous $(k-1)$\textsuperscript{st} layer, the pre-activated output $o^{i,k}_j$ of the $j$\textsuperscript{th} neuron in the $k$\textsuperscript{th} layer, and the output of the $j$\textsuperscript{th} neuron of the $k$\textsuperscript{th} layer after applying bRELU $\varphi_{0,t_j^{i,k}}(o^{i,k}_j)=\min(t_j^{i,k},\max(0,o^{i,k}_j))$ .}
    \label{fig:appendix:mip_lemma_schematic_representation}
    \vspace{0.5cm}
\end{figure}
\begin{proof}
Recall, that $o^{i,k}\coloneqq W^{i,k}z^{i,k-1}+b^{i,k}$ denotes the pre-activated output of the $k$\textsuperscript{th} layer. Let $j\in \{1,\ldots,d^{i,k}\}$ denote the $j$\textsuperscript{th} neuron of that layer and let $l_j^{i,k}\le o_j^{i,k}\le u_j^{i,k}$, where $l_j^{i,k}$ and $u_j^{i,k}$ are the tight box bounds computed by forward propagating the empty, i.e., $x=(0,\ldots,0)$, and the full, i.e., $x=(1,\ldots,1)$ bundle (see \cite[Appendix C.5 and Fact C.1]{weissteiner2022monotone}). Moreover, let $\varphi_{0,t_j^{i,k}}(o_j^{i,k})=\min(t_j^{i,k},\max(0,o_j^{i,k}))$ with $t_j^{i,k}\ge 0$ be the output of the $j$\textsuperscript{th} neuron in the $k$\textsuperscript{th} layer.  In \Cref{fig:appendix:mip_lemma_schematic_representation}, we present a schematic representation for a single neuron. In \Cref{fig:appendix:bReLU}, we present an example of the bReLU activation function.

We distinguish the following three exclusive cases:
\begin{itemize}
     \item \textbf{Case 1:} $o^{i,k}_j < 0$, i.e., the red line segment in \Cref{fig:appendix:bReLU}. Per definition it follows that $ \varphi_{0,t_j^{i,k}}(o^{i,k}_j) = 0$.
     Setting $\alpha_j^{i,k}=\beta_j^{i,k}=0$ in \Crefrange{eq:lemma(i)}{eq:lemma(v)} implies that $z^{i, k}_j= 0$.

     \item \textbf{Case 2:} $o^{i,k}_j \in [0, t^{i,k}_j]$, i.e., the blue line segment in \Cref{fig:appendix:bReLU}. Per definition it follows that $\varphi_{0,t_j^{i,k}}(o^{i,k}_j) = o^{i,k}_j$.
     Setting $\alpha_j^{i,k}=1$ and  $\beta_j^{i,k}=0$ in \Crefrange{eq:lemma(i)}{eq:lemma(v)} implies that $z^{i, k}_j =  o^{i,k}_j$.
     \item \textbf{Case 3:} $o^{i,k}_j > t^{i,k}_j$, i.e., the green line segment in \Cref{fig:appendix:bReLU}. Per definition it follows that $\varphi_{0,t_j^{i,k}}(o^{i,k}_j) = t_j^{i,k}$
    Setting $\alpha_j^{i,k}=\beta_j^{i,k}=1$ in \Crefrange{eq:lemma(i)}{eq:lemma(v)} implies that $z^{i, k}_j  =  t^{i,k}_j$.
\end{itemize}

Thus, in total $z^{i, k}=\varphi_{0,t_j^{i,k}}(o^{i,k}_j)$.
\end{proof}

\begin{figure}[t!]
    \centering
    \resizebox{\columnwidth}{!}{
    \includegraphics{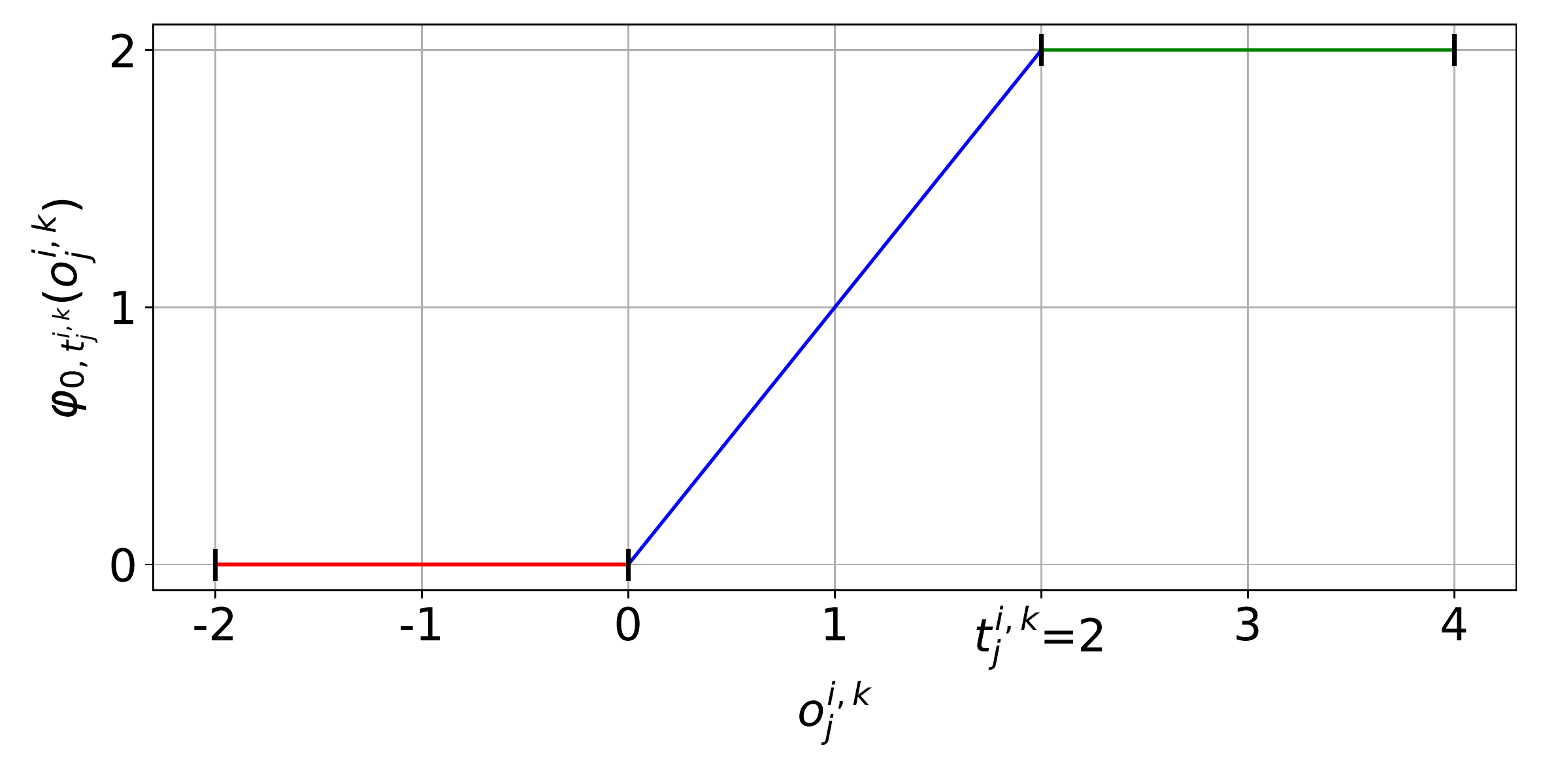}
    }
    \vskip -0.35cm
    \caption{Example plot for agent $i$, layer $k$ and neuron $j$ of the bReLU activation function $\varphi_{0,t_j^{i,k}}(\cdot)\coloneqq\min(t_j^{i,k},\max(0,\cdot))$ \cite{weissteiner2022monotone} with cutoff $t_j^{i,k}=2$ in the interval $[-2,4]$. Shown are the pre-activated output $o_j^{i,k}$ of the $j$\textsuperscript{th} neuron in the $k$\textsuperscript{th} layer (x-axis) and the output of the $j$\textsuperscript{th} neuron in the  $k$\textsuperscript{th} layer after applying bRELU (y-axis).}
    \label{fig:appendix:bReLU}
    \vspace{0.5cm}
\end{figure}
Using \Cref{lem:mvnn_layer}, we can now proof \Cref{thm:milp} which provides our new and more succinct MILP formulation.

\begin{proof}
Consider the ML-based WDP from \Cref{eq:ML-WDP}. For each bidder $i\in N$, we first set
$z^{i,0}$ equal to the input bundle $a_i$. Then we proceed by using
\Cref{lem:mvnn_layer} for $k=1$, i.e., we reformulate the output of the 1\textsuperscript{st} layer as the linear \Crefrange{eq:lemma(i)}{eq:lemma(iv)}. We iterate
this procedure until we have reformulated the last hidden layer, i.e,
layer $k=K_i-1$. Doing so for each bidder $i\in N$ yields the desired MILP formulation from \Cref{thm:milp}.
\end{proof}

\subsection{Removing Constraints via Box Bounds}\label{subsec:app:Removing Constraints with Box Bounds}
Let $l_j^{i,k}\le o_j^{i,k}\le u_j^{i,k}$, where $l_j^{i,k}$ and $u_j^{i,k}$ are the \emph{tight} box bounds computed by forward propagating the empty, i.e., $x=(0,\ldots,0)$, and the full, i.e., $x=(1,\ldots,1)$ bundle (see \cite[Appendix C.5 and Fact C.1]{weissteiner2022monotone}).

In the following cases one can remove the constraints and corresponding variables in \Cref{lem:mvnn_layer} and thus also in \Cref{thm:milp}.
\begin{itemize}[leftmargin=*,topsep=5pt,partopsep=0pt, parsep=0pt]
    \item \textbf{Case 1:} $0\le t_j^{i,k} < l_j^{i,k}\le u_j^{i,k}$.
    Then one can simply set 
    \begin{align}
        z_j^{i,k}\coloneqq t_j^{i,k}
    \end{align} and remove the $j$\textsuperscript{th} components from \Crefrange{eq:lemma(i)}{eq:lemma(v)} of the corresponding  layer $k$ ($o_j^{i,k}$ lies for sure in the green line segment in \Cref{fig:appendix:bReLU}).
    \item \textbf{Case 2:} $l_j^{i,k} \le u_j^{i,k}< 0\le t_j^{i,k}$.
    Then one can simply set 
    \begin{align}
        z_j^{i,k}\coloneqq 0
    \end{align} and remove the $j$\textsuperscript{th} components from \Crefrange{eq:lemma(i)}{eq:lemma(v)} of the corresponding layer $k$ ($o_j^{i,k}$ lies for sure in the red line segment in \Cref{fig:appendix:bReLU}).
    \item \textbf{Case 3:} $0\le l_j^{i,k}\le u_j^{i,k}\le t_j^{i,k}$.
    Then one can simply set 
    \begin{align}
        z_j^{i,k}\coloneqq o_j^{i,k}
    \end{align} and remove the $j$\textsuperscript{th} components from \Crefrange{eq:lemma(i)}{eq:lemma(v)} of the corresponding layer $k$ ($o_j^{i,k}$ lies for sure in the blue line segment in \Cref{fig:appendix:bReLU}).
   \item \textbf{Case 4:} $0\le l_j^{i,k}\le t_j^{i,k}< u_j^{i,k}$.
    Then one can set 
    \begin{align}
        \alpha_j^{i,k}\coloneqq 1
    \end{align} and only one binary decision variable for the $j$\textsuperscript{th} neuron of the $k$\textsuperscript{th} layer remains. ($o_j^{i,k}$ lies for sure in the union of the blue and green line segment in \Cref{fig:appendix:bReLU}).
    \item \textbf{Case 5:} $l_j^{i,k}\le 0< u_j^{i,k}\le t_j^{i,k}$.
    Then one can set 
    \begin{align}
        \beta_j^{i,k}\coloneqq 0
    \end{align} and only one binary decision variable for the $j$\textsuperscript{th} neuron of the $k$\textsuperscript{th} layer remains. ($o_j^{i,k}$ lies for sure in the union of the red and blue line segment in \Cref{fig:appendix:bReLU}).
\end{itemize}

\subsection{MILP for MVNNs with Linear Skip Connection}\label{subsec:appMILP for MVNNs with Linear Skip Connection}
In this section, we provide a simple extension of the MILP in \Cref{thm:milp} for MVNNs with a linear skip connection. First, we define MVNNs with a linear skip connection.

\begin{definition}[MVNN with Linear Skip Connection]\label{def:lskipMVNN}
		An MVNN with linear skip connection $\lskipMVNNi{}:\X \to \mathbb{R}_+$  for agent $i\in N$ is defined as
		\begin{align}\label{eq:MVNNlskip}
		\lskipMVNNi{x} &=W^{i,K_i}\varphi_{0,t^{i,K_i}}\left(\ldots\varphi_{0,t^{i,1}}(W^{i,1}x+b^{i,1})\ldots\right)\\
		&+W^{i,0}x\notag,
		\end{align}
		\begin{itemize}[leftmargin=*,topsep=0pt,partopsep=0pt, parsep=0pt]
		\item $K_i+1\in\mathbb{N}$ is the number of layers ($K_i-1$ hidden layers),
		\item $\{\varphi_{0,t^{i, k}}{}\}_{k=1}^{K_i-1}$ are the MVNN-specific activation functions with cutoff $t^{i, k}>0$, called \emph{bounded ReLU (bReLU)}:
		\begin{align}\label{itm:app:MVNNactivation}
		\varphi_{0,t^{i, k}}(\cdot)\coloneqq\min(t^{i, k}, \max(0,\cdot))
		\end{align}
		\item $W^i\coloneqq (W^{i,k})_{k=0}^{K_i}$ with $W^{i,k}\ge0$ and $b^i\coloneqq (b^{i,k})_{k=1}^{K_i-1}$ with $b^{i,k}\le0$ are the \emph{non-negative} weights and \emph{non-positive} biases of dimensions $d^{i,k}\times d^{i,k-1}$ (except $W^{i,0}$ which is of dimension $d^{i,K_i}\times d^{i,0}$) and $d^{i,k}$, whose parameters are stored in $\theta=(W^i,b^i)$, where $W^{i,0}$ represents the linear skip connection.
		\end{itemize}
\end{definition}
For the MILP of an MVNN with linear skip connection the only thing that changes is the objective, i.e., one needs to replace \Cref{eq:milp_objective} in \Cref{thm:milp} with 
\begin{align}
    &\max\limits_{a\in \F, z^{i,k},\alpha^{i,k},\beta^{i,k}}\left\{\sum_{i \in N} W^{i, K_{i}} z^{i, K_{i}-1}+W^{i, 0}z^{i,0}\right\}\label{eq:milp_objective_linear_skip_mvnn}
\end{align}

\section{Experiment Details}\label{sec:Experiment Details}
In this section, we present all details of our experiments from \Cref{sec:experiments}.

\subsection{SATS Domains}\label{subsec:appendix_SATS_domains}
In this section, we provide a more detailed overview of the three SATS domains\footnote{We do not consider the simplest of all SATS domains, i.e., the \textit{Global Synergy Value Model (GSVM)} \cite{goeree2010hierarchical}, since prior work already achieves 0\% efficiency loss without integrating any notion of uncertainty \cite{weissteiner2022monotone}, and thus GSVM can be seen as already ``solved''.}, which we use to experimentally evaluate BOCA:
\begin{itemize}[leftmargin=*,topsep=0pt,partopsep=0pt, parsep=0pt]
\item \textbf{Local Synergy Value Model (LSVM)} \cite{scheffel2012impact} has $18$ items, $5$ \emph{regional} and $1$ \emph{national bidder}. Complementarities arise from spatial proximity of items.
\item \textbf{Single-Region Value Model (SRVM)} \cite{weiss2017sats} has $29$ items and $7$ bidders (categorized as  \emph{local}, \emph{high frequency} \emph{regional}, or \emph{national}) and models large UK 4G spectrum auctions.
\item \textbf{Multi-Region Value Model (MRVM)} \cite{weiss2017sats} has $98$ items and $10$ bidders (\emph{local}, \emph{regional}, or \emph{national}) and models large Canadian 4G spectrum auctions.
\end{itemize}
In the efficiency experiments in this paper (i.e., \Cref{tab:efficiency_loss_mlca}, 
 \Cref{tab:efficiency_loss_mlca_appendix}, \Cref{tab:efficiency_loss_mlca_appendix_onlyMean}, 
 \Cref{tab:efficiency_loss_mlca_appendix_20iQ}, and \Cref{tab:efficiency_loss_mlca_appendix_20iQ_onlyMean}), we instantiated for each SATS domain the $50$ synthetic CA instances with the seeds $\{10001,\ldots,10050\}$. We used \href{https://github.com/spectrumauctions/sats/releases/}{SATS version 0.8.0}. All experiments were conducted on a compute cluster running Debian GNU/Linux 10 with Intel Xeon E5-2650 v4 2.20GHz processors with 24 cores and 128GB RAM and Intel E5 v2 2.80GHz processors with 20 cores and 128GB RAM and Python 3.7.10.

\subsection{Hyperparameter Optimization}\label{subsec:appendix_hpo}

\begin{table*}[htbp]
    \robustify\bfseries
	\centering
	\begin{sc}
	\resizebox{1\textwidth}{!}{
\begin{tabular}{llll}
\toprule
\multicolumn{1}{c}{\textbf{Category}}&\multicolumn{1}{c}{\textbf{Hyperparameter}}               & \multicolumn{1}{c}{\textbf{Range}} & \multicolumn{1}{c}{\textbf{Log-Uniform Sampling}} \\ \midrule
Data &Number of Training Data Points: $|\Dtr|$              & \begin{tabular}[c]{@{}l@{}}LSVM: 50\\ SRVM: 100\\ MRVM: 100\end{tabular}                                                                                                          &               \\
\cmidrule(l{2pt}r{2pt}){2-4}
&Number of Test Data Points: $|\Dtest|$               & \begin{tabular}[c]{@{}l@{}}LSVM: $2^{10}$\\ SRVM: $2^{13}$\\ MRVM: $2^{13}$\end{tabular}                                                                                    &               \\\cmidrule(l{2pt}r{2pt}){2-4}
&Number of SATS Instances                & 100           & \\\midrule
Generic MVNN&Architecture: (\#neurons per hidden layer, \#hidden layers)                 & \begin{tabular}[c]{@{}l@{}}LSVM: \{(96, 1), (32, 2), (16, 3)\}\\ SRVM: \{(32, 2), (16, 3)\}\\ MRVM: \{(96, 1), (64, 1), (20, 2)\}\end{tabular} &               \\ \cmidrule(l{2pt}r{2pt}){2-4}
&Linear Skip Connection (see \Cref{def:lskipMVNN})   & \{True, False\}                        &               \\ \cmidrule(l{2pt}r{2pt}){2-4}
&Batch Size                       & $\{\nicefrac{|\Dtr|}{4}, \nicefrac{|\Dtr|}{2},|\Dtr|\}$ &               \\ \cmidrule(l{2pt}r{2pt}){2-4}
&Epochs                       & $\{4000 \frac{\textsc{Batch Size}}{|\Dtr|},4500\frac{\textsc{Batch Size}}{|\Dtr|},5000\frac{\textsc{Batch Size}}{|\Dtr|},\ldots, 8000\frac{\textsc{Batch Size}}{|\Dtr|}\}$ &               \\\cmidrule(l{2pt}r{2pt}){2-4}
&Dropout Probability                & {[}0, 0.8{]}       &               \\\cmidrule(l{2pt}r{2pt}){2-4}
&Dropout Probability Decay Factor     & {[}0.75, 1.0{]}  &               \\ \midrule
Generic Loss & Optimizer                & Adam &               \\\cmidrule(l{2pt}r{2pt}){2-4}
&Learning Rate                & \begin{tabular}[c]{@{}l@{}}GSVM: {[}0.0002, 0.002{]}\\ LSVM: {[}0.001, 0.01{]}\\ SRVM: {[}0.0008, 0.006{]}\\ MRVM: {[}0.0007, 0.004{]}\end{tabular}                                          & Yes           \\\cmidrule(l{2pt}r{2pt}){2-4}
&L2-Regularization: $\lambda$ (see \Cref{eq:NOMU_Loss_Extension})                          & {[}1e-10, 1e-3{]} & Yes           \\\cmidrule(l{2pt}r{2pt}){2-4}
&Smooth L1-Loss $\beta$ (see \Cref{def:appendix_smoothL1Loss}) & \{1/32, 1/64, 1/128\} &               \\\cmidrule(l{2pt}r{2pt}){2-4}
&Clip Grad Norm               & [1e-6, 1]                         & Yes           \\ \midrule
NEW NOMU Loss&Number of Artificial Input Data Points: $|\Dart|$ (see \Cref{subsec:NOMU for Combinatorial Assignment})                      & \{64,80,96,\ldots,512\}        &            \\ \cmidrule(l{2pt}r{2pt}){2-4}
&$\musqr$ (see \Cref{eq:NOMU_Loss_Extension})                      & 1       &            \\\cmidrule(l{2pt}r{2pt}){2-4}
&$\muexp$ (see \Cref{eq:NOMU_Loss_Extension})                      & {[}1e-6, 5e-1{]}        & Yes           \\\cmidrule(l{2pt}r{2pt}){2-4}
&$\piMean$ (see \Cref{eq:NOMU_Loss_Extension})               & {[}64, 256{]}           & Yes              \\\cmidrule(l{2pt}r{2pt}){2-4}
&$\piOneuUB$ (see \Cref{eq:NOMU_Loss_Extension})                & 0.25           &               \\\cmidrule(l{2pt}r{2pt}){2-4}
&$\cexp$ (see \Cref{eq:NOMU_Loss_Extension})                       & {[}64, 256{]}           & Yes              \\ \midrule
MVNN Initialization&Random Initialized ts                 & Unif([0,1])                        &               \\\cmidrule(l{2pt}r{2pt}){2-4}
&Trainable ts                 & \{True, False\}                        &               \\\cmidrule(l{2pt}r{2pt}){2-4}
&Initial Expectation: $\Einit$ (see \Cref{sec:appendix:RecomendedHyperParametersInit})                       & {[}1, 2{]}                               &               \\\cmidrule(l{2pt}r{2pt}){2-4}
&Initial Variance: $\Vinit$ (see \Cref{sec:appendix:RecomendedHyperParametersInit})                    & {[}0.02, 0.16{]}                         &     Yes          \\\cmidrule(l{2pt}r{2pt}){2-4}
&Initial Bias: $\Biasinit$ (see \Cref{sec:appendix:RecomendedHyperParametersInit})                     &    0.05                     &               \\\cmidrule(l{2pt}r{2pt}){2-4}
&Initial ``b'' Constant: $\Binit$ (see \Cref{sec:appendix:RecomendedHyperParametersInit})                    &    0.05                     &               \\\cmidrule(l{2pt}r{2pt}){2-4}
&Initial ``Little'' Constant: $\littleConst$ (see \Cref{sec:appendix:RecomendedHyperParametersInit})                   &    0.1                     &               \\ \midrule \bottomrule
\end{tabular}
}
\end{sc}
\caption{Hyperparameter ranges used in our HPO for random search. If not explicitly stated otherwise, the ranges apply to all considered SATS domains.}
\label{tab:app:hpo_ranges}
\end{table*}

In this section, we provide details on the exact hyperparameter ranges that we used in our HPO. \Cref{tab:app:hpo_ranges} shows all hyperparameter ranges that we used. In the following, we explain selected hyperparameters:
\begin{itemize}
    \item \textsc{Dropout Probability Decay Factor}: After each epoch $t$ the dropout probability for the next epoch $t+1$ $p^{t+1}_{\text{drop}}$ is updated as: $p^{t+1}_{\text{drop}} = p^{t}_{\text{drop}}\cdot \kappa$, where $\kappa$ denotes this factor.
    \item \textsc{Clip Grad Norm}: Parameter for gradient clipping in \textsc{PyTorch} via \emph{torch.nn.utils.clip\_grad\_norm\_()}.
    \item \textsc{Random Initialized ts}: Uniform distribution, which is used to initialize the bReLU cutoffs $t^{i,k}$ i.i.d.\ uniformly at random, i.e.  $t^{i,k}\sim \text{Unif}(A,B)$ (setting $A=B$ makes those cutoffs deterministic).
    \item \textsc{Trainable ts}: If set to \textsc{True},
    the cutoffs of the bReLU activation function $\{t^{i,k}\}_{k=1}^{K_i-1}$ are learned (i.e., trained) during the training procedure of the corresponding MVNN.
\end{itemize}

\begin{table*}[htbp]
	\robustify\bfseries
	\centering
	\begin{sc}
	\resizebox{1\textwidth}{!}{
    \begin{tabular}{crrrrccc}
    \toprule
    \textbf{Domain} & \textbf{Quantile Parameter q}&$\boldsymbol{\Qinit}$&$\boldsymbol{\Qround}$&$\boldsymbol{\Qmax}$ & \multicolumn{1}{c}{\textbf{Efficiency Loss in \%\,\,\textdownarrow}}  & \multicolumn{1}{c}{\textbf{Revenue in \%\,\,\textuparrow}} & \multicolumn{1}{c}{\textbf{Runtime in Hours}}\\
    \midrule
    LSVM & 0.60& 40& 4& 100 & 0.69$\pm$\scriptsize\,0.41 & 74.73$\pm$\scriptsize\,3.68 & \hphantom{1}4.90\\
    & 0.75& 40& 4& 100&  0.69$\pm$\scriptsize\,0.44 & 75.07$\pm$\scriptsize\,3.71 & \hphantom{1}5.53\\
    & \ccell 0.90& \ccell40& \ccell4& \ccell100& \ccell0.39$\pm$\scriptsize\,0.30 & \ccell73.53$\pm$\scriptsize\,3.72 & \ccell15.64 \\
    & 0.95& 40& 4& 100&  0.40$\pm$\scriptsize\,0.35 & 73.88$\pm$\scriptsize\,3.93 & 15.58  \\
    \midrule
    SRVM & 0.60& 40& 4& 100&  0.16$\pm$\scriptsize\,0.04 & 54.34$\pm$\scriptsize\,1.48 & 24.67\\
    & \ccell0.75& \ccell40& \ccell4& \ccell100&  \ccell0.06$\pm$\scriptsize\,0.02 & \ccell54.22$\pm$\scriptsize\,1.46 & \ccell18.80\\
    & 0.90& 40& 4& 100&  0.54$\pm$\scriptsize\,0.08 & 53.89$\pm$\scriptsize\,1.44 & 33.90\\
    & 0.95& 40& 4& 100&  0.62$\pm$\scriptsize\,0.11 & 54.25$\pm$\scriptsize\,1.54 & 33.26\\
    \midrule
    MRVM & 0.60& 40& 4& 100&  7.88$\pm$\scriptsize\,0.43 & 41.81$\pm$\scriptsize\,1.06 & 61.48 \\
    & 0.75& 40& 4& 100&  8.44$\pm$\scriptsize\,0.43 & 41.89$\pm$\scriptsize\,0.93 & 34.91\\
    & \ccell0.90& \ccell40& \ccell4& \ccell100&  \ccell7.77$\pm$\scriptsize\,0.34 & \ccell42.04$\pm$\scriptsize\,0.89 & \ccell28.15\\
    & 0.95& 40& 4& 100&  7.98$\pm$\scriptsize\,0.34 & 42.28$\pm$\scriptsize\,1.00 & 27.92 \\
    \bottomrule
    \end{tabular}
}
    \end{sc}
    \vskip -0.2cm
    \caption{Detailed BOCA results. We present efficiency loss, relative revenue and runtime with our MVNN-based uUB $\uUBMi$ as $\mathcal{A}_i$. Shown are averages including a 95\%-normal-CI on a test set of $50$ instances in all three considered SATS domains. The best MVNN-based uUBs per domain (w.r.t. the quantile parameter $q$ based on the lowest efficiency loss are marked in grey.}
    \label{tab:efficiency_loss_mlca_appendix}
    \vskip 1cm
\end{table*}

\subsubsection{Evaluation Metric HPO}
We motivate our choice of the two terms in our evaluation metric~\Cref{eq:evaluation_metric_HPO} in the following way:
\begin{enumerate}
    \item\label{itm:TestSetTermEvaluationMetricHPO} The first term \[|\Dtest|^{-1}\hspace{-0.5cm}\sum\limits_{(x,y)\in \Dtest}\hspace{-0.5cm}\max\{(y\hspace{-0.08cm}-\hspace{-0.08cm}\uUBMi(x))q,(\uUBMi(x)\hspace{-0.08cm}-\hspace{-0.08cm}y)(1\hspace{-0.08cm}-\hspace{-0.08cm}q)\}\hspace{-0.08cm}\] of \Cref{eq:evaluation_metric_HPO} is the standard quantile-loss applied on the test data set. Achieving a low value in this evaluation metric is intuitively desirable since for values $q>0.5$ we penalize predictions that are too low more severely than predictions that are too high.
    It is also theoretically well motivated, since the true\footnote{By \enquote{true posterior} we denote the posterior coming from the \enquote{true prior} that we sample our value functions~$\hvi{}$ from.} posterior $q$-credible bound $\text{uUB}_\alpha(x)\coloneqq \inf\{y\in\R : \PP[{\hviPlain(x)\leq y | \Dtr}]\geq \alpha\} \ \forall x\in X$ would minimize this evaluation metric in expectation. As we average this term over 100 different value functions and as we use a large test data set for each of them, this is a good approximation for the expectation.
    \item The second term $\text{MAE}(\Dtr)$ of \Cref{eq:evaluation_metric_HPO} might appear to be counter-intuitive, because we are using the training data set. However, in BO it is particularly important to fit well through the noiseless training data points. First, the training data points in BO have already been chosen to lie in a region of potential maximizers. Second, in BO, relative uncertainty \citep[Appendix~A.2.1.]{heiss2022nomu} and particularly \emph{Desiderata D2} from \citep{heiss2022nomu} are more important than calibration as discussed in \citep[Appendix~D.2.3]{heiss2022nomu}. Adding a constant value to $\uUBMi$ would calibrate them but would not change the argmax~\Cref{eq:ML-based-WDP} (i.e., would not change the selected queries). However, the \ref{itm:TestSetTermEvaluationMetricHPO}\textsuperscript{st} term of our evaluation metric \Cref{eq:evaluation_metric_HPO} alone would assign quite low values to $\uUBMi$ of the form $\meanMi+c$. Fortunately, the second term $\text{MAE}(\Dtr)$ prevents \Cref{eq:evaluation_metric_HPO} from assigning low values to $\uUBMi$ of the form $\meanMi+c$.
\end{enumerate}

\subsection{Details MVNN Training}\label{subsec:Details MVNN-Training}
Both for the HPO as well as when running our efficiency experiments, we use the following two techniques to achieve numerically more robust results.
\begin{enumerate}
    \item At the end of the training procedure we use the best weights from all epochs, and not the ones from the last epoch.
    \item If at the end of the training procedure the $R^2$ (coefficient of determination) was below $0.9$ on the training set, we retrain once and finally select the model with the best performance across these two attempts. 
\end{enumerate}

\subsection{Details MILP Parameters}\label{subsec:DetailsMILP Parameters}
In our experiments we use for all MILPs \textsc{CPLEX} version 20.1.0. Furthermore, we set the following MILP parameters: $\texttt{time\_limit}=600$ sec, $\texttt{relative\_gap}=0.005$, $\texttt{integrality\_tol}=1e-06$, $\texttt{feasibility\_tol}=1e-09$. Thus, the MILP solver stops as soon as it has found a feasible integer solution proved to be within 0.5\% of optimality and otherwise stops after 600 seconds. \textsc{CPLEX} is automatically proving upper bounds for the relative gap (using duality) while approximating the solution for the MILP. Thus, when \textsc{CPLEX} finds a solution with a relative gap of at most $0.5\%$, we have a guarantee that the solution found by \textsc{CPLEX} is \emph{at most} $0.5\%$ worse than the unknown true solution. However, we have no a priori theoretical guarantee that \textsc{CPLEX} is always able to find such a solution and such a bound within 600 sec. Note that we set the time limit to 600 sec to allow researchers with a limited budget to reproduce our results hundreds of times with different seeds, while achieving already SOTA. For an actual spectrum auction the additional costs for increasing the computational budget by a factor of 10 would be negligible compared to the statistical significant increase of 200 million USD of revenue gained from using BOCA instead of MVNN in the case of MRVM for example (see \Cref{subsec:revenue}). In our experiments, the median relative gaps\footnote{In \texttt{CPLEX}, the \enquote{relative gap} always refers to the proven upper bound of the relative gap.} across all runs per domain in \Cref{tab:efficiency_loss_mlca} were: LSVM=0.004987, SRVM=0.005000, MRVM=0.004991, indicating that in most of the runs the MILP solver found within the time limit an optimal solution within tolerance.

\subsection{Details BOCA Results}\label{subsec:Detailed MLCA Results}
In this section, we provide detailed efficiency loss results of BOCA. Specifically, we present in \Cref{tab:efficiency_loss_mlca_appendix} efficiency loss results for all four HPO winner configurations (with respect to the evaluation metric based on the four different quantile parameters $q\in \{0.6,0.75,0.9,0.95\}$). Furthermore, we present the \emph{relative revenue} $\sum_{i\in N}p(R)_i/V(a^*)$ of an allocation $a^*_R\in \mathcal{F}$ and payments $p(R) \in \mathbb{R}^n_+$ determined by BOCA when eliciting reports $R$ as well as the average runtime in hours for a single instance (i.e, how long it would take to run a single auction). Note that since we stop BOCA when we have already found an allocation with 0\% efficiency loss, the relative revenue numbers (in LSVM and SRVM) are pessimistic and typically increase if we let BOCA run until $\Qmax$ is reached. For the runtime results the opposite holds.

\subsection{Revenue}\label{subsec:revenue}
Comparing the revenue of BOCA (see \Cref{tab:efficiency_loss_mlca_appendix}) to the revenue of MVNN-MLCA and NN-MLCA in \citep[Table 6]{weissteiner2022monotone}, we see that overall the mean relative revenue is as good or better than SOTA. For LSVM (and SRVM) this comparison could be flawed because we stop computing further queries when reaching 0\% efficiency loss as described above.

However, for MRVM, we always compute all 100 queries as can be seen in \citep[Table 6]{weissteiner2022monotone}. Thus MRVM allows for a fair comparison of the relative revenue. In MRVM, BOCA's average relative revenue is $7.58\%$ points higher than the one of MVNN-MLCA
, which corresponds to $\sim\!\!400$ million USD in this domain.\footnote{The revenue of the 2014 Canadian 4G auction was $5.27$ billion USD \citep{ausubel2017practical}.
If one accumulates the revenue of all spectrum auctions, one would obtain significantly larger values.}
A pairwise $t$-test with null hypothesis of equal means resulted in $p_{\text{VAL}}=5\mathrm{e}{-10}$.
Moreover, with high statistical significance ($p_{\text{VAL}}=2\mathrm{e}{-4}$), BOCA achieves on average more than 200 million USD \emph{more} than MVNN-MLCA.
The comparison of BOCA to NN-MLCA is also in favour of BOCA, but much closer (and BOCA outperforms NN-MLCA in terms of social welfare with a $p_\text{VAL}=2\mathrm{e}{-5}$, see \Cref{tab:efficiency_loss_mlca}).

The strength of BOCA with respect to revenue is quite intuitive since each query that explores regions of high uncertainty in the bundle space is beneficial for all economies, while \enquote{exploiting} an allocation which leads to high efficiency in one certain economy is mainly beneficial for this certain economy (see \Cref{sec:appendix_A Machine Learning powered ICA} for the definition of main and marginal economies). As discussed in \citep[Appendix E.3]{weissteiner2022monotone}, MLCA queries the main economy more often than any other economy and revenue is high if the social welfare in the marginal economies is high relative to the social welfare in the main economy. Thus, exploration favours high revenue in settings (such as ours) where the main economy is queried more often than the marginal economies.

\begin{figure*}
    \centering
    \resizebox{\textwidth}{!}{
    \includegraphics{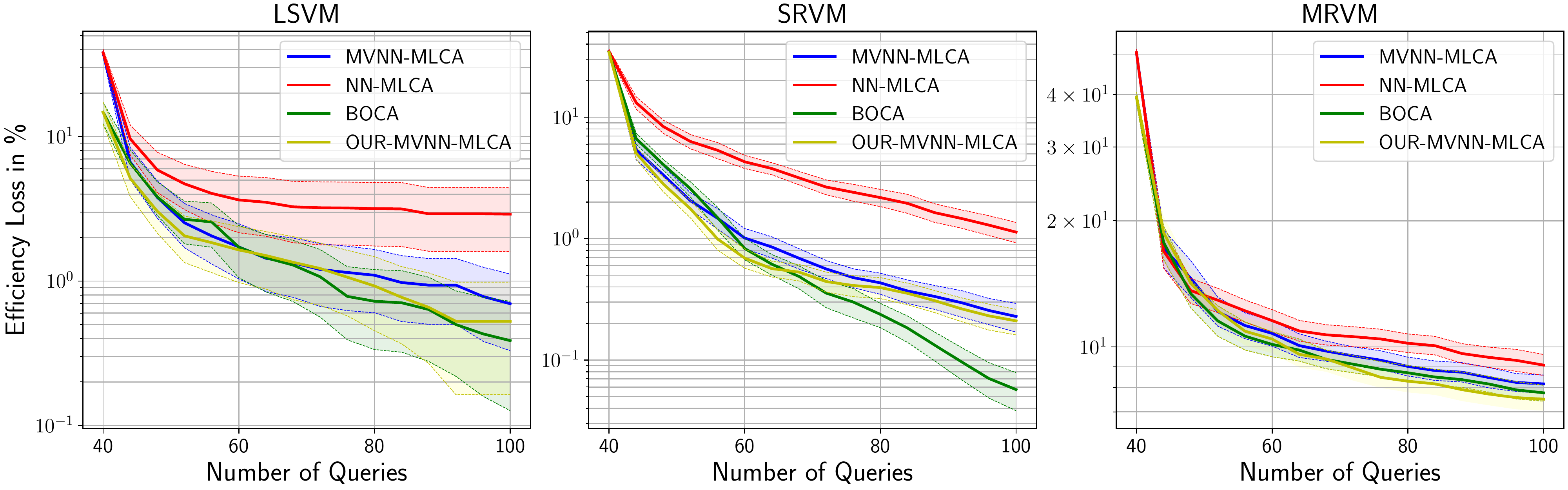}
    }
    \caption{Efficiency loss paths (i.e., regret plots) for $\Qinit=40$ of OUR-MVNN-MLCA winners (yellow) from \Cref{tab:efficiency_loss_mlca_appendix_onlyMean}   compared to BOCA winners (green) from \Cref{tab:efficiency_loss_mlca_appendix} and to the results presented in \citep{weissteiner2022monotone} of MVNNs (blue) and plain NNs (red). Shown are averages with 95\% CIs over 50 instances.
    }
    \label{fig:efficiency_path_plot_summary_appendix}
\end{figure*}
\subsection{Understanding BOCA's Performance Increase}\label{subsec:Ablation Study - When is Exploration Needed}
In this section, we present further efficiency loss results when using only the mean MVNN $\meanMi$ of $\NOMUi$ in MLCA as $\Ai{}$. We call this method OUR-MVNN-MLCA. Note that OUR-MVNN-MLCA does not integrate any explicit notion of uncertainty and is thus the same method as MVNN-MLCA from \cited{weissteiner2022monotone}, but now with our new proposed parameter initialization method (see \Cref{subsec:New MVNN Random Initialization}) and optimized with our HPO. Thus, this experiment investigates how much of the performance gain is attributed to the integration of our notion of uncertainty compared to the other changes we made (i.e., our parameter initialization method and our HPO).

In \Cref{fig:efficiency_path_plot_summary_appendix}, we present the efficiency loss path plots for OUR-MVNN-MLCA (yellow) compared to the results presented in \Cref{fig:efficiency_path_plot_summary} in the main paper. For each domain we use for OUR-MVNN-MLCA (yellow) the best mean models $\meanMi$ from $\NOMUi$ based on the smallest efficiency loss (i.e., the grey-marked winners from \Cref{tab:efficiency_loss_mlca_appendix_onlyMean}).

As expected, we see that BOCA, i.e., integrating a notion of uncertainty, is as good or better than OUR-MVNN-MLCA, i.e., only using the mean model. This effect is statistically significant in SRVM, while in LSVM and MRVM both lead to results that are statistically on par. The degree to which exploration via a notion of uncertainty is beneficial depends on intrinsic characteristics of the domain (e.g., the dimensionality or multi-modality of the objective function). Specifically, in MRVM where the query budget of $\Qmax=100$ is extremely small compared to the dimensionality of the domain (i.e., MRVM has $m=98$ items and $n=10$ bidders thus the dimensionality is $980=98\cdot10$), it appears that exploitation might be beneficial compared to adding exploration~\citep{de2021greed} and the power of adding exploration may reveal itself only when increasing the query budget. However, in LSVM and SRVM ($m=18$ and $m=29$), we see that adding exploration with an uUB $\uUBMi$ decreases the efficiency loss. Finally, these results also suggest that our proposed new parameter initialization method for MVNNs discussed in \Cref{subsec:New MVNN Random Initialization} tends to be better than a simple generic one, i.e, OUR-MVNN-MLCA is in every domain on average better than MVNN-MLCA.

Finally, in \Cref{tab:efficiency_loss_mlca_appendix_onlyMean} we present the complete results of this experiment for all quantile parameters $q$ per domain. 

We refer to \Cref{subsec:Reduced Number of Initial Queries} for an additional comparison of BOCA and OUR-MVNN-MLCA for a reduced number of initial queries.

\begin{table*}
	\robustify\bfseries
	\centering
	\begin{sc}
	\resizebox{1\textwidth}{!}{
    \begin{tabular}{crrrrccc}
    \toprule
    \textbf{Domain} & \textbf{Quantile Parameter q}&$\boldsymbol{\Qinit}$&$\boldsymbol{\Qround}$&$\boldsymbol{\Qmax}$ & \multicolumn{1}{c}{\textbf{Efficiency Loss in \%\,\,\textdownarrow}}  & \multicolumn{1}{c}{\textbf{Revenue in \%\,\,\textuparrow}} & \multicolumn{1}{c}{\textbf{Runtime in Hours}}\\
    \midrule
    LSVM & 0.60& 40& 4& 100 &  0.66$\pm$\scriptsize\,0.44 & 74.93$\pm$\scriptsize\,3.48 & \hphantom{1}4.62\\
    & 0.75& 40& 4& 100&   0.61$\pm$\scriptsize\,0.42 & 75.14$\pm$\scriptsize\,3.53 & \hphantom{1}4.16\\
    & \ccell 0.90& \ccell40& \ccell4& \ccell100&\ccell0.53$\pm$\scriptsize\,0.41 & \ccell73.80$\pm$\scriptsize\,3.88 & \ccell10.44\\
    & 0.95& 40& 4& 100& 0.56$\pm$\scriptsize\,0.40 & 74.83$\pm$\scriptsize\,3.67 & 10.31\\
    \midrule
    SRVM & 0.60& 40& 4& 100& 0.22$\pm$\scriptsize\,0.04 & 54.24$\pm$\scriptsize\,1.51 & 21.18\\
    &  \ccell0.75& \ccell40& \ccell4&  \ccell100&\ccell0.21$\pm$\scriptsize\,0.05 & \ccell54.47$\pm$\scriptsize\,1.51 & \ccell17.81\\
    & 0.90& 40& 4& 100& 0.36$\pm$\scriptsize\,0.06 & 54.51$\pm$\scriptsize\,1.49 & 20.03\\
    & 0.95& 40& 4& 100&  0.35$\pm$\scriptsize\,0.05 & 54.51$\pm$\scriptsize\,1.48 & 19.14\\
    \midrule
    MRVM & 0.60& 40& 4& 100 & 7.86$\pm$\scriptsize\,0.42 & 41.91$\pm$\scriptsize\,0.91 & 41.51\\
    & 0.75& 40& 4& 100&  8.41$\pm$\scriptsize\,0.40 & 40.28$\pm$\scriptsize\,0.92 & \hphantom{1}6.89\\
    & \ccell0.90& \ccell40& \ccell4& \ccell100&\ccell7.51$\pm$\scriptsize\,0.47 & \ccell40.22$\pm$\scriptsize\,0.96 & \ccell\hphantom{1}4.91\\
    & 0.95& 40& 4& 100&  8.01$\pm$\scriptsize\,0.48 & 40.62$\pm$\scriptsize\,1.08 & \hphantom{1}9.03\\
    \bottomrule
    \end{tabular}
}
    \end{sc}
    \vskip -0.2cm
    \caption{Detailed OUR-MVNN-MLCA results. We present efficiency loss, relative revenue and runtime of OUR-MVNN-MLCA, i.e., MLCA with our MVNN-based mean $\meanMi$ as $\mathcal{A}_i$. Shown are averages including a 95\%-normal-CI on a test set of $50$ instances in all three considered SATS domains. The best $\meanMi$ per domain (w.r.t. the quantile parameter $q$) based on the lowest efficiency loss is marked in grey.}
    \label{tab:efficiency_loss_mlca_appendix_onlyMean}
\end{table*}
\subsection{Reduced Number of Initial Queries}\label{subsec:Reduced Number of Initial Queries}
In this section, we present results of BOCA and OUR-MVNN-MLCA for a reduced number of initial random queries. Specifically, we chose $\Qinit=20$ initial random queries (instead of $\Qinit=40$ that were selected in the main set of experiments following prior work; see \Cref{tab:efficiency_loss_mlca_appendix} and \Cref{tab:efficiency_loss_mlca_appendix_onlyMean}). All other parameters are left untouched.

Note that the previous literature used only $\Qinit=40$, since batching reduces the cost per queries in practice, i.e., not only the queries are costly but also the rounds are costly. Reducing $\Qinit$ from 40 to 20 increases the number of rounds from 15 to 20 (given that $\Qround=4$ queries, i.e., 3 marginal economy queries and 1 main economy query, are asked per round). Furthermore, the experiments get computationally more costly as we reduce $\Qinit$, because we need to perform more (MV)NN trainings and solve more MILPs.

Because of this, we only compare BOCA vs. OUR-MVNN-MLCA in this section (as in \Cref{subsec:Ablation Study - When is Exploration Needed}). Thus, in this section, we do not study the benefits of our new initialization method (\Cref{subsec:New MVNN Random Initialization}), and only focus on studying the benefits of incorporating our proposed uncertainty model (\Cref{subsec:NOMU for Combinatorial Assignment}).

In \Cref{tab:efficiency_loss_mlca_appendix_20iQ}, we present the BOCA results for $\Qinit=20$. To isolate the effect of the integrated uncertainty in BOCA, we present the corresponding results for OUR-MVNN-MLCA for $\Qinit=20$ in  \Cref{tab:efficiency_loss_mlca_appendix_20iQ_onlyMean}, (see \Cref{subsec:Ablation Study - When is Exploration Needed} for a description of this experiment setting).

\paragraph{Effect of $\Qinit$ Parameter} Comparing \Cref{tab:efficiency_loss_mlca_appendix_20iQ} to \Cref{tab:efficiency_loss_mlca_appendix}, we see that by reducing the randomly sampled initial queries from $\Qinit=40$ to $\Qinit=20$, BOCA' efficiency loss tends to be on average smaller for LSVM and MRVM and, perhaps surprisingly, larger for SRVM. However, by comparing the grey-marked winner models in each domain, we see that it does not make a statistically significant difference whether $\Qinit$ is chosen to be 20 or 40.

\paragraph{BOCA vs. OUR-MVNN-MLCA for $\Qinit=20$}Also when comparing BOCA to OUR-MVNN-MLCA for a reduced number of $\Qinit=20$ initial queries, i.e., comparing \Cref{tab:efficiency_loss_mlca_appendix_20iQ} to \Cref{tab:efficiency_loss_mlca_appendix_20iQ_onlyMean}, we find that in each domain the BOCA winner model and the OUR-MVNN-MLCA winner model perform statistically on par (even though the average efficiency loss of the BOCA winner model is always better than that of the OUR-MVNN-MLCA winner model).
This can also be seen in the efficiency loss paths (i.e., regret plots) shown in \Cref{fig:efficiency_path_plot_summary_appendix_qinit20}. Furthermore, \Cref{fig:efficiency_path_plot_summary_appendix_qinit20} suggests that for a small query budget of $\Qmax<60$, exploitation might be more important, while for query budgets larger than 80, i.e., $\Qmax>80$, exploration might pay off more.
\begin{table*}[htbp]
	\robustify\bfseries
	\centering
	\begin{sc}
	\resizebox{1\textwidth}{!}{
    \begin{tabular}{crrrrccc}
    \toprule
    \textbf{Domain} & \textbf{Quantile Parameter q}&$\boldsymbol{\Qinit}$&$\boldsymbol{\Qround}$&$\boldsymbol{\Qmax}$ & \multicolumn{1}{c}{\textbf{Efficiency Loss in \%\,\,\textdownarrow}}  & \multicolumn{1}{c}{\textbf{Revenue in \%\,\,\textuparrow}} & \multicolumn{1}{c}{\textbf{Runtime in Hours}}\\
    \midrule
    LSVM & 0.60& 20& 4& 100 &0.79$\pm$\scriptsize\,0.47 & 73.74$\pm$\scriptsize\,3.65 & \hphantom{1}5.93\\
    & 0.75& 20& 4& 100&   0.61$\pm$\scriptsize\,0.42 & 73.99$\pm$\scriptsize\,3.56 & \hphantom{1}7.31 \\
    & 0.90& 20& 4& 100&   0.37$\pm$\scriptsize\,0.24 & 73.18$\pm$\scriptsize\,3.60 & 23.99\\
    & \ccell 0.95& \ccell20& \ccell4& \ccell100&   \ccell0.16$\pm$\scriptsize\,0.20 & \ccell72.84$\pm$\scriptsize\,3.45 & \ccell22.21\\
    \midrule
    SRVM & 0.60& 20& 4& 100&0.14$\pm$\scriptsize\,0.04 & 53.86$\pm$\scriptsize\,1.44 & 32.07\\
    &  \ccell0.75& \ccell20& \ccell4&  \ccell100& \ccell0.12$\pm$\scriptsize\,0.09 & \ccell53.87$\pm$\scriptsize\,1.55 & \ccell25.89\\
    & 0.90& 20& 4& 100&  0.72$\pm$\scriptsize\,0.11 & 53.94$\pm$\scriptsize\,1.59 & 45.20\\
    & 0.95& 20& 4& 100&  0.75$\pm$\scriptsize\,0.11 & 54.09$\pm$\scriptsize\,1.55 & 45.09\\
    \midrule
    MRVM & 0.60& 20& 4& 100& 7.94$\pm$\scriptsize\,0.36 & 42.14$\pm$\scriptsize\,0.98 & 68.84\\
    & 0.75& 20& 4& 100& 8.31$\pm$\scriptsize\,0.31 & 40.92$\pm$\scriptsize\,0.73 & 27.09\\
    & 0.90& 20& 4& 100& 7.92$\pm$\scriptsize\,0.33 & 42.61$\pm$\scriptsize\,0.89 & 21.70\\
    & \ccell0.95& \ccell20& \ccell4& \ccell100& \ccell7.45$\pm$\scriptsize\,0.37 & \ccell41.19$\pm$\scriptsize\,0.86 & \ccell21.51\\
    \bottomrule
    \end{tabular}
}
    \end{sc}
    \vskip -0.2cm
    \caption{
    Detailed BOCA results for a reduced number of $\Qinit=20$ initial random queries. We present efficiency loss, relative revenue and runtime of MLCA with our MVNN-based uUB $\uUBMi$ as $\mathcal{A}_i$. Shown are averages including a 95\%-normal-CI on a test set of $50$ instances in all three considered SATS domains. The best MVNN-based uUBs per domain (w.r.t. the quantile parameter $q$) based on the lowest efficiency loss are marked in grey.}
    \label{tab:efficiency_loss_mlca_appendix_20iQ}
\end{table*}
\begin{table*}[htbp]
	\robustify\bfseries
	\centering
	\begin{sc}
	\resizebox{1\textwidth}{!}{
    \begin{tabular}{crrrrccc}
    \toprule
    \textbf{Domain} & \textbf{Quantile Parameter q}&$\boldsymbol{\Qinit}$&$\boldsymbol{\Qround}$&$\boldsymbol{\Qmax}$ & \multicolumn{1}{c}{\textbf{Efficiency Loss in \%\,\,\textdownarrow}}  & \multicolumn{1}{c}{\textbf{Revenue in \%\,\,\textuparrow}} & \multicolumn{1}{c}{\textbf{Runtime in Hours}}\\
    \midrule
    LSVM & 0.60& 20& 4& 100 &0.71$\pm$\scriptsize\,0.43 & 73.73$\pm$\scriptsize\,3.57 & \hphantom{1}4.87\\
    & 0.75& 20& 4& 100&  0.71$\pm$\scriptsize\,0.45 & 74.71$\pm$\scriptsize\,3.32 & \hphantom{1}5.11\\
    & \ccell0.90& \ccell20& \ccell4& \ccell100&  \ccell0.58$\pm$\scriptsize\,0.38 & \ccell74.81$\pm$\scriptsize\,3.51 & \ccell12.76\\
    & 0.95& 20& 4& 100&  0.59$\pm$\scriptsize\,0.34 & 73.88$\pm$\scriptsize\,3.64 & 14.11\\
    \midrule
    SRVM & 0.60& 20& 4& 100 & 0.18$\pm$\scriptsize\,0.04 & 54.13$\pm$\scriptsize\,1.45 & 30.47\\
    & \ccell 0.75& \ccell20& \ccell4&  \ccell100& \ccell0.16$\pm$\scriptsize\,0.04 & \ccell54.32$\pm$\scriptsize\,1.49 & \ccell28.62\\
    & 0.90& 20& 4& 100& 0.29$\pm$\scriptsize\,0.06 & 54.55$\pm$\scriptsize\,1.51 & 31.74\\
    & 0.95& 20& 4& 100&  0.30$\pm$\scriptsize\,0.05 & 54.35$\pm$\scriptsize\,1.58 & 30.39\\
    \midrule
    MRVM & 0.60& 20& 4& 100 & 7.73$\pm$\scriptsize\,0.43 & 42.51$\pm$\scriptsize\,0.75 & 25.39\\
    & 0.75& 20& 4& 100&  8.14$\pm$\scriptsize\,0.40 & 41.46$\pm$\scriptsize\,0.85 & \hphantom{1}7.70\\
    & 0.90& 20& 4& 100& 7.73$\pm$\scriptsize\,0.41 & 41.42$\pm$\scriptsize\,0.85 & 10.06\\
    & \ccell0.95& \ccell20& \ccell4& \ccell100&\ccell7.52$\pm$\scriptsize\,0.36 & \ccell41.33$\pm$\scriptsize\,1.11 & \ccell\hphantom{1}9.44\\
    \bottomrule
    \end{tabular}
}
    \end{sc}
    \vskip -0.2cm
    \caption{Detailed OUR-MVNN-MLCA results for a reduced number of $\Qinit=20$ initial random queries. We present efficiency loss, relative revenue and runtime of OUR-MVNN-MLCA, i.e., MLCA with our MVNN-based mean $\meanMi$ as $\mathcal{A}_i$. Shown are averages including a 95\%-normal-CI on a test set of $50$ instances in all three considered SATS domains.}
    \label{tab:efficiency_loss_mlca_appendix_20iQ_onlyMean}
\end{table*}

\begin{figure*}
    \centering
    \resizebox{\textwidth}{!}{
    \includegraphics{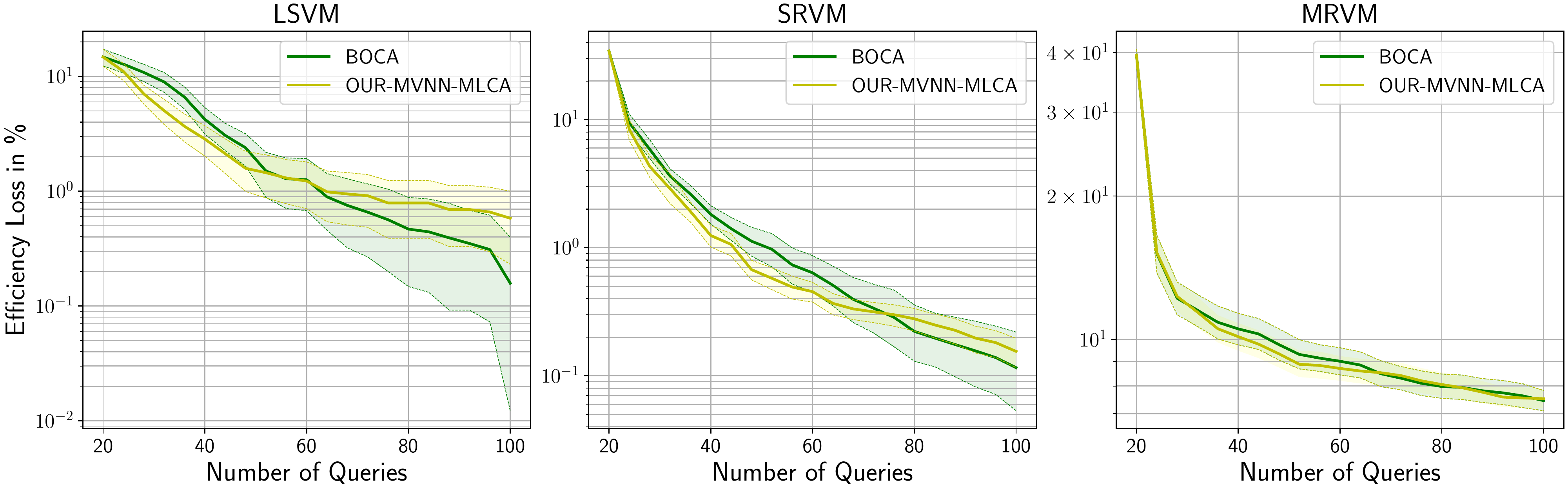}
    }
    \caption{Efficiency loss paths (i.e., regret plots) for a reduced number of initial queries $\Qinit=20$ of BOCA winners (green) from \Cref{tab:efficiency_loss_mlca_appendix_20iQ} compared to OUR-MVNN-MLCA winners (yellow) from \Cref{tab:efficiency_loss_mlca_appendix_20iQ_onlyMean}. Shown are averages with 95\% CIs over 50 instances.
    }
    \label{fig:efficiency_path_plot_summary_appendix_qinit20}
\end{figure*}

\fi
\end{document}